\documentclass{article}

\usepackage{microtype}
\usepackage{graphicx}

\usepackage[accepted]{icml2026}

\usepackage{float}
\usepackage[inline]{enumitem}
\usepackage{dsfont}
\usepackage[utf8]{inputenc} 
\usepackage[T1]{fontenc}    
\usepackage{url}            
\usepackage{booktabs,threeparttable}       
\usepackage{amsfonts}       
\usepackage[x11names, table]{xcolor} 
\usepackage{pifont}
\usepackage{makecell}
\usepackage{etoc}
\etocsettocdepth{3}
\etocframedstyle[1]{\textbf{\textsc{Table of Contents}}}
\etocsettocdepth{3}



\usepackage{amsmath}
\DeclareMathOperator*{\minimize}{minimize}

\DeclareMathOperator*{\argmin}{arg\,min} 

\DeclareMathAlphabet{\mathsfit}{T1}{\sfdefault}{\mddefault}{\sldefault}

\makeatletter
\usepackage[bookmarksopen=true,bookmarks=true,colorlinks=true,backref=page,breaklinks]{hyperref}
\renewcommand*{\backref}[1]{}
\renewcommand*{\backrefalt}[4]{({\footnotesize%
  \ifcase #1 Not cited.%
    \or page~#2%
    \else pages #2
  \fi%
  })}

\makeatother

\usepackage[commenters={kyrkim},noautonum]{shortex}
\usepackage{thm-restate}

\usepackage{pgf}
\usepackage{pgfplots}
\usepackage{pgfplotstable}
\usepackage{tikz}
\usetikzlibrary{positioning}
\usetikzlibrary{arrows}
\pgfplotsset{compat=1.17}
\usepgfplotslibrary{external} 

\allowdisplaybreaks

\icmltitlerunning{SGVI with Price's Gradient Estimator from Bures--Wasserstein to Parameter Space}

\begin{document}

\twocolumn[
  \icmltitle{Stochastic Gradient Variational Inference with Price's Gradient Estimator\\from Bures--Wasserstein to Parameter Space}



  \icmlsetsymbol{equal}{*}

  \begin{icmlauthorlist}
    \icmlauthor{Kyurae Kim}{upenn}
    \icmlauthor{Qiang Fu}{yale}
    \icmlauthor{Yi-An Ma}{ucsd}
    \icmlauthor{Jacob R. Gardner}{upenn}
    \icmlauthor{Trevor Campbell}{ubc}
  \end{icmlauthorlist}

  \icmlaffiliation{upenn}{University of Pennsylvania, Philadelphia, U.S.}
  \icmlaffiliation{yale}{Yale University, New Haven,  U.S.}
  \icmlaffiliation{ubc}{University of British Columbia, Vancouver, Canada}
  \icmlaffiliation{ucsd}{University of California San Diego, San Diego, U.S.}

  \icmlcorrespondingauthor{Kyurae Kim}{\href{mailto:kyrkim@seas.upenn.edu}{kyrkim@seas.upenn.edu}}
  \icmlcorrespondingauthor{Qiang Fu}{\href{mailto:qiang.fu@yale.edu}{qiang.fu@yale.edu}}
  \icmlcorrespondingauthor{Yi-An Ma}{\href{mailto:yianma@ucsd.edu}{yianma@ucsd.edu}}
  \icmlcorrespondingauthor{Jacob R. Gardner}{\href{mailto:jacobrg@seas.upenn.edu}{jacobrg@seas.upenn.edu}}
  \icmlcorrespondingauthor{Trevor Campbell}{\href{mailto:trevor@stat.ubc.ca}{trevor@stat.ubc.ca}}

  \icmlkeywords{Variational inference, Wasserstein gradient descent, Bures-Wasserstein gradient descent, stochastic gradient descent, gradient estimation, Bayesian inference}

  \vskip 0.3in
]



\printAffiliationsAndNotice{}  

\begin{abstract}
For approximating a target distribution given only its unnormalized log-density, stochastic gradient-based variational inference (VI) algorithms are a popular approach.
For example, Wasserstein VI (WVI) and black-box VI (BBVI) perform gradient descent in measure space (Bures--Wasserstein space) and parameter space, respectively.
Previously, for the Gaussian variational family, convergence guarantees for WVI have shown superiority over existing results for black-box VI with the reparametrization gradient, suggesting the measure space approach might provide some unique benefits.
In this work, however, we close this gap by obtaining identical state-of-the-art iteration complexity guarantees for both.
In particular, we identify that WVI's superiority stems from the specific gradient estimator it uses, which BBVI can also leverage with minor modifications.
The estimator in question is usually associated with Price's theorem and utilizes second-order information (Hessians) of the target log-density.
We will refer to this as Price's gradient.
On the flip side, WVI can be made more widely applicable by using the reparametrization gradient, which requires only gradients of the log-density.
We empirically demonstrate that the use of Price's gradient is the major source of performance improvement.
\end{abstract}

\section{Introduction}

Variational inference (VI;~\citealp{jordan_variational_1998,blei_variational_2017,peterson_explorations_1989,hinton_keeping_1993}) is a collection of algorithms for approximating a target distribution $\pi$ over some family of parametric distributions $\mathcal{Q}$ when only the unnormalized density of $\pi$, denoted by $\widetilde{\pi}$, is available.
When $\pi$ is supported on $\mathbb{R}^d$ such that its associated potential function $U = - \log \widetilde{\pi}$ is in $\mathbb{R}^{d} \to \mathbb{R}$, it is common to leverage stochastic gradient-based algorithms, as they only require local information of $U$~\citep{graves_practical_2011,salimans_fixedform_2013,wingate_automated_2013,titsias_doubly_2014,ranganath_black_2014,kingma_autoencoding_2014,rezende_stochastic_2014}. 
Most VI algorithms are designed to minimize the variational free energy, also known as the negative evidence lower bound~\citep{jordan_introduction_1999}, defined as $\mathcal{F}(q) \triangleq \mathcal{E}(q) + \mathcal{H}(q)$, where $\mathcal{E}$ is the energy functional associated with $U$, while $\mathcal{H}$ is the Boltzmann entropy.
That is, we solve
{%
\setlength{\belowdisplayskip}{1.5ex} \setlength{\belowdisplayshortskip}{1.5ex}
\setlength{\abovedisplayskip}{1.5ex} \setlength{\abovedisplayshortskip}{1.5ex}
\[
    \minimize_{q \in \mathcal{Q}} \;\; \big\{ \; \mathcal{F}\lt(q\rt) = \mathrm{KL}\lt(q, \pi\rt) + \mathcal{F}\lt(\pi\rt) \; \big\} 
    \nonumber  
\]
}%
through only zeroth-, first-, and second-order information of $U$.
Since $\mathcal{F}$ is equal to the exclusive KL divergence~\citep{kullback_information_1951} between $q$ and $\pi$ up to the constant $\mathcal{F}(\pi)$, this also minimizes $q \mapsto \mathrm{KL}(q, \pi)$.

\begin{table*}[t]
    \caption{Overview of Main Theoretical Results.}\label{table:overview}
    \centering
    {
    \begingroup
    \begin{threeparttable}
    \renewcommand{\arraystretch}{1.5} 
    \begin{tabular}{cccccc}
        Algorithm & Space & \makecell{Gradient\\Estimator} & \makecell{Maximum\\Step Size} & \makecell{Iterations Complexity} & Reference \\ \toprule
        SPGD & $\Lambda$ & Reparam. & $1/(d L \kappa)$ & $d \kappa^2 \mathrm{tr}\lt(\mu \Sigma_{*}\rt) \nicefrac{1}{\epsilon}$  & \makecell{\citet{domke_provable_2023}\\\citet{kim_convergence_2023}} \\
        SPBWGD & $\mathrm{BW}(\mathbb{R}^d)$ & Bonnet--Price & $1/(L \kappa^2)$ & $d \kappa \frac{1}{\epsilon} \log \frac{1}{\epsilon} + \kappa^3 \log\lt( \Delta^2 \frac{1}{\epsilon} \rt) $ & \citet{diao_forwardbackward_2023} \\ 
        SPGD & $\Lambda$ & Bonnet--Price & $1/(L \kappa)$ & 
        $
        d \kappa \frac{1}{\epsilon}
        +
        \sqrt{d} \kappa^{3/2} \log\lt( \kappa \Delta^2  \rt) \frac{1}{\sqrt{\epsilon}}
        +
        \kappa^2 \log \frac{1}{\epsilon}
        $
        & \cref{thm:parameterspace_proximal_gradient_descent_pricestein} \\
        SPBWGD & $\mathrm{BW}(\mathbb{R}^d)$ & Bonnet--Price & 
        $1/(L \kappa)$ & 
        $ 
        d \kappa \frac{1}{\epsilon}
        +
        \sqrt{d} \kappa^{3/2} \log\lt( \kappa \Delta^2  \rt) \frac{1}{\sqrt{\epsilon}}
        +
        \kappa^2 \log \frac{1}{\epsilon}
        $
        & \cref{thm:wasserstein_proximal_gradient_descent_pricestein} \\
    \end{tabular}
    \begin{tablenotes}
      \small
      \item \textit{Note:} The complexity statements assume $\mu$-strong convexity and $L$-smoothness of $U$ (\cref{assumption:potential_convex_smooth}); Denote the initialization as $q_0$, the last iterate as $q_T$, and the global optimizer of $\mathcal{F}$ as $q_*$; $\epsilon > 0$ is the target accuracy level for ensuring $\mu \mathbb{E} {\mathrm{W}_2\lt(q_{T},q_*\rt)}^2 \leq \epsilon$; $\kappa = L/\mu$ is the condition number of $U$; $\Delta^2 = \mu {\mathrm{W}_2(q_0, q_*)}^2$ is the distance between the initialization and the optimum.
    \end{tablenotes}
    \end{threeparttable}
    \endgroup
    }%
    \vspace{-1ex}
\end{table*}

A common approach for minimizing $\mathcal{F}$ is, informally speaking, to leverage some sort of stochastic gradient descent (SGD;~\citealp{robbins_stochastic_1951,bottou_online_1999,bottou_optimization_2018,bach_nonasymptotic_2011,nemirovski_robust_2009,shalev-shwartz_pegasos_2011}) scheme, where intractable terms (such as the gradient of the energy $\mathcal{E}$) are stochastically estimated.
There are two popular ways to realize this conceptual algorithm.
The most widely used approach in practice is to represent each $q_{\lambda} \in \mathcal{Q}$ via a Euclidean vector of \textit{variational parameters} $\lambda \in \Lambda$, and run gradient descent on the Euclidean space of parameters $\Lambda \subseteq \mathbb{R}^p$~\citep{wingate_automated_2013,kucukelbir_automatic_2017,titsias_doubly_2014,ranganath_black_2014,salimans_fixedform_2013}.
This is now referred to as \textit{black-box variational inference} (BBVI).
The other approach is to define a tractable notion of measure-valued derivatives, which can directly perform gradient descent in measure space.
In particular, recent advances in our understanding of the Wasserstein geometry~\citep{villani_optimal_2009,chewi_statistical_2025} and gradient flows~\citep{ambrosio_gradient_2005,jordan_variational_1998} have contributed to the development of (parametric\footnote{We focus on WVI on \textit{parametric} families, which excludes particle-based WVI methods.}) Wasserstein variational inference (WVI;~\citealp{lambert_variational_2022,diao_forwardbackward_2023,huix_theoretical_2024,talamon_variational_2025}) algorithms that take this route.
We also note that natural gradient VI algorithms~\citep{khan_bayesian_2023,khan_fast_2018,lin_fast_2019,tan_analytic_2025} utilize parameter gradients while measuring distance using the KL pseudo-metric.
However, NGVI will not be the focus of this work.

Given that the WVI methods utilize a proper metric~\citep{villani_optimal_2009} between measures---the Wasserstein-2 metric $\mathrm{W}_2$---it is natural to expect them to outperform BBVI.
Indeed, current theoretical evidence suggests that this is the case.
Consider the Gaussian variational family, also known as the Bures--Wasserstein space~\citep{bures_extension_1969,bhatia_bures_2019}, set as
{%
\setlength{\belowdisplayskip}{1.5ex} \setlength{\belowdisplayshortskip}{1.5ex}
\setlength{\abovedisplayskip}{1.5ex} \setlength{\abovedisplayshortskip}{1.5ex}
\[
    \mathcal{Q} = \mathrm{BW}(\mathbb{R}^d) \triangleq \{ \mathrm{Normal}(m, \Sigma) \mid m \in \mathbb{R}^d , \Sigma \in \mathbb{S}_{\succ 0}^d \} \; .
    \nonumber
\]
}%
Here, $(\mathrm{BW}(\mathbb{R}^d), \mathrm{W}_2)$ forms a metric space.
Denoting the global minimizer as $q_* = \mathrm{Normal}(m_*, \Sigma_*) = \argmin_{q \in \mathcal{Q}} \mathcal{F}\lt(q\rt)$, for a $\mu$-strongly convex and $L$-smooth potential $U$, the algorithm by \citet{diao_forwardbackward_2023} requires $\mathrm{O}\lt( d \kappa \epsilon^{-1} \log \epsilon^{-1} + \kappa^3 \log \epsilon^{-1} \rt)$ steps to ensure $\mu \mathbb{E}[{\mathrm{W}_2\lt(q, q_*\rt)}^2] \leq \epsilon$.
In contrast, the BBVI equivalent with a certain covariance parametrization, the reparametrization gradient estimator~\citep{ho_perturbation_1983,rubinstein_sensitivity_1992}, and stochastic proximal gradient descent (SPGD;~\citealp{nemirovski_robust_2009}) requires $\mathrm{O}\lt( d \kappa^2 \mathrm{tr}\lt(\mu \Sigma_*\rt) \epsilon^{-1} \rt)$ steps~\citep{kim_convergence_2023,domke_provable_2023}.
It might therefore appear that the guarantees for BBVI are weaker in the limit $\epsilon \to 0$.


In this work, we demonstrate that the difference in theoretical guarantees originates from the specific gradient estimator for the scale parameter used in stochastic implementations of WVI rather than the geometry being used.
The estimator in question can be derived via Price's theorem~\citep{price_useful_1958} and leverages second-order information (Hessians) of the target log-density.
We will refer to this as Price's gradient estimator.
Through Stein~\citep{liu_siegels_1994,stein_estimation_1981} or Price's theorem, BBVI with SPGD can also make use of essentially the same gradient estimator, resulting in an iteration complexity of 
$
\mathrm{O}\big(
        d \kappa \epsilon^{-1}
        +
        \sqrt{d} \kappa^{3/2} \log\lt( \kappa \Delta^2  \rt) \epsilon^{-1/2}
        +
        \kappa^2 \log \epsilon^{-1}
\big)
$.
Furthermore, to ensure a fair comparison, we present a refined analysis of the WVI counterpart of SPGD~\citep{diao_forwardbackward_2023}, resulting in an iteration complexity of
$
\mathrm{O}\big(
        d \kappa \epsilon^{-1}
        +
        \sqrt{d} \kappa^{3/2} \log\lt( \kappa \Delta^2  \rt) \epsilon^{-1/2}
        +
        \kappa^2 \log \epsilon^{-1}
\big)
$.
These results suggest that the specific implementation of BBVI studied here and WVI might not be that different after all.
While this has been suggested by~\citet{yi_bridging_2023,hoffman_blackbox_2020}, this work further supports this fact by contributing a non-asymptotic discrete-time analysis.
Our theoretical results are organized in \cref{table:overview}.

In addition, we demonstrate that WVI can also leverage the reparametrization gradient traditionally used in BBVI~\citep{titsias_doubly_2014,kingma_autoencoding_2014,rezende_stochastic_2014}.
Unlike Price's gradient previously used in WVI, the reparametrization gradient only requires first-order information ($\nabla U$).
Thus, the resulting WVI algorithm should be more widely applicable in practice.
Given this fact, we empirically compare the performance of BBVI and WVI, both with the Hessian-based gradient estimators and the reparametrization gradient.
Results demonstrate that a large fraction of the performance difference stems from the use of Hessian-based gradients, supporting the claim that the gradient estimator is the main source of performance.

\vspace{-1ex}
\section{Background}

\paragraph{Notation.}
For any $x, y \in \mathbb{R}^d$, we denote the Euclidean inner product and norm as $\inner{x,y} = x^{\top} y$ and $\norm{x}_2 \triangleq \sqrt{\inner{x,x}}$.
For any matrix $A,B \in \mathbb{R}^{d \times d}$, $\operatorname{tr}\lt(A\rt) \triangleq \sum_{i=1}^d A_{ii}$, $\inner{A,B}_{\mathrm{F}} \triangleq \mathrm{tr}(A^{\top}B)$, $\norm{A}_{\mathrm{F}} \triangleq \sqrt{\inner{A,A}_{\mathrm{F}}}$, and the $\ell_2$ operator norm of $A$ is denoted as $\norm{A}_2$.
The symmetric and positive definite subsets of $\mathbb{R}^{d \times d}$ will be denoted as $\mathbb{S}^d$ and $\mathbb{S}_{\succ 0}^d$, while $\mathbb{L}_{\succ 0}^d$ denotes the set of lower-triangular matrices with strictly positive eigenvalues.
We will represent a measure and its density with the same symbol.
For the space of square-integrable measures $\mathcal{P}_2(\mathcal{X}) \triangleq \{q \mid \int_{\mathcal{X}} \norm{x}^2 \mathrm{d}q(x) < +\infty \}$, for some $q \in \mathcal{P}_2(\mathbb{R}^d)$, the set of integrable functions is denoted as $\mathrm{L}^2\lt(q\rt) \triangleq \lt\{ f \mid \int \norm{f}_2^2 \mathrm{d}q < +\infty \rt\}$.
For any two probability measures $p, q \in \mathcal{P}_2\lt(\mathbb{R}^d\rt)$, we denote the set of couplings between the two as $\Psi\lt(p, q\rt)$.
Then the squared Wasserstein-2 distance between $p$ and $q$ is ${\mathrm{W}_2\lt(p, q\rt)}^2 \triangleq \inf_{\psi \in \Psi\lt(p, q\rt)} \int_{\mathbb{R}^d \times \mathbb{R}^d} \norm{x - y}_2^2 \, \mathrm{d}\psi(x, y) $.
For some measurable map $M : \mathbb{R}^d \to \mathbb{R}^d$ and measure $q$ supported on $\mathbb{R}^d$, $M_{\# q}$ denotes the corresponding push-forward measure.
Unless stated otherwise, the coupling attaining the infimum of $W_2$, $\psi^* \in \Psi(p, q)$, is referred to as ``the optimal coupling,'' which is guaranteed to exist~\citep[Theorem 4.1]{villani_optimal_2009} and is unique by Brenier's Theorem~\citep{brenier_polar_1991}.

\subsection{Problem Setup}

Our focus will be on first-order stochastic optimization algorithms for solving the problem
{%
\setlength{\belowdisplayskip}{-1ex} \setlength{\belowdisplayshortskip}{-1ex}
\setlength{\abovedisplayskip}{2ex} \setlength{\abovedisplayshortskip}{2ex}
\[
    \minimize_{q \in \mathcal{Q}} \;\; \lt\{ \mathcal{F}\lt(q\rt) \triangleq \mathcal{E}\lt(q\rt) + \mathcal{H}\lt(q\rt) 
    \rt\}
    \nonumber
    \; ,
\]
}%
\vspace{-2ex}
\begin{center}
   {\begingroup
  \setlength\tabcolsep{3pt} 
  \begin{tabular}{llll}
    \text{where} & $\mathcal{E}\lt(q\rt)$ & $\triangleq \int_{\mathbb{R}^d} U\lt(z\rt) \, q\lt(\mathrm{d}z\rt)$      & \text{(Energy)} \\
    & $\mathcal{H}\lt(q\rt)$ & $\triangleq \int_{\mathbb{R}^d} \log q(z) \, q(\mathrm{d}z) $ . & \text{(Boltzmann entropy)}
  \end{tabular}
  \endgroup}
\vspace{-2ex}
\end{center}
We consider the ``non-conjugate'' setup, where $\mathcal{E}$ is intractable due to the expectation over $q$.
Suppose we can parametrize each $q_{\lambda} \in \mathcal{Q}$ with a Euclidean vector of parameters $\lambda \in \Lambda$.
Then it is equivalent to minimize $\lambda \mapsto \mathcal{F}(q_{\lambda})$ over the Euclidean parameter space $\Lambda \subseteq \mathbb{R}^p$.

Informally speaking, when $U$ is ``regular,'' $\mathcal{E}$ also tends to be regular.
For instance, if $U$ is Lipschitz-smooth, then $\mathcal{E}$ also tends to exhibit appropriate notion of Lipschitz-smoothness~\citep{domke_provable_2020,diao_forwardbackward_2023,lambert_variational_2022}.
The entropy term $\mathcal{H}$, however, does not enjoy Lipschitz smoothness in general.
For instance, for Gaussians $q = \mathrm{Normal}(m, \Sigma)$, $\mathcal{H}(q)$ blows up as the covariance $\Sigma$ becomes singular.
Typically, in optimization, such non-smoothness is remedied by relying on proximal gradient algorithms~\citep[\S 9.3]{wright_optimization_2021}.
Indeed, for minimizing $\mathcal{F}$, stochastic proximal gradient algorithms have been proposed for both the Bures--Wasserstein~\citep{diao_forwardbackward_2023} and Euclidean parameter spaces~\citep{domke_provable_2020}.
In the following sections, we will introduce these algorithms.

\subsection{Wasserstein Variational Inference via Stochastic Proximal Bures--Wasserstein Gradient Descent}\label{section:spbwgd}
Since the seminal work of \citet{jordan_variational_1998}, it is known that $\mathcal{F}$ can be minimized by simulating its Wasserstein gradient flow.
The forward-backward discretization of this flow results in the Wasserstein-analog of proximal gradient descent~\citep{wibisono_sampling_2018,bernton_langevin_2018,salim_wasserstein_2020} operating on the metric space $(\mathcal{P}_2(\mathbb{R}^d), \mathrm{W}_2)$.
This algorithm, however, is not directly implementable.
Recently, \citet{lambert_variational_2022,diao_forwardbackward_2023} demonstrated that, by constraining optimization to the \textit{Bures}--Wasserstein manifold $\mathrm{BW}(\mathbb{R}^d) \subset \mathcal{P}_2(\mathbb{R}^d)$~\citep{bures_extension_1969,bhatia_bures_2019}, the algorithm becomes implementable.
In particular, the proximal Bures--Wasserstein gradient descent scheme iterates, for each $t \geq 0$, 
{%
\setlength{\belowdisplayskip}{1.5ex} \setlength{\belowdisplayshortskip}{1.5ex}
\setlength{\abovedisplayskip}{1.5ex} \setlength{\abovedisplayshortskip}{1.5ex}
\[
    q_{t+ \nicefrac{1}{2}} &= {\lt( \mathrm{Id} - \gamma_t \nabla_{\mathrm{BW}} \mathcal{E}\lt(q_t\rt) \rt)}_{\# q_t}
    \label{eq:bures_wasserstein_gradient_descent_step}
    \\
    q_{t+1} &= 
    \operatorname{JKO}_{\gamma_t \mathcal{H}} (q_{t+ \nicefrac{1}{2}}) 
    \nonumber
    \; ,
\]
}%
where, for any $q = \mathrm{Normal}(m, \Sigma) \in \mathrm{BW}\lt(\mathbb{R}^d\rt)$, the Bures--Wasserstein gradient of $\mathcal{E}$ can be derived~\citep[Appendix C]{lambert_variational_2022} as
{%
\setlength{\belowdisplayskip}{1ex} \setlength{\belowdisplayshortskip}{1ex}
\setlength{\abovedisplayskip}{1.5ex} \setlength{\abovedisplayshortskip}{1.5ex}
\[
    \nabla_{\mathrm{BW}} \mathcal{E}\lt(q\rt)
    \;\triangleq\;
    x \mapsto
    \nabla_m \mathcal{E}\lt(q\rt)
    +
    2 \nabla_{\Sigma} \mathcal{E}\lt(q\rt)
    \lt(x - m\rt) \; ,
    \nonumber
\]
}%
while, for any $\mathcal{G} : \mathrm{BW}(\mathbb{R}^d) \to \mathbb{R} \cup \{+\infty\}$, 
{%
\setlength{\belowdisplayskip}{1.5ex} \setlength{\belowdisplayshortskip}{1.5ex}
\setlength{\abovedisplayskip}{1ex} \setlength{\abovedisplayshortskip}{1ex}
\[
    \operatorname{JKO}_{\mathcal{G}}(q)
    &\triangleq
    \argmin_{p \in \mathrm{BW}(\mathbb{R}^d)}
    \lt\{
        \mathcal{G}\lt(p\rt) 
        +
        (\nicefrac{1}{2}) {\mathrm{W}_2\lt(p, q\rt)}^2
    \rt\} 
    \; .
    \nonumber
\]
}%
is the Wasserstein-analog of the proximal operator, commonly referred to as the ``JKO operator.'' 
For the special case of $\mathcal{G} = \gamma_t \mathcal{H}$ and the Bures--Wasserstein space, the JKO operator admits the tractable closed-form expression~\citep[Example 7]{wibisono_sampling_2018} 
{%
\setlength{\belowdisplayskip}{1.5ex} \setlength{\belowdisplayshortskip}{1.5ex}
\setlength{\abovedisplayskip}{1.5ex} \setlength{\abovedisplayshortskip}{1.5ex}
\[
    \mathcal{N}(\mu_t, \Sigma_{t+1}) = \operatorname{JKO}_{\mathcal{G}}( \mathcal{N}(\mu_t, \Sigma_{t+1/2} )) \; ,
    \nonumber
\]
}%
where
{%
\setlength{\belowdisplayskip}{1.5ex} \setlength{\belowdisplayshortskip}{1.5ex}
\setlength{\abovedisplayskip}{1.5ex} \setlength{\abovedisplayshortskip}{1.5ex}
\[
    \Sigma_{t+1} = \frac{1}{2} \lt(\Sigma_{t+\nicefrac{1}{2}} + 2 \gamma_t \mathrm{I}_d + {\lt( \Sigma_{t+\nicefrac{1}{2}} \lt(\Sigma_{t+\nicefrac{1}{2}} + 4 \gamma_t \mathrm{I}_d \rt) \rt)}^{\nicefrac{1}{2}} \rt) 
    \nonumber
    \; ,
\]
}%
which is key to obtaining an implementable algorithm.

Still, the Bures--Wasserstein gradient $\nabla_{\mathrm{BW}} \mathcal{E}(q)$ involves expectations over $q = \mathrm{Normal}(\mu, \Sigma)$ that are generally not tractable.
Therefore, these have to be replaced with stochastic estimates~\citep{lambert_variational_2022} of $\nabla_m \mathcal{E}$ and $\nabla_{\Sigma} \mathcal{E}$ resulting in the estimator 
{%
\setlength{\belowdisplayskip}{1.5ex} \setlength{\belowdisplayshortskip}{1.5ex}
\setlength{\abovedisplayskip}{1.5ex} \setlength{\abovedisplayshortskip}{1.5ex}
\[
    \widehat{\nabla_{\mathrm{BW}}} \mathcal{E}\lt(q; \epsilon\rt)
    \;\triangleq\;
    x \mapsto
    \widehat{\nabla_m \mathcal{E}}\lt(q; \epsilon\rt)
    +
    2 \widehat{\nabla_{\Sigma} \mathcal{E}}\lt(q; \epsilon\rt)
    \lt(x - m\rt) \; ,
    \nonumber
\]
}%
where $\epsilon \sim \varphi = \mathrm{Normal}(0_d, \mathrm{I}_d)$ is standard Gaussian noise. 
Replacing $\nabla_{\mathrm{BW}} \mathcal{E}$ in \cref{eq:bures_wasserstein_gradient_descent_step} with $\widehat{\nabla_{\mathrm{BW}}} \mathcal{E}$ results in stochastic proximal Bures--Wasserstein gradient descent (SPBWGD; \citealp{diao_forwardbackward_2023}).
\citeauthor{lambert_variational_2022,diao_forwardbackward_2023} rely on the estimators
{%
\setlength{\belowdisplayskip}{1.5ex} \setlength{\belowdisplayshortskip}{1.5ex}
\setlength{\abovedisplayskip}{1.5ex} \setlength{\abovedisplayshortskip}{1.5ex}
\[
    \widehat{\nabla_m^{\mathrm{bonnet}} \mathcal{E}}\lt(q; \epsilon\rt) &\triangleq \nabla U\lt(Z\rt)
    \nonumber
    \\
    \widehat{\nabla_{\Sigma}^{\mathrm{price}} \mathcal{E}}\lt(q; \epsilon\rt) &\triangleq (\nicefrac{1}{2}) \nabla^2 U\lt(Z\rt) \; , \quad
    \label{eq:price_estimator}
\]
}%
where  $Z = \mathrm{cholesky}(\Sigma) \epsilon + \mu$.
The fact that these estimators are unbiased follows from the theorems by~\citet{bonnet_transformations_1964} and \citet{price_useful_1958} or Riemannian geometry~\citep[Appendix B.3]{altschuler_averaging_2021}.
For each $t \geq 0$, the resulting update rule for the iterate $q_t = \mathrm{Normal}(m_t, \Sigma_t)$ is
{%
\setlength{\belowdisplayskip}{1ex} \setlength{\belowdisplayshortskip}{1ex}
\setlength{\abovedisplayskip}{1.5ex} \setlength{\abovedisplayshortskip}{1.5ex}
\[
    m_{t+1} &= m_t - \gamma_t \widehat{\nabla_{m_t} \mathcal{E}}\lt(q_t; \epsilon_t\rt)  
    \nonumber
    \\
    M_{t+1} &= \mathrm{I}_d - 2 \gamma_t \widehat{\nabla_{\Sigma_t} \mathcal{E}}\lt(q_t; \epsilon_t\rt) 
    \nonumber
    \\
    \nonumber
    \Sigma_{t+\nicefrac{1}{2}} &= M_{t+1} \Sigma_t M_{t+1}^{\top} 
    \\
    \nonumber
    \Sigma_{t+1} &= \frac{1}{2} \lt(\Sigma_{t+\nicefrac{1}{2}} + 2 \gamma_t \mathrm{I}_d + {\lt( \Sigma_{t+\nicefrac{1}{2}} \lt(\Sigma_{t+\nicefrac{1}{2}} + 4 \gamma_t \mathrm{I}_d \rt) \rt)}^{\nicefrac{1}{2}} \rt) \; ,
    \nonumber
\]
}%
where the standard Gaussian noise sequence ${(\epsilon_t)}_{t \geq 0}$ is sampled as $\epsilon_t \stackrel{\mathrm{i.i.d.}}{\sim} \varphi$.
Note the update rule for $\Sigma_{t+1/2}$ is different from the one originally presented by \citet{diao_forwardbackward_2023}; a transpose has been added to $M_{t+1}$.
This change will become necessary later in \cref{section:experiments} when we replace $\widehat{\nabla_{\Sigma} \mathcal{E}}$ with an estimator that is not almost surely symmetric.

\subsection{Black-Box Variational Inference via Stochastic Proximal Gradient Descent}\label{section:bbvi_spgd}

An alternative to SBWPGD is to optimize over the Euclidean space of parameters $\Lambda$.
Recall, in this case, each $q_{\lambda} \in \Lambda$ is assumed to be associated with a Euclidean vector $\lambda \in \Lambda$.
Then, if we have access to an unbiased estimator of $\nabla_{\lambda} \mathcal{E}(q_{\lambda})$, denoted as $\widehat{\nabla_{\lambda} \mathcal{E}}(q_{\lambda})$, $\mathcal{F}$ can be minimized via SPGD, which, for each $t \geq 0$,  updates the \textit{variational parameters} $\lambda_t$ as
{%
\setlength{\belowdisplayskip}{1ex} \setlength{\belowdisplayshortskip}{1ex}
\setlength{\abovedisplayskip}{1ex} \setlength{\abovedisplayshortskip}{1ex}
\[
    \lambda_{t + \nicefrac{1}{2}} &= \lambda_t - \gamma_t \widehat{\nabla_{\lambda_t} \mathcal{E}}\lt(q_{\lambda_t}; \epsilon_t\rt)
    \nonumber
    \\
    \lambda_{t + 1} &= \mathrm{prox}_{\lambda \mapsto \gamma_t \mathcal{H}\lt(q_{\lambda}\rt)} \lt( \lambda_{t + \nicefrac{1}{2}} \rt) \; ,
    \nonumber
\]
}%
where $\epsilon_t \stackrel{\mathrm{i.i.d.}}{\sim} \varphi$ is some randomness source and $\mathrm{prox}$ is the canonical Euclidean proximal operator~\citep{parikh_proximal_2014} defined as, for any proper lower semi-continuous convex function $g : \mathbb{R}^p \to \mathbb{R} \cup \{+\infty\}$, 
{%
\setlength{\belowdisplayskip}{1.0ex} \setlength{\belowdisplayshortskip}{1.0ex}
\setlength{\abovedisplayskip}{2ex} \setlength{\abovedisplayshortskip}{2ex}
\[
    \mathrm{prox}_{g}\lt(\lambda\rt)
    \triangleq
    \argmin_{\lambda' \in \Lambda} \lt\{ g\lt(\lambda'\rt) + \frac{1}{2} \norm{\lambda' - \lambda}_2^2 \rt\} \; .
    \nonumber
\]
}%
For some classes of variational families $\mathcal{Q}$ and parametrizations, the proximal operator can be made tractable~\citep{domke_provable_2020}. 
Before that, however, we must come up with an unbiased estimator of the parameter gradient $\nabla_{\lambda} \mathcal{E}(q_{\lambda})$.


Suppose the variational family $\mathcal{Q}$ and the parametrization $\lambda \mapsto q_{\lambda}$ satisfy the following:
\begin{definition}\label{def:reparametrization_family}
For some $\Lambda \subset \mathbb{R}^p$, the variational family $\mathcal{Q} = \lt\{ q_{\lambda} \mid \lambda \in \Lambda \rt\}$ is referred to as a reparameterizable family if there exists some bijective map $\phi_{\lambda} : \mathbb{R}^d \to \mathbb{R}^d$ differentiable with respect to $\lambda$ and a base distribution $\varphi \in \mathcal{P}_2(\mathbb{R}^d)$ such that, for all $\lambda \in \Lambda$,
{%
\setlength{\belowdisplayskip}{0ex} \setlength{\belowdisplayshortskip}{0ex}
\setlength{\abovedisplayskip}{1ex} \setlength{\abovedisplayshortskip}{1ex}
\[
    Z \sim q_{\lambda} 
    \quad\Leftrightarrow\quad
    Z \stackrel{\mathrm{d}}{=} \phi_{\lambda}\lt(\epsilon\rt) \, ;  
    \; \epsilon \sim \varphi
    \nonumber
    \; .
\]
}%
\end{definition}
Here, $\stackrel{\mathrm{d}}{=}$ is equivalence in distribution.
Then an immediate option for estimating $\nabla_{\lambda} \mathcal{E}(q_{\lambda})$ is to use the reparametrization gradient (\citealp{ho_perturbation_1983,rubinstein_sensitivity_1992}; see also \citealp{mohamed_monte_2020})
{%
\setlength{\belowdisplayskip}{1.5ex} \setlength{\belowdisplayshortskip}{1.5ex}
\setlength{\abovedisplayskip}{1.5ex} \setlength{\abovedisplayshortskip}{1.5ex}
\[
    {\widehat{\nabla_{\lambda}^{\text{rep}} \mathcal{E}}\lt(q_{\lambda}\rt)}
    \triangleq
    \nabla_{\lambda} \phi_{\lambda}\lt(\epsilon\rt) 
    \nabla U\lt(\phi_{\lambda}\lt(\epsilon\rt)\rt)  \; ,
    \label{eq:reparametrization_gradient}
\]
}%
which can be derived by combining the law of the unconscious statistician with the Leibniz integral rule.
This combination of SGD with the reparametrization gradient---commonly referred to as BBVI---is widely used in practice through probabilistic programming frameworks such as Stan~\citep{carpenter_stan_2017}, Turing~\citep{fjelde_turingjl_2025}, Pyro~\citep{bingham_pyro_2019}, and PyMC~\citep{patil_pymc_2010}.

The wide adoption of BBVI in practice is partly due to its flexibility: \cref{def:reparametrization_family} applies to a very wide range of families from Gaussians~\citep{titsias_doubly_2014} to normalizing flows~\citep{rezende_variational_2015}.
Furthermore, \cref{eq:reparametrization_gradient} uses only gradients of $U$, which can be efficiently computed via automatic differentiation~\citep{kucukelbir_automatic_2017}.
In this work, however, we will further restrict our attention to the Gaussian variational family with a \textit{specific} parametrization:
\begin{assumption}\label{assumption:parametrization}
    The variational family $\mathcal{Q}$ is the Gaussian variational family, where each member
    $q_{\lambda} = \mathrm{Normal}(m, CC^{\top}) \in \mathcal{Q}$
    is parametrized as
{%
\setlength{\belowdisplayskip}{0.5ex} \setlength{\belowdisplayshortskip}{0.5ex}
\setlength{\abovedisplayskip}{1.5ex} \setlength{\abovedisplayshortskip}{1.5ex}
    \[
        \Lambda = \lt\{ \lambda = (m, \operatorname{vec}(C)) \mid m \in \mathbb{R}^d, C \in \mathbb{L}^d_{\succ 0} \rt\} \subset \mathbb{R}^p
        \nonumber
    \]
}%
    while the reparametrization function is set as
{%
\setlength{\belowdisplayskip}{0ex} \setlength{\belowdisplayshortskip}{0ex}
\setlength{\abovedisplayskip}{1ex} \setlength{\abovedisplayshortskip}{1ex}
    \[
        \phi_{\lambda}\lt(\epsilon\rt)  = C \epsilon + m 
        \quad\text{and}\quad
        \varphi = \mathrm{Normal}(0_d, \mathrm{I}_d) \; .
        \nonumber
    \]
}
\end{assumption}
Under this parametrization, \cref{eq:reparametrization_gradient} reduces to
{%
\setlength{\belowdisplayskip}{1.5ex} \setlength{\belowdisplayshortskip}{1.5ex}
\setlength{\abovedisplayskip}{1.5ex} \setlength{\abovedisplayshortskip}{1.5ex}
\[
    \hspace{-.7em}
    \widehat{\nabla_{\lambda}^{\text{rep}} \mathcal{E}}(q_{\lambda}; \epsilon)  
    =
    \begin{bmatrix}
        \widehat{\nabla_m^{\text{bonnet}} \mathcal{E}}\lt(q_{\lambda}; \epsilon\rt) \\
        \widehat{\nabla_C^{\text{rep}} \mathcal{E}}\lt(q_{\lambda}; \epsilon\rt) 
    \end{bmatrix}
    =
    \begin{bmatrix}
        \nabla U\lt( \phi_{\lambda}(\epsilon) \rt) \\
        {\epsilon} {\nabla U\lt( \phi_{\lambda}(\epsilon) \rt)}^{\top}
    \end{bmatrix}
    \label{eq:reparametrization_gradient_location_scale}
    .
\]
}%
Furthermore, the proximal operator for the entropy has the closed-form solution~\citep{domke_provable_2020,domke_provable_2023} 
{%
\setlength{\belowdisplayskip}{1.5ex} \setlength{\belowdisplayshortskip}{1.5ex}
\setlength{\abovedisplayskip}{1.5ex} \setlength{\abovedisplayshortskip}{1.5ex}
\[
    (m, C') &= \mathrm{prox}_{\lambda \mapsto \gamma_t \mathcal{H}(q_{\lambda_t})}\lt((m, C)\rt) , \; \text{where}
    \nonumber
    \\
    {[C']}_{ij} &= \begin{cases}
        (\nicefrac{1}{2}) \lt( C_{ii} + \sqrt{C_{ii} + 4 \gamma_t } \rt) & \text{if $i = j$} \\
        C_{ij} & \text{if $i \neq j$} \; .
    \end{cases}
    \nonumber
\]
}%
Compared to alternative ways to parametrize Gaussians, this ``linear'' parametrization is particularly well-behaved~\citep{kim_convergence_2023} and also computationally efficient: each step of BBVI only needs $\mathrm{O}(d^2)$ operations except for evaluating $\nabla U$.
Furthermore, it has been shown that the Bures--Wasserstein gradient $\nabla_{\mathrm{BW}} \mathcal{F}$ is equal to the parameter gradient $\nabla_{\lambda} \mathcal{F}(q_{\lambda})$ under the parametrization of~\cref{assumption:parametrization} up to a coordinate transformation~\citep{yi_bridging_2023}.

    
    

\subsection{Price Gradient Estimators}
A crucial point here is that, for the Gaussian variational family, the estimators traditionally used in WVI (\cref{eq:stein_price_identity_scale_gradient}) and BBVI (\cref{eq:reparametrization_gradient}) both target the same quantities up to constant factor adjustments and are therefore interchangeable.

\begin{proposition}\label{thm:pricestein}
    For any twice-differentiable function $f$, Gaussian $q = \mathrm{Normal}(m, \Sigma)$, and assuming the expectations exist,
{%
\setlength{\belowdisplayskip}{0ex} \setlength{\belowdisplayshortskip}{0ex}
\setlength{\abovedisplayskip}{1.5ex} \setlength{\abovedisplayshortskip}{1.5ex}    
    \[
        \nabla_{\Sigma} \mathbb{E}_q f
        &=
        \frac{1}{2}
        \mathbb{E}_q \nabla^2 f
        \label{eq:price}
        \\
        &=
        \frac{1}{2}
        \Sigma^{-1} 
        \mathbb{E}_{X \sim q}\lt[ \lt(X - m\rt) {\nabla f\lt(X\rt)}^{\top} \rt] 
        \label{eq:stein}
        \; .
    \]
}%
\end{proposition}
\vspace{-2ex}
\begin{proof}
    \cref{eq:price} is Price's theorem~\citep{price_useful_1958}, while \cref{eq:stein} is Stein's identity~\citep{stein_estimation_1981,liu_siegels_1994}.
\end{proof}
\vspace{-1ex}

Denoting $\Sigma = CC^{\top}$, an immediate corollary is 
{%
\setlength{\belowdisplayskip}{1.5ex} \setlength{\belowdisplayshortskip}{1.5ex}
\setlength{\abovedisplayskip}{1.5ex} \setlength{\abovedisplayshortskip}{1.5ex}
\[
    \nabla_{C} \mathcal{E}(q_{\lambda})
    &=
    C^{\top} \mathbb{E}_q \nabla^2 U
    \nonumber
    \\
    &=
    C^{-1} \mathbb{E}_{X \sim q}\big[ {\lt(X - m\rt)} {\nabla U\lt(X\rt)}^{\top}  \big]
    \; ,
    \label{eq:stein_price_identity_scale_gradient}
\]
}%
where, when restricting $C \in \mathbb{L}_{\succ 0}^{d}$, the gradient only needs to be projected to the lower-triangular subspace ($\mathtt{tril}$).
Under \cref{assumption:parametrization}, $C^{-1} {\lt(X - m\rt)} {\nabla U\lt(X\rt)}^{\top}$ exactly corresponds to the reparametrization gradient in \cref{eq:reparametrization_gradient}.
At the same time, \cref{eq:stein_price_identity_scale_gradient} points towards an analog of $\nabla^{\text{price}}_{\Sigma}$ for the scale parameter $C$:
{%
\setlength{\belowdisplayskip}{1.5ex} \setlength{\belowdisplayshortskip}{1.5ex}
\setlength{\abovedisplayskip}{1.5ex} \setlength{\abovedisplayshortskip}{1.5ex}
\[
    &\widehat{\nabla_{C}^{\text{price}} \mathcal{E}}\lt(q_{\lambda}; \epsilon\rt)
    =
    C^{\top} \nabla^2 U\lt(X\rt) \; ,
    \nonumber
    \\
    &\text{where}\quad
    X = \phi_{\lambda}(\epsilon) \; .
    \label{eq:price_estimator_scale}
\]
}%
%

%
Conveniently, this estimator also stays unbiased when $\nabla^2 U$ is replaced with an unbiased estimator of $\nabla^2 U$, enabling doubly stochastic optimization~\citep{titsias_doubly_2014}.
Similarly, \cref{thm:pricestein} points towards a gradient estimator that could be used in WVI,
{%
\setlength{\belowdisplayskip}{1.5ex} \setlength{\belowdisplayshortskip}{1.5ex}
\setlength{\abovedisplayskip}{1.5ex} \setlength{\abovedisplayshortskip}{1.5ex}
\[
    &\widehat{\nabla_{\Sigma}^{\text{rep}} \mathcal{E}}(q_{\lambda}; \epsilon) 
    = 
    \Sigma^{-1} \lt(X - m\rt) {\nabla U\lt(X\rt)}^{\top}, 
    \\
    &\text{where}\quad
    X = \phi_{\lambda}(\epsilon)
    \nonumber
    \; .
\]
}%
Note that similar remarks have already been made by \citet{rezende_stochastic_2014,lin_steins_2025,graves_practical_2011,opper_variational_2009}.
Therefore, the use of these estimators is by no means new.
However, these Hessian-based estimators have not been widely adopted in BBVI, nor have they been analyzed in detail.

A natural question here is how much the choice of gradient estimator affects the performance of different algorithms.
Past experience in stochastic gradient-based VI has shown that the choice of gradient estimator crucially affects performance both in practice~\citep{kucukelbir_automatic_2017,geffner_approximation_2020,geffner_difficulty_2021,geffner_using_2018,geffner_rule_2020,agrawal_advances_2020,miller_reducing_2017,wang_joint_2024,fujisawa_multilevel_2021,buchholz_quasimonte_2018} and in theory~\citep{kim_linear_2024,xu_variance_2019,luu_stochastic_2025}.
Indeed, in our theoretical analysis, we will demonstrate that, once we use the same gradient estimator, the state-of-the-art iteration complexities of SPBWGD and SPGD become the same.

\section{Theoretical Analysis}

\subsection{Theoretical Setup}

For our theoretical analysis, we assume the following regularity conditions on $U$.

\begin{assumption}\label{assumption:potential_convex_smooth}
The potential $U : \mathbb{R}^d \to \mathbb{R}$ is twice differentiable and there exists some $\mu \in (0, +\infty)$ and $L \in [0, +\infty)$ such that, for all $z \in \mathbb{R}^d$, the following holds:
{%
\setlength{\belowdisplayskip}{1ex} \setlength{\belowdisplayshortskip}{1ex}
\setlength{\abovedisplayskip}{1.5ex} \setlength{\abovedisplayshortskip}{1.5ex}
\[
    \mu \mathrm{I}_d \quad\preceq\quad \nabla^2 U(z) \quad\preceq\quad L \mathrm{I}_d \; .
    \nonumber
\]
}%
\end{assumption}
\vspace{-1ex}

This assumption corresponds to assuming that the density of $\pi$ is $\mu$-log-concave and $L$-log-smooth, and has been widely used to establish the iteration complexity of stochastic gradient-based VI~\citep{kim_convergence_2023,domke_provable_2023,lambert_variational_2022,diao_forwardbackward_2023} and sampling algorithms~\citep{chewi_logconcave_2024}.
Crucially, the energy $\mathcal{E}$ is now well behaved: On the Bures--Wasserstein geometry, it is $\mu$-geodesically convex and $L$-geodesically smooth~\citep{diao_forwardbackward_2023}.
Similarly, under \cref{assumption:parametrization}, $\lambda \mapsto \mathcal{E}(q_{\lambda})$ is $\mu$-strongly convex and $L$-smooth~\citep{domke_provable_2020}.

For stochastic first-order optimization algorithms, the choice of step size \textit{schedule} is crucial for obtaining tight bounds~\citep{bach_nonasymptotic_2011}.
In all cases, we will consider a two-stage step size schedule~\citep{gower_sgd_2019,stich_unified_2019} of the form of, for some base step size $\gamma_0 \in (0, +\infty)$, switching time $t_* \geq 0$, and offset $\tau \geq 0$, 
{%
\setlength{\belowdisplayskip}{1.5ex} \setlength{\belowdisplayshortskip}{1.5ex}
\setlength{\abovedisplayskip}{1.5ex} \setlength{\abovedisplayshortskip}{1.5ex}
\[
    \gamma_t = \begin{cases}
        \gamma_0 & \text{if $t < t_*$} \\
        \frac{1}{\mu} \frac{2 (t + \tau) + 1}{ {(t + \tau + 1)}^2 } & \text{if $t \geq t_*$} \; .
        \label{eq:stepsize_schedule}
    \end{cases}
\]
}%
This two-stage schedule holds the step size constant for a certain period ($t \in \{0, \ldots, t_* - 1\}$) and then starts decreasing the step size at a rate of $\gamma_t \asymp 1/(\mu t)$.
The choice of asymptote $1/(\mu t)$ ensures an optimal asymptotic convergence rate of $\mathrm{O}(1/T)$ for strongly convex objectives~\citep{lacoste-julien_simpler_2012,shamir_stochastic_2013}.

\subsection{Main Results}
We now present the iteration-complexity guarantees for VI with SPBWGD and SPGD using Price's gradient for the scale/covariance component.
First, the Bonnet--Price estimator of the Bures--Wasserstein gradient is formally defined in functional form as, for any $q = \mathrm{Normal}(m, \Sigma)$,
{%
\setlength{\belowdisplayskip}{1.5ex} \setlength{\belowdisplayshortskip}{1.5ex}
\setlength{\abovedisplayskip}{1.5ex} \setlength{\abovedisplayshortskip}{1.5ex}
\[
    &
    \widehat{\nabla^{\text{Bonnet--Price}}_{\mathrm{BW}} \mathcal{E}}(q; \epsilon) 
    \nonumber
    \\
    &\qquad\triangleq \;
    x \mapsto \widehat{\nabla_m^{\text{bonnet}} \mathcal{E}}(q; \epsilon) + 2 \widehat{\nabla^{\text{price}}_{\Sigma} \mathcal{E}}(q; \epsilon)  (x - m)
    \nonumber
    \\
    &\qquad=\; 
    x \mapsto \nabla U\lt(Z\rt) + \nabla^2 U\lt(Z\rt) (x - m) \; ,
    \nonumber
\]
}%
where $Z = \mathrm{cholesky}(\Sigma) \epsilon + \mu$ and $\epsilon \sim \varphi = \mathrm{Normal}(0_d, \mathrm{I}_d)$.
Then we obtain the following iteration complexity:
\newpage

\begin{theorem}[name=SPBWGD, restate={[name=Restated]thmwassersteinproximalgradientdescentpricestein}]
    \label{thm:wasserstein_proximal_gradient_descent_pricestein}
    Suppose \cref{assumption:potential_convex_smooth} holds and the gradient estimator $\widehat{\nabla^{\text{\rm{}bonnet--price}}_{\mathrm{BW}} \mathcal{E}}$ is used.
    Then, for any $\epsilon > 0$, there exists some $t_*$ and $\tau$ (shown explicitly in the proof) such that running stochastic proximal Bures--Wasserstein gradient descent with the step size schedule in \cref{eq:stepsize_schedule} with $\gamma_0 = {1}/({10 L \kappa})$ guarantees 
{%
\setlength{\belowdisplayskip}{1.5ex} \setlength{\belowdisplayshortskip}{1.5ex}
\setlength{\abovedisplayskip}{1.5ex} \setlength{\abovedisplayshortskip}{1.5ex}
    \[
        T 
        &\gtrsim
            d \kappa \frac{1}{\epsilon}
            +
            \sqrt{d} \kappa^{3/2} \log\lt( \kappa \Delta^2  \rt) \frac{1}{\sqrt{\epsilon}}
            +
            \kappa^2 \log\lt( \Delta^2 \frac{1}{\epsilon} \rt)
        \nonumber
        \\
        &\quad\Rightarrow\quad
        \mu \mathbb{E}[{\mathrm{W}_2\lt(q_T, q_*\rt)}^2] \leq \epsilon \; ,
        \nonumber
    \]
}%
    where $\Delta^2 = \mu {\mathrm{W}\lt(q_0, q_*\rt)}^2$.
\end{theorem}
\vspace{-2ex}
\begin{proof}
    The full proof is deferred to \cref{section:proof_wasserstein_proximal_gradient_descent_steinprice}.
\end{proof}
\vspace{-2ex}

This improves over the $\mathrm{O}(d \kappa \epsilon^{-1} \log \epsilon^{-1} + \kappa^3 \log ( \Delta \epsilon^{-1} ) )$ complexity obtained by \citet[Thm 5.8]{diao_forwardbackward_2023}.
In particular, our result allows for step sizes larger by a factor of $\kappa$.
Consequently, the dependence on $\kappa$ in the non-asymptotic term ($\log 1/\epsilon$) is improved by a factor of $\kappa$.
Furthermore, the asymptotic complexity in $\epsilon \to 0$ is improved by a factor of $\log 1/\epsilon$. 

Let's now compare this result against the iteration complexity of SPGD.
Formally, we define the Bonnet--Price gradient estimator
{%
\setlength{\belowdisplayskip}{1.5ex} \setlength{\belowdisplayshortskip}{1.5ex}
\setlength{\abovedisplayskip}{1.5ex} \setlength{\abovedisplayshortskip}{1.5ex}
\[
    \widehat{\nabla_{\lambda}^{\text{Bonnet--Price}} \mathcal{E}}\lt(q_{\lambda}; \epsilon\rt)
    &\triangleq
    \begin{bmatrix}
         \widehat{\nabla_{m}^{\text{bonnet}} \mathcal{E}}\lt(q_{\lambda}; \epsilon\rt) \\
         \widehat{\nabla_{C}^{\text{price}} \mathcal{E}}\lt(q_{\lambda}; \epsilon\rt) 
    \end{bmatrix} 
    \nonumber
    =
    \begin{bmatrix}
         \nabla U\lt(Z\rt) \\
         C^{\top} \nabla^2 U\lt(Z\rt)
    \end{bmatrix} 
    \nonumber
    \; .
\]
}%
Using this estimator in BBVI with PSGD results in the following iteration complexity. 

\begin{theorem}[name=SPGD, restate={[name=Restated]thmparameterspaceproximalgradientdescentpricestein}]
    \label{thm:parameterspace_proximal_gradient_descent_pricestein}
    Suppose \cref{assumption:potential_convex_smooth} holds and the gradient estimator $\widehat{\nabla^{\text{\rm{}bonnet--price}}_{\lambda} \mathcal{E}}$ is used.
    Then, for any $\epsilon > 0$, there exists some $t_*$ and $\tau$ (stated explicitly in the proof) such that running stochastic proximal gradient descent with the step size schedule in \cref{eq:stepsize_schedule} with $\gamma_0 = 1/(10 L \kappa)$ guarantees 
{%
\setlength{\belowdisplayskip}{1.5ex} \setlength{\belowdisplayshortskip}{1.5ex}
\setlength{\abovedisplayskip}{1.5ex} \setlength{\abovedisplayshortskip}{1.5ex}    
    \[
        T 
        &\gtrsim
            d \kappa \frac{1}{\epsilon}
            +
            \sqrt{d} \kappa^{3/2} \log\lt( \kappa \Delta^2 \rt) \frac{1}{\sqrt{\epsilon}}
            +
            \kappa^2 \log\lt( \Delta^2 \frac{1}{\epsilon} \rt)
        \nonumber
        \\
        &\quad\Rightarrow\quad
        \mu \mathbb{E}[{\mathrm{W}_2\lt(q_T, q_*\rt)}^2] \leq \epsilon \; ,
        \nonumber
    \]
}%
    where $\Delta^2 = \mu \norm{\lambda_0 - \lambda_*}^2$.
\end{theorem}
\vspace{-2ex}
\begin{proof}
    The full proof is deferred to \cref{section:proof_parameterspace_proximal_gradient_descent_pricestein}.
\end{proof}
\vspace{-2ex}

Previously, for Gaussian variational families with a dense covariance (the ``full-rank'' Gaussian family), in the limit of $\epsilon \to 0$,  \citet{kim_convergence_2023,domke_provable_2023} reported an iteration complexity of  $\mathrm{O}(d \kappa^2  \operatorname{tr}(\mu \Sigma_*) \epsilon^{-1})$, which used the canonical reparametrization gradient (\cref{eq:reparametrization_gradient_location_scale}).
Compared to this, Price's gradient improves the iteration complexity by a factor of $\kappa \operatorname{tr}(\mu \Sigma_*)$. (Note that $d/L \leq \operatorname{tr}(\Sigma_*) \leq d/\mu$.)
This is comparable to the complexity of BBVI with the mean-field Gaussian family (diagonal covariance), which is $\mathrm{O}((\log d) \kappa^2   
\operatorname{tr}(\mu \Sigma_*)  \epsilon^{-1})$~\citep{kim_nearly_2025}.
This suggests that, with Price's gradient, BBVI on a full-rank Gaussian family can be as fast as using a mean-field Gaussian family and the reparametrization gradient.

An immediate implication of \cref{thm:wasserstein_proximal_gradient_descent_pricestein,thm:parameterspace_proximal_gradient_descent_pricestein} is that the gap between the best known iteration complexity bounds between the two algorithms has now been closed.
In addition, \cref{section:proof_sketch} that follows will explain that this resemblance is unsurprising, as the convergence analyses of both algorithms rely on nearly the same properties.
Though, since we lack matching lower bounds, we cannot yet claim that the two algorithms behave exactly the same.
However, our results do provide evidence towards this fact along with the continuous-time results of \citet{yi_bridging_2023,hoffman_blackbox_2020}.
This again reinforces the intuition that, for stochastic optimization algorithms, the quality of the gradient estimator has the largest impact on the performance.

\subsection{Proof Sketch}\label{section:proof_sketch}
The overall structure of the proofs for both SPBWGD and SPGD is identical.
If we had access to exact gradients instead of stochastic estimates, under \cref{assumption:potential_convex_smooth}, $\norm{\lambda_t - \lambda_*}_2$ or $\mathrm{W}_2\lt(q_t, q_*\rt)$ would contract exponentially in $t$.
When dealing with stochastic gradients, however, the noise in the estimates perturbs the iterates.
We thus need to show that the variance of the noise is bounded and the contraction is strong enough such that controlling the step size schedule $\gamma_t$ can neutralize the perturbations.

First, under \cref{assumption:potential_convex_smooth}, we can define a Bregman divergence associated with $U$,
{%
\setlength{\belowdisplayskip}{1.5ex} \setlength{\belowdisplayshortskip}{1.5ex}
\setlength{\abovedisplayskip}{1.5ex} \setlength{\abovedisplayshortskip}{1.5ex} 
\[
    \mathrm{D}_U\lt(x, y\rt)
    \triangleq
    U\lt(x\rt) - U\lt(y\rt) - \inner{ \nabla U(y), x - y } \; .
    \nonumber
\]
}%
For both SPBWGD and SPGD, we establish gradient variance bounds involving $\mathrm{D}_U$.

\begin{lemma}[restate={[name=Restated]thmwassersteingradientvariancebound}]
\label{thm:wasserstein_gradient_variance_bound}
Suppose \cref{assumption:potential_convex_smooth} holds and $q_* = \mathrm{Normal}(\mu_*, \Sigma_*) \in \argmin_{q \in \mathrm{BW}(\mathbb{R}^d)} \mathcal{F}\lt(q\rt)$.
Then, for any $q \in \mathrm{BW}(\mathbb{R}^d)$, and any coupling $\psi \in \Psi\lt(q, q_*\rt)$,
{%
\setlength{\belowdisplayskip}{0.5ex} \setlength{\belowdisplayshortskip}{0.5ex}
\setlength{\abovedisplayskip}{1.5ex} \setlength{\abovedisplayshortskip}{1.5ex} 
\[
    &\mathbb{E}_{(X, X_*) \sim \psi, \epsilon \sim \varphi}\big[ 
    \norm{\widehat{\nabla^{\text{\rm{}bonnet--price}}_{\mathrm{BW}} \mathcal{E}}\lt(q; \epsilon\rt)\lt(X\rt) - \nabla \mathcal{E}\lt(q_*\rt)\lt(X_*\rt)}_2^2
    \big] 
    \nonumber
    \\
    &\qquad\qquad\qquad\leq
    10 L \kappa \, \mathbb{E}_{ (X, X_*) \sim \psi } \lt[ \mathrm{D}_{U}\lt(X, X_*\rt) \rt]
    +
    10 d L 
    \nonumber
    \; .
\]
}%
\end{lemma}
\vspace{-2ex}
\begin{proof}
    The proof is deferred to \cref{section:proof_wasserstein_gradient_variance_bound}.
\end{proof}
\vspace{-2ex}

This is a refinement of Lemma 5.6 by~\citet{diao_forwardbackward_2023}.
Specifically, instead of upper-bounding the gradient variance with the squared Wasserstein distance ${\mathrm{W}_2\lt(q, q_*\rt)}^2$, we bound it with the Bregman divergence $\mathrm{D}_U$.
In fact, for the optimal coupling $\psi^* \in \Psi(q, q_*)$, we have
{%
\setlength{\belowdisplayskip}{1.5ex} \setlength{\belowdisplayshortskip}{1.5ex}
\setlength{\abovedisplayskip}{1.5ex} \setlength{\abovedisplayshortskip}{1.5ex}
\[
    \mathbb{E}_{(X,X_*) \sim \psi^*}[\mathrm{D}_{U}\lt(X, X_*\rt)] \leq L {\mathrm{W}_2\lt(q, q_*\rt)}^2 
    \nonumber
\]
}%
(\cref{eq:bregman_divergence_convex_smooth} in \cref{section:bregman_properties}.)
Therefore, the use of the Bregman divergence avoids paying for an extra factor of $\kappa = L/\mu$.
The corresponding bound for parameter-space SPGD has the exact same constants.

\begin{lemma}[restate={[name=Restated]thmparametergradientvariancebound}]
\label{thm:parameter_gradient_variance_bound}
Suppose \cref{assumption:potential_convex_smooth,assumption:parametrization} holds, and $\lambda_* \in \argmin_{\lambda \in \Lambda} \mathcal{F}\lt(q_{\lambda}\rt)$.
Then, for any $\lambda \in \Lambda$ and any coupling $\psi \in \Psi\lt(q_{\lambda}, q_{\lambda_*}\rt)$, 
{%
\setlength{\belowdisplayskip}{0.5ex} \setlength{\belowdisplayshortskip}{0.5ex}
\setlength{\abovedisplayskip}{1.5ex} \setlength{\abovedisplayshortskip}{1.5ex}
\[
    &
    \mathbb{E}_{\epsilon \sim \varphi}\big[ 
    \norm{\widehat{\nabla_{\lambda}^{\text{\rm{}bonnet--price}} \mathcal{E}}\lt(q_{\lambda}; \epsilon\rt) - \nabla \mathcal{E}\lt(q_{\lambda_*}\rt)}_2^2
    \big] 
    \nonumber
    \\
    &\qquad\qquad\leq
    10 L \kappa \, \mathbb{E}_{(X, X_*) \sim \psi}\lt[ \mathrm{D}_{U}\lt(X, X_*\rt)
    \rt]
    +
    10 d L 
    \; .
    \nonumber
\]
}%
\end{lemma}
\vspace{-2ex}
\begin{proof}
    The proof is deferred to \cref{section:proof_parameter_gradient_variance_bound}.
\end{proof}
\vspace{-2ex}

The bounds \cref{thm:wasserstein_gradient_variance_bound,thm:parameter_gradient_variance_bound}, however, are not immediately usable for a convergence analysis.
The ``growth'' or ``multiplicative noise'' term $\mathbb{E}[\mathrm{D}_{U}\lt(X, X_*\rt)]$ needs to be related to the growth of $\mathcal{E}$.
Notice that both \cref{thm:wasserstein_gradient_variance_bound,thm:parameter_gradient_variance_bound} hold for any coupling in $\Psi\lt(q, q_*\rt)$.
By specifying the coupling $\psi$, we will invoke properties of the geometry associated with $\psi$ induced by each SPGD and SPBWGD, which will allow us to relate $\mathbb{E}[\mathrm{D}_{U}\lt(X, X_*\rt)]$ with the appropriate notion of growth of $\mathcal{E}$.

Let's start with SPBWGD.
For the optimal coupling $\psi^* \in \Psi(q, q_*)$, we will define the Wasserstein analog of the Bregman divergence
{%
\setlength{\belowdisplayskip}{1.5ex} \setlength{\belowdisplayshortskip}{1.5ex}
\setlength{\abovedisplayskip}{1.5ex} \setlength{\abovedisplayshortskip}{1.5ex}
\[
    &\mathrm{D}_{\mathcal{E}}\lt(q, q_*\rt)
    \triangleq
    \mathcal{E}\lt(q\rt) - \mathcal{E}\lt(q_*\rt)
    \nonumber
    \\
    &\qquad\quad
    - \mathbb{E}_{(X, X_*) \sim \psi_*}\inner{ \nabla \mathcal{E}(q_*)(X_*), X - X_* } 
    \geq 0
    \label{eq:energy_wasserstein_bregman}
    \; .
\]
}%
The non-negativity of $\mathrm{D}_{\mathcal{E}}$ follows from the fact that $\mathcal{E}$ is geodesically convex under \cref{assumption:potential_convex_smooth} (\cref{thm:energy_strongly__geodesicallyconvex}).
Then the Bures--Wasserstein geometry yields the following:
\begin{lemma}[restate={[name=Restated]thmwassersteinbregmanequivalence}]\label{thm:wasserstein_bregman_equivalence}
For any $p, q \in \mathrm{BW}(\mathbb{R}^d)$, denote the coupling optimal for the squared Euclidean cost between $p$ and $q$ as $\psi_* \in \Psi\lt(p, q\rt)$.
Then
{%
\setlength{\belowdisplayskip}{0ex} \setlength{\belowdisplayshortskip}{0ex}
\setlength{\abovedisplayskip}{1.0ex} \setlength{\abovedisplayshortskip}{1.0ex}
\[
    \mathbb{E}_{(X,Y) \sim \psi_*}\lt[\mathrm{D}_U\lt(X, Y\rt)\rt]
    =
    \mathrm{D}_{\mathcal{E}}\lt(p, q\rt) \; .
    \nonumber
\]
}%
\end{lemma}
\vspace{-2ex}
\begin{proof}
    The proof is deferred to \cref{section:proof_wasserstein_bregman_equivalence}.
\end{proof}
\vspace{-2ex}
For SPGD, first consider some $\lambda, \lambda' \in \Lambda$, where $\lambda = (m, \operatorname{vec} C)$ and $\lambda' = (m', \operatorname{vec} C')$ and the map
{%
\setlength{\belowdisplayskip}{1.5ex} \setlength{\belowdisplayshortskip}{1.5ex}
\setlength{\abovedisplayskip}{1.5ex} \setlength{\abovedisplayshortskip}{1.5ex}
\[
    M^{\mathrm{rep}}_{q_{\lambda} \mapsto q_{\lambda'}}(z) \triangleq  C' C^{-1} (z - m) + m' \; .
    \nonumber
\]
}%
This is, in fact, a transport map from $q_{\lambda}$ to $q_{\lambda'}$.
Then, under \cref{assumption:parametrization}, the identities 
{%
\setlength{\belowdisplayskip}{0.5ex} \setlength{\belowdisplayshortskip}{0.5ex}
\setlength{\abovedisplayskip}{1.5ex} \setlength{\abovedisplayshortskip}{1.5ex}
\[
    \norm{\lambda - \lambda'}_2^2
    &=
    \mathbb{E}_{Z \sim q}\norm{ Z - M^{\mathrm{rep}}_{q_{\lambda} \mapsto q_{\lambda'}}(Z) }_2^2  
    \label{eq:parameter_distance_coupling_distance_equivalence}
\]
}%
and
{%
\setlength{\belowdisplayskip}{1.5ex} \setlength{\belowdisplayshortskip}{1.5ex}
\setlength{\abovedisplayskip}{0.5ex} \setlength{\abovedisplayshortskip}{0.5ex}
\[
    &
    \inner{ \nabla_{\lambda} \mathcal{E}(q_{\lambda'}) , \lambda - \lambda'}
    \nonumber
    \\
    &\qquad=
    \mathbb{E}_{Z' \sim q_{\lambda'}} \inner{ \nabla U\lt(Z'\rt), M^{\mathrm{rep}}_{q_{\lambda'} \mapsto q_{\lambda}}(Z') - Z'} 
    \label{eq:parameter_inner_product_coupling_inner_product_equivalence}
\]
}%
hold.
Also, from the fact that the Wasserstein distance is the cost of the optimal coupling, we have the ordering $\norm{\lambda - \lambda'}_2 \geq {\mathrm{W}_2\lt(q_{\lambda}, q_{\lambda'}\rt)}$.
That is, the metric associated with parameter space is a coupling-based distance measure.
Now, under \cref{assumption:parametrization,assumption:potential_convex_smooth}, $\lambda \mapsto \mathcal{F}(q_{\lambda})$ is $\mu$-strongly convex.
Then it is well known that the gradient flow minimizes $\norm{\lambda_t - \lambda_*}_2^2$ exponentially in time.
The identity \cref{eq:parameter_distance_coupling_distance_equivalence} implies that this flow is also minimizing a coupling distance, and in turn the Wasserstein distance.
Back to the proof sketch, \cref{eq:parameter_inner_product_coupling_inner_product_equivalence} implies the following:

\begin{lemma}[restate={[name=Restated]thmparameterbregmanequivalence}]\label{thm:parameter_bregman_equivalence}
Suppose \cref{assumption:parametrization} hold.
Then, for any $\lambda, \lambda' \in \Lambda$, denote the coupling induced by the transport map $M^{\mathrm{rep}}_{q_{\lambda} \mapsto q_{\lambda'}}$ as $\psi^{\mathrm{rep}}$.
Then
{%
\setlength{\belowdisplayskip}{1.5ex} \setlength{\belowdisplayshortskip}{1.5ex}
\setlength{\abovedisplayskip}{1.5ex} \setlength{\abovedisplayshortskip}{1.5ex}
\[
    \mathbb{E}_{(X,X') \sim \psi^{\mathrm{rep}}}\lt[\mathrm{D}_U\lt(X, X'\rt)\rt]
    =
    \mathrm{D}_{\lambda \mapsto \mathcal{E}(q_{\lambda})}\lt(\lambda, \lambda'\rt) \; .
    \nonumber
\]
}%
\end{lemma}
\vspace{-2ex}
\begin{proof}
    The proof is deferred to \cref{section:proof_parameter_bregman_equivalence}.
\end{proof}
\vspace{-2ex}

Therefore, the optimal coupling $\psi^*$ and the coupling associated with \cref{assumption:parametrization} $\psi^{\mathrm{rep}}$ respectively retrieve a relationship with the growth of $\mathcal{E}$.

Lastly, we need to show that the contraction is able to counteract the ``growth'' of the gradient variance.
Typically, the contraction of first-order methods follows from the coercivity of the gradient operator.
For our proof, however, instead of obtaining a full contraction, we establish a slight generalization of coercivity (previously developed by~\citealt{gorbunov_unified_2020}), that allow us to control the Bregman terms $\mathrm{D}_{\lambda \mapsto \mathcal{E}(q_{\lambda})}$ and $\mathrm{D}_{\mathcal{E}}$.
For SPBWGD, our result reads:

\begin{lemma}[restate={[name=Restated]thmwassersteincoercivity}]\label{thm:wasserstein_coercivity}
    Suppose \cref{assumption:potential_convex_smooth} holds.
    Then, for any $q \in \mathrm{BW}(\mathbb{R}^d)$ and $q_* = \argmin_{q \in \mathrm{BW}(\mathbb{R}^d)} \mathcal{F}(q)$, where we denote their coupling $\psi_* \in \Psi(q, q_*)$ optimal in terms of squared Euclidean distance,
    {%
\setlength{\belowdisplayskip}{0.5ex} \setlength{\belowdisplayshortskip}{0.5ex}
\setlength{\abovedisplayskip}{1.5ex} \setlength{\abovedisplayshortskip}{1.5ex}
    \[
        &
        \mathbb{E}_{(X,Y) \sim \psi_*} \inner{ \nabla_{\mathrm{BW}} \mathcal{E}\lt(p\rt)\lt(X\rt) - \nabla_{\mathrm{BW}} \mathcal{E}\lt(q\rt)\lt(Y\rt) , X - Y }
        \nonumber
        \\
        &\qquad\qquad\qquad\qquad\geq
        \frac{\mu}{2} \, {\mathrm{W}_2\lt(q_t, q_*\rt)}^2 
        +
        \mathrm{D}_{\mathcal{E}}\lt(q_t, q_*\rt)
        \; .
        \nonumber
    \]
}%
\end{lemma}
\vspace{-2ex}
\begin{proof}
    The proof is deferred to \cref{section:proof_wasserstein_coercivity}.
\end{proof}
\vspace{-2ex}

The corresponding result for parameter space SPGD is:

\begin{lemma}[restate={[name=Restated]thmparametercoercivity}]\label{thm:parameter_coercivity}
    Suppose \cref{assumption:potential_convex_smooth,assumption:parametrization} hold.
    Then, for any $\lambda \in \Lambda$ and $\lambda_* = \argmin_{\lambda \in \Lambda} \mathcal{F}(q_{\lambda})$,
{%
\setlength{\belowdisplayskip}{0.5ex} \setlength{\belowdisplayshortskip}{0.5ex}
\setlength{\abovedisplayskip}{1.5ex} \setlength{\abovedisplayshortskip}{1.5ex}
    \[
        &
        \inner*{ \nabla_{\lambda_t} \mathcal{E}\lt(q_{\lambda_t}\rt)  - \nabla_{\lambda_*} \mathcal{E}\lt(q_{\lambda_*}\rt), \lambda_t - \lambda_* }
        \nonumber
        \\
        &\qquad\qquad\qquad\geq
        \frac{\mu}{2} \norm{\lambda - \lambda_*}_2^2
        +
        \mathrm{D}_{\lambda \mapsto \mathcal{E}(q_{\lambda})}\lt(\lambda, \lambda_*\rt) \; .
        \nonumber
    \]
}%
\end{lemma}
\vspace{-2ex}
\begin{proof}
    The proof is deferred to \cref{section:proof_parameter_coercivity}.
\end{proof}
\vspace{-2ex}

The extra ``Bregman term'' on the right-hand sides directly allows control over the Bregman term in the gradient variance bounds.

\begin{figure*}[t]
    \centering
    \subfloat[Dogs]{
        \includegraphics[scale=0.8]{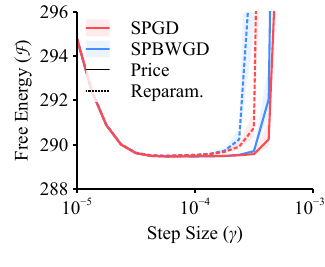}
        \vspace{-2ex}
    }
    \subfloat[Basketball]{
        \includegraphics[scale=0.8]{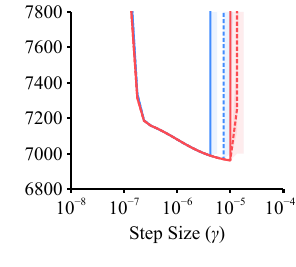}
        \vspace{-2ex}
    }
    \subfloat[Lynx]{
        \includegraphics[scale=0.8]{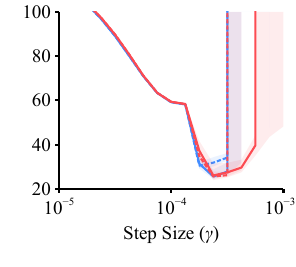}
        \vspace{-2ex}
    }
    \subfloat[NES2000]{
        \includegraphics[scale=0.8]{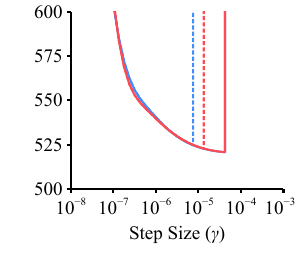}
        \vspace{-2ex}
    }
    \\
    \subfloat[Bones]{
        \includegraphics[scale=0.8]{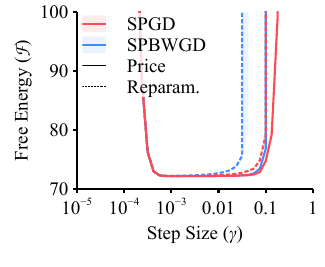}
        \vspace{-2ex}
    }
    \subfloat[SISLOB]{
        \includegraphics[scale=0.8]{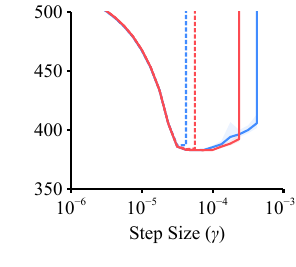}
        \vspace{-2ex}
    }
    \subfloat[pilots]{
        \includegraphics[scale=0.8]{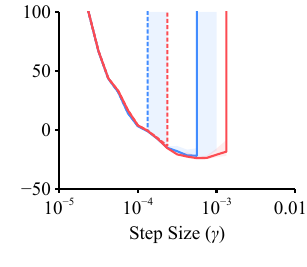}
        \vspace{-2ex}
    }
    \subfloat[Downloads]{
        \includegraphics[scale=0.8]{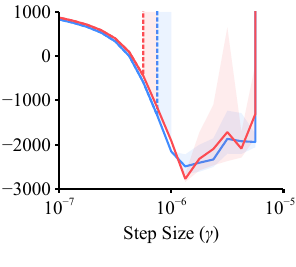}
        \vspace{-2ex}
    }
    \\
    \subfloat[rats]{
        \includegraphics[scale=0.8]{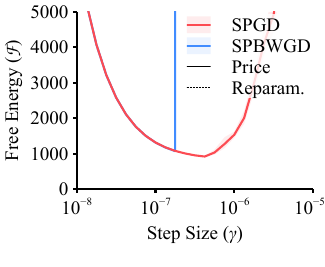}
        \vspace{-2ex}
    }
    \subfloat[Radon]{
        \includegraphics[scale=0.8]{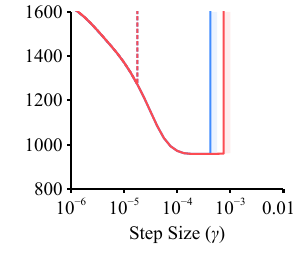}
        \vspace{-2ex}
    }
    \subfloat[Butterfly]{
        \includegraphics[scale=0.8]{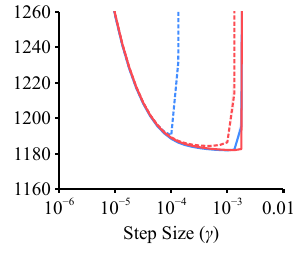}
        \vspace{-2ex}
    }
    \subfloat[Birds]{
        \includegraphics[scale=0.8]{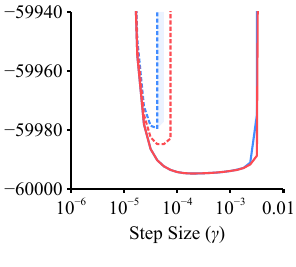}
        \vspace{-2ex}
    }
    \caption{
    \textbf{
        Variational free energy ($\mathcal{F}$) at $T = 4000$ versus step size $\gamma$.
    }
    Additional results for different stopping times $T$ can be found in \cref{section:additional_results} of the Appendix.
    Note that the dotted line is missing on Rats because all the corresponding runs diverged.
    The solid lines are the mean estimated over 32 independent repetitions, while the shaded regions are the 95\% bootstrap confidence intervals.
    }
    \label{fig:envelopes}
    \vspace{-1ex}
\end{figure*}

\section{Empirical Analysis}\label{section:experiments}

For the empirical evaluation, we will compare the performance of VI with SPGD and SPBWGD with the reparametrization and Price estimators.

\vspace{-1ex}
\paragraph{Setup and Implementation.}
For SPGD and SPBWGD, we follow the implementation described in \cref{section:bbvi_spgd} and \cref{section:spbwgd}, respectively.
As mentioned in \cref{section:spbwgd}, the reparamtrization gradient $\widehat{\nabla_{\Sigma}^{\text{rep}} \mathcal{E}}$ is not almost surely symmetric.
Therefore, the modified SPBWGD implementation in \cref{section:spbwgd} is used.
All algorithms were implemented in the Julia language~\citep{bezanson_julia_2017} and the AdvancedVI.jl library (\texttt{v0.6.1}), which is part of the Turing probabilistic programming ecosystem~\citep{ge_turing_2018,fjelde_turingjl_2025}
\footnote{All of the code needed to reproduce the results is available in the following repository: \url{https://github.com/Red-Portal/sgvi_second_order_gradient_estimators.git}}.
The experimental problems were taken from the PosteriorDB~\citep{magnusson_posteriordb_2025} benchmark suite of Stan models~\citep{carpenter_stan_2017}, which was made accessible from Julia through the BridgeStan interface~\citep{roualdes_bridgestan_2023}.
The benchmark problems are described in more detail in \cref{section:benchmark_problems}.
All methods are initialized at $q_0 = \mathrm{Normal}(0_d, 0.34 \, \mathrm{I}_d)$.
The gradients were estimated using 8 Monte Carlo samples in all cases, while for evaluation, $\mathcal{F}$ was estimated using $2^{12}$ samples.
We run both SPGD and SPBWGD with a fixed step size $\gamma$ and estimate the free energy $\mathcal{F}(q_t)$ of the iterate $q_t$ at each iteration.

\vspace{-1ex}
\paragraph{Results.}
Part of the results are shown in \cref{fig:envelopes}, while the full set of results can be found in \cref{section:additional_results}.
First, when using Price's gradient (solid line), both SPGD and SPBWGD achieve similar performance, except for Rats, where SPGD remains stable over a wider range of step sizes.
On the other hand, when using the reparametrization gradient, both SPGD and SPBWGD perform poorly: they require smaller step sizes.
In fact, on Rats, none of the methods using the reparametrization gradient converged for all step sizes between $10^{-8}$ and $10^{0}$.
In addition, on some problems, SPBWGD appears to perform worse than SPGD.
For example, on Rats with Price's gradient, SPBWGD requires step sizes orders of magnitude smaller to prevent divergence.
A possible explanation is that SPBWGD requires an estimator of $\nabla_{\Sigma}\mathcal{E}(q)$, whereas SPGD uses an estimator for $\nabla_{C}\mathcal{E}(q)$.
These two are related through the chain rule as $(1/2) C^{-\top} \nabla_{\Sigma}\mathcal{E}(q) = \nabla_{C}\mathcal{E}(q)$; the extra scaling of $C^{-\top}$ could make an estimator for $\nabla_{\Sigma}\mathcal{E}(q)$ noisier than $\nabla_{C}\mathcal{E}(q)$.

\vspace{-1ex}
\section{Discussions}
\vspace{-1ex}

In this work, we theoretically analyzed stochastic gradient-based VI algorithms operating in the Euclidean space of parameters (SPGD) and Bures--Wasserstein space (SPBWGD).
Our results improve upon the state-of-the-art complexity guarantees for both, closing the gap between them.
For SPBWGD, we have technically improved the previous results by \citet{diao_forwardbackward_2023}.
Meanwhile, for SPGD, we have shown that the use of the Price gradient $\widehat{\nabla_{C}^{\text{price}}} \mathcal{E}$ achieves better theoretical guarantees than those obtained~\citep{domke_provable_2023,kim_convergence_2023} under the reparametrization gradient $\widehat{\nabla_{C}^{\text{rep}} \mathcal{E}}$.
This shows that the previously observed advantage of SPBWGD was due to the choice of gradient estimator rather than the geometry.

However, this doesn't completely rule out the possibility that measure-space algorithms can be more effective.
NGVI, which uses the Fisher metric, yields a preconditioned update to the location parameter $m_t$ reminiscent of Newton's method~\citep{khan_bayesian_2023}.
Possibly due to this, empirical evidence suggests that NGVI methods can converge significantly faster than BBVI~\citep{lin_fast_2019}.
However, our theoretical understanding of NGVI is still limited, where existing analyses assume conjugacy~\citep{wu_understanding_2024} or assumptions much stronger than those considered in this work~\citep{sun_natural_2025,kumar_optimization_2025}.

An interesting theoretical aspect of Price's gradient is that its variance becomes zero when $\pi$ is Gaussian. (The Hessian is constant.)
But this happens only for the scale $C_t$ or covariance $\Sigma_t$ component of the Bonnet--Price gradient, and the location component $m_t$ is still noisy.
Fortunately, this can be addressed by applying the control variate of~\citet{luu_stochastic_2025} with the coefficient $c_k$ set as $c_k = 1$, which enables the ``interpolation condition''~\citep{vaswani_fast_2019} (zero variance on the optimum of $\mathcal{F}$).
For Gaussian targets, this would imply linear convergence~\citep{schmidt_fast_2013,kim_linear_2024,domke_provable_2023}.

On a practical note, it isn't clear if Price's gradient is always better.
For instance, at each iteration, SPGD with the reparametrization gradient requires $\Omega(d^2)$ operations (matrix-vector product for computing $\phi_{\lambda}$).
Moving to second-order increases the cost to $\Omega(d^3)$ operations (matrix-matrix product for computing $C \nabla^2 U$).
Meanwhile, SPBWGD requires $\Omega(d^3)$ in both cases.
Therefore, when $\nabla U$ can be computed in $\mathrm{O}(d^2)$ time, as $d \to \infty$, BBVI with SPGD and the reparametrization gradient could be more efficient ($\Theta(d^2)$ versus $\Theta(d^3)$) depending on the conditioning $\kappa$.
In addition, WVI requires numerically sensitive operations, such as matrix square roots and Cholesky decompositions, which makes it less robust.

\clearpage
\section*{Acknowledgements}
The authors sincerely thank Kaiwen Wu for bringing the second-order gradient estimators to our attention, Jason Altschuler for suggesting that we look into Bures--Wasserstein variational inference methods, Andre Wibisono for providing valuable comments, and the anonymous reviewers for constructive comments.

Y.-A. Ma was supported by the NSF award CCF-2112665 (TILOS) and partly by the CDC-RFA-FT-23-0069 from the CDC’s Center for Forecasting and Outbreak Analytics;
T. Campbell was supported by an NSERC Discovery Grant;

\section*{Impact Statement}
This work studies the theoretical properties of variational inference, which is a collection of algorithms for approximating distributions. 
Therefore, our work is not expected to have direct societal consequences other than those of downstream applications of variational inference.

\bibliographystyle{icml2026}
\bibliography{references}

@inproceedings{agrawal_advances_2020,
  title = {Advances in Black-Box {{VI}}: {{Normalizing}} Flows, Importance Weighting, and Optimization},
  shorttitle = {Advances in Black-Box {{VI}}},
  booktitle = {Advances in {{Neural Information Processing Systems}}},
  author = {Agrawal, Abhinav and Sheldon, Daniel R and Domke, Justin},
  year = 2020,
  volume = {33},
  pages = {17358--17369},
  publisher = {Curran Associates, Inc.},
  urldate = {2022-12-14}
}

@techreport{ali_tagging_2019,
  type = {Stan {{Case Study}}},
  title = {Tagging {{Basketball Events}} with {{HMM}} in {{Stan}}},
  author = {Ali, Imad},
  year = 2019
}

@inproceedings{altschuler_averaging_2021,
  title = {Averaging on the {{Bures-Wasserstein}} Manifold: Dimension-Free Convergence of Gradient Descent},
  shorttitle = {Averaging on the {{Bures-Wasserstein}} Manifold},
  booktitle = {Advances in {{Neural Information Processing Systems}}},
  author = {Altschuler, Jason and Chewi, Sinho and Gerber, Patrik R and Stromme, Austin},
  year = 2021,
  volume = {34},
  pages = {22132--22145},
  publisher = {Curran Associates, Inc.},
  urldate = {2026-01-15}
}

@book{ambrosio_gradient_2005,
  title = {Gradient Flows: In Metric Spaces and in the Space of Probability Measures},
  shorttitle = {Gradient Flows},
  author = {Ambrosio, Luigi and Gigli, Nicola and Savar{\'e}, Giuseppe},
  year = 2005,
  series = {Lectures in Mathematics {{ETH Z\"urich}}},
  publisher = {Birkh\"auser},
  address = {Boston},
  langid = {english},
  lccn = {515.42}
}

@article{article,
  title = {{{JAGS}}: {{A}} Program for Analysis of Bayesian Graphical Models Using Gibbs Sampling},
  author = {Plummer, Martyn},
  year = 2003,
  journal = {3rd International Workshop on Distributed Statistical Computing (DSC 2003); Vienna, Austria},
  volume = {124}
}

@inproceedings{bach_nonasymptotic_2011,
  title = {Non-Asymptotic Analysis of Stochastic Approximation Algorithms for Machine Learning},
  booktitle = {Advances in {{Neural Information Processing Systems}}},
  author = {Bach, Francis and Moulines, Eric},
  year = 2011,
  volume = {24},
  pages = {451--459},
  publisher = {Curran Associates, Inc.},
  urldate = {2023-07-16}
}

@techreport{bales_selecting_2019,
  type = {{{ArXiv}} Preprint},
  title = {Selecting the {{Metric}} in {{Hamiltonian Monte Carlo}}},
  author = {Bales, Ben and Pourzanjani, Arya and Vehtari, Aki and Petzold, Linda},
  year = 2019,
  number = {arXiv:1905.11916},
  eprint = {1905.11916},
  primaryclass = {stat},
  institution = {arXiv},
  urldate = {2025-01-29},
  archiveprefix = {arXiv}
}

@article{bansal_potentialfunction_2019,
  title = {Potential-Function Proofs for Gradient Methods},
  author = {Bansal, Nikhil and Gupta, Anupam},
  year = 2019,
  journal = {Theory of Computing},
  volume = {15},
  number = {1},
  pages = {1--32},
  urldate = {2024-03-02},
  langid = {english}
}

@article{beckner_generalized_1989,
  title = {A {{Generalized Poincare Inequality}} for {{Gaussian Measures}}},
  author = {Beckner, William},
  year = 1989,
  journal = {Proceedings of the American Mathematical Society},
  volume = {105},
  number = {2},
  eprint = {2046956},
  eprinttype = {jstor},
  pages = {397--400},
  urldate = {2026-01-24}
}

@inproceedings{bernton_langevin_2018,
  title = {Langevin {{Monte Carlo}} and {{JKO}} Splitting},
  booktitle = {Proceedings of the {{Conference On Learning Theory}}},
  author = {Bernton, Espen},
  year = 2018,
  series = {{{PMLR}}},
  volume = {75},
  pages = {1777--1798},
  publisher = {JMLR},
  urldate = {2024-04-03},
  langid = {english}
}

@article{bezanson_julia_2017,
  title = {Julia: {{A}} Fresh Approach to Numerical Computing},
  author = {Bezanson, Jeff and Edelman, Alan and Karpinski, Stefan and Shah, Viral B},
  year = 2017,
  journal = {SIAM review},
  volume = {59},
  number = {1},
  pages = {65--98}
}

@article{bhatia_bures_2019,
  title = {On the {{Bures}}--{{Wasserstein}} Distance between Positive Definite Matrices},
  author = {Bhatia, Rajendra and Jain, Tanvi and Lim, Yongdo},
  year = 2019,
  journal = {Expositiones Mathematicae},
  volume = {37},
  number = {2},
  pages = {165--191},
  urldate = {2026-01-18}
}

@article{bingham_pyro_2019,
  title = {Pyro: {{Deep}} Universal Probabilistic Programming},
  shorttitle = {Pyro},
  author = {Bingham, Eli and Chen, Jonathan P. and Jankowiak, Martin and Obermeyer, Fritz and Pradhan, Neeraj and Karaletsos, Theofanis and Singh, Rohit and Szerlip, Paul and Horsfall, Paul and Goodman, Noah D.},
  year = 2019,
  journal = {Journal of Machine Learning Research},
  volume = {20},
  number = {28},
  pages = {1--6},
  urldate = {2022-12-19}
}

@article{blei_variational_2017,
  title = {Variational Inference: {{A}} Review for Statisticians},
  shorttitle = {Variational {{Inference}}},
  author = {Blei, David M. and Kucukelbir, Alp and McAuliffe, Jon D.},
  year = 2017,
  journal = {Journal of the American Statistical Association},
  volume = {112},
  number = {518},
  pages = {859--877},
  urldate = {2021-05-09},
  langid = {english}
}

@article{bonnet_transformations_1964,
  title = {{Transformations des signaux al\'eatoires a travers les syst\`emes non lin\'eaires sans m\'emoire}},
  author = {Bonnet, Georges},
  year = 1964,
  journal = {Annales des T\'el\'ecommunications},
  volume = {19},
  number = {9},
  pages = {203--220},
  urldate = {2026-01-07},
  langid = {french}
}

@incollection{bottou_online_1999,
  title = {On-Line Learning and Stochastic Approximations},
  booktitle = {On-{{Line Learning}} in {{Neural Networks}}},
  author = {Bottou, L{\'e}on},
  year = 1999,
  edition = {1},
  pages = {9--42},
  publisher = {Cambridge University Press},
  urldate = {2021-05-20}
}

@article{bottou_optimization_2018,
  title = {Optimization Methods for Large-Scale Machine Learning},
  author = {Bottou, L{\'e}on and Curtis, Frank E. and Nocedal, Jorge},
  year = 2018,
  journal = {SIAM Review},
  volume = {60},
  number = {2},
  pages = {223--311},
  urldate = {2019-07-02},
  langid = {english}
}

@incollection{brascamp_extensions_2002,
  title = {On Extensions of the {{Brunn-Minkowski}} and {{Pr\'ekopa-Leindler}} Theorems, Including Inequalities for Log Concave Functions, and with an Application to the Diffusion Equation},
  booktitle = {Inequalities: {{Selecta}} of {{Elliott H}}. {{Lieb}}},
  author = {Brascamp, Herm Jan and Lieb, Elliott H.},
  year = 2002,
  pages = {441--464},
  publisher = {Springer},
  address = {Berlin, Heidelberg},
  urldate = {2026-01-20},
  langid = {english}
}

@article{brenier_polar_1991,
  title = {Polar Factorization and Monotone Rearrangement of Vector-Valued Functions},
  author = {Brenier, Yann},
  year = 1991,
  journal = {Communications on Pure and Applied Mathematics},
  volume = {44},
  number = {4},
  pages = {375--417},
  urldate = {2026-01-14},
  copyright = {Copyright \copyright{} 1991 Wiley Periodicals, Inc., A Wiley Company},
  langid = {english}
}

@inproceedings{buchholz_quasimonte_2018,
  title = {Quasi-{{Monte Carlo}} Variational Inference},
  booktitle = {Proceedings of the {{International Conference}} on {{Machine Learning}}},
  author = {Buchholz, Alexander and Wenzel, Florian and Mandt, Stephan},
  year = 2018,
  series = {{{PMLR}}},
  volume = {80},
  pages = {668--677},
  publisher = {JMLR},
  urldate = {2022-12-19},
  langid = {english}
}

@article{bures_extension_1969,
  title = {An Extension of {{Kakutani}}'s Theorem on Infinite Product Measures to the Tensor Product of Semifinite $w^{*}$-Algebras},
  author = {Bures, Donald},
  year = 1969,
  journal = {Transactions of the American Mathematical Society},
  volume = {135},
  eprint = {1995012},
  eprinttype = {jstor},
  pages = {199--212},
  publisher = {American Mathematical Society},
  urldate = {2026-01-18}
}

@techreport{carpenter_predatorprey_2018,
  type = {Stan {{Case Study}}},
  title = {Predator-Prey Population Dynamics: The {{Lotka-Volterra}} Model in Stan},
  author = {Carpenter, Bob},
  year = 2018
}

@article{carpenter_stan_2017,
  title = {Stan: {{A}} Probabilistic Programming Language},
  shorttitle = {{\emph{Stan}}},
  author = {Carpenter, Bob and Gelman, Andrew and Hoffman, Matthew D. and Lee, Daniel and Goodrich, Ben and Betancourt, Michael and Brubaker, Marcus and Guo, Jiqiang and Li, Peter and Riddell, Allen},
  year = 2017,
  journal = {Journal of Statistical Software},
  volume = {76},
  number = {1},
  pages = {1--32},
  urldate = {2020-07-23},
  langid = {english}
}

@book{chewi_logconcave_2024,
  title = {Log-Concave Sampling},
  author = {Chewi, Sinho},
  year = 2024,
  edition = {November 3, 2024},
  publisher = {Unpublished draft},
  url = {https://chewisinho.github.io/main.pdf},
  urldate = {2024-04-11}
}

@book{chewi_statistical_2025,
  title = {Statistical Optimal Transport: {{\'Ecole}} d'{{\'Et\'e}} de {{Probabilit\'es}} de {{Saint-Flour XLIX}} - 2019},
  shorttitle = {Statistical Optimal Transport},
  author = {Chewi, Sinho and {Niles-Weed}, Jonathan and Rigollet, Philippe},
  year = 2025,
  series = {Lecture Notes in Mathematics {{\'Ecole}} d'{{\'Et\'e}} de {{Probabilit\'es}} de {{Saint-Flour}}},
  number = {2364},
  publisher = {Springer},
  address = {Cham},
  langid = {english}
}

@techreport{cooney_modelling_2017,
  type = {Stan {{Case Study}}},
  title = {Modelling Loss Curves in Insurance with {{RStan}}},
  author = {Cooney, Mick},
  year = 2017
}

@inproceedings{dalalyan_further_2017,
  title = {Further and Stronger Analogy between Sampling and Optimization: {{Langevin Monte Carlo}} and Gradient Descent},
  shorttitle = {Further and Stronger Analogy between Sampling and Optimization},
  booktitle = {Proceedings of the {{Conference}} on {{Learning Theory}}},
  author = {Dalalyan, Arnak},
  year = 2017,
  series = {{{PMLR}}},
  volume = {65},
  pages = {678--689},
  publisher = {JMLR},
  urldate = {2024-03-10},
  langid = {english}
}

@inproceedings{diao_forwardbackward_2023,
  title = {Forward-Backward {{Gaussian}} Variational Inference via {{JKO}} in the {{Bures-Wasserstein}} Space},
  booktitle = {Proceedings of the {{International Conference}} on {{Machine Learning}}},
  author = {Diao, Michael Ziyang and Balasubramanian, Krishna and Chewi, Sinho and Salim, Adil},
  year = 2023,
  series = {{{PMLR}}},
  volume = {202},
  pages = {7960--7991},
  publisher = {JMLR},
  urldate = {2023-10-26},
  langid = {english}
}

@article{dieuleveut_stochastic_2023,
  title = {Stochastic Approximation beyond Gradient for Signal Processing and Machine Learning},
  author = {Dieuleveut, Aymeric and Fort, Gersende and Moulines, Eric and Wai, Hoi-To},
  year = 2023,
  journal = {IEEE Transactions on Signal Processing},
  volume = {71},
  pages = {3117--3148}
}

@inproceedings{domke_provable_2020,
  title = {Provable Smoothness Guarantees for Black-Box Variational Inference},
  booktitle = {Proceedings of the International Conference on Machine Learning},
  author = {Domke, Justin},
  year = 2020,
  series = {{{PMLR}}},
  volume = {119},
  pages = {2587--2596},
  publisher = {JMLR}
}

@inproceedings{domke_provable_2023,
  title = {Provable Convergence Guarantees for Black-Box Variational Inference},
  booktitle = {Advances in Neural Information Processing Systems},
  author = {Domke, Justin and Gower, Robert and Garrigos, Guillaume},
  year = 2023,
  volume = {36},
  pages = {66289--66327},
  publisher = {Curran Associates, Inc.}
}

@article{dorazio_estimating_2006,
  title = {Estimating Species Richness and Accumulation by Modeling Species Occurrence and Detectability},
  author = {Dorazio, Robert M and Royle, J Andrew and S{\"o}derstr{\"o}m, Bo and Glimsk{\"a}r, Anders},
  year = 2006,
  journal = {Ecology},
  volume = {87},
  number = {4},
  pages = {842--854},
  publisher = {Wiley Online Library}
}

@article{durmus_highdimensional_2019,
  title = {High-Dimensional {{Bayesian}} Inference via the Unadjusted {{Langevin}} Algorithm},
  author = {Durmus, Alain and Moulines, {\'E}ric},
  year = 2019,
  journal = {Bernoulli},
  volume = {25},
  number = {4A},
  pages = {2854--2882},
  urldate = {2024-03-21}
}

@article{fjelde_turingjl_2025,
  title = {Turing.Jl: A General-Purpose Probabilistic Programming Language},
  shorttitle = {Turing.Jl},
  author = {Fjelde, Tor Erlend and Xu, Kai and Widmann, David and Tarek, Mohamed and Pfiffer, Cameron and Trapp, Martin and Axen, Seth D. and Sun, Xianda and Hauru, Markus and Yong, Penelope and Tebbutt, Will and Ghahramani, Zoubin and Ge, Hong},
  year = 2025,
  journal = {ACM Transactions on Probabilistic Machine Learning},
  volume = {1},
  number = {3},
  pages = {1--48},
  urldate = {2025-04-29}
}

@article{fujisawa_multilevel_2021,
  title = {Multilevel {{Monte Carlo}} Variational Inference},
  author = {Fujisawa, Masahiro and Sato, Issei},
  year = 2021,
  journal = {Journal of Machine Learning Research},
  volume = {22},
  number = {278},
  pages = {1--44},
  urldate = {2022-11-08}
}

@techreport{garrigos_handbook_2023,
  type = {{{arXiv}} Preprint},
  title = {Handbook of Convergence Theorems for (Stochastic) Gradient Methods},
  author = {Garrigos, Guillaume and Gower, Robert M.},
  year = 2023,
  number = {arXiv:2301.11235},
  urldate = {2023-05-08}
}

@inproceedings{ge_turing_2018,
  title = {Turing: A Language for Flexible Probabilistic Inference},
  booktitle = {Proceedings of the {{International Conference}} on {{Machine Learning}}},
  author = {Ge, Hong and Xu, Kai and Ghahramani, Zoubin},
  year = 2018,
  series = {{{PMLR}}},
  volume = {84},
  pages = {1682--1690},
  publisher = {JMLR}
}

@inproceedings{geffner_approximation_2020,
  title = {Approximation Based Variance Reduction for Reparameterization Gradients},
  booktitle = {Advances in {{Neural Information Processing Systems}}},
  author = {Geffner, Tomas and Domke, Justin},
  year = 2020,
  volume = {33},
  pages = {2397--2407},
  publisher = {Curran Associates, Inc.},
  urldate = {2023-07-03}
}

@inproceedings{geffner_difficulty_2021,
  title = {On the Difficulty of Unbiased Alpha Divergence Minimization},
  booktitle = {Proceedings of the International Conference on Machine Learning},
  author = {Geffner, Tomas and Domke, Justin},
  year = 2021,
  series = {{{PMLR}}},
  volume = {139},
  pages = {3650--3659},
  publisher = {JMLR}
}

@inproceedings{geffner_rule_2020,
  title = {A Rule for Gradient Estimator Selection, with an Application to Variational Inference},
  booktitle = {Proceedings of the International Conference on Artificial Intelligence and Statistics},
  author = {Geffner, Tomas and Domke, Justin},
  year = 2020,
  series = {{{PMLR}}},
  volume = {108},
  pages = {1803--1812},
  publisher = {JMLR}
}

@inproceedings{geffner_using_2018,
  title = {Using Large Ensembles of Control Variates for Variational Inference},
  booktitle = {Advances in {{Neural Information Processing Systems}}},
  author = {Geffner, Tomas and Domke, Justin},
  year = 2018,
  volume = {31},
  pages = {9960--9970},
  publisher = {Curran Associates, Inc.},
  urldate = {2022-11-10}
}

@article{gelfand_illustration_1990,
  title = {Illustration of {{Bayesian}} Inference in Normal Data Models Using {{Gibbs}} Sampling},
  author = {Gelfand, Alan E. and Hills, Susan E. and {Racine-Poon}, Amy and Smith, Adrian F. M.},
  year = 1990,
  journal = {Journal of the American Statistical Association},
  volume = {85},
  number = {412},
  pages = {972--985},
  urldate = {2025-02-03},
  langid = {english}
}

@book{gelman_bayesian_2014,
  title = {Bayesian Data Analysis},
  author = {Gelman, Andrew and Carlin, John and Stern, Hal and Dunson, David and Vehtari, Aki and Rubin, Donald},
  year = 2014,
  series = {Chapman \& {{Hall}}/{{CRC}} Texts in Statistical Science},
  edition = {3},
  publisher = {CRC Press},
  address = {Boca Raton},
  lccn = {QA279.5 .G45 2014}
}

@book{gelman_data_2021,
  title = {Data Analysis Using Regression and Multilevel/Hierarchical Models},
  author = {Gelman, Andrew and Hill, Jennifer},
  year = 2021,
  series = {Analytical Methods for Social Research},
  edition = {23rd printing},
  publisher = {Cambridge Univ. Press},
  address = {Cambridge},
  langid = {english}
}

@inproceedings{gorbunov_unified_2020,
  title = {A Unified Theory of {{SGD}}: {{Variance}} Reduction, Sampling, Quantization and Coordinate Descent},
  shorttitle = {A {{Unified Theory}} of {{SGD}}},
  booktitle = {Proceedings of the {{International Conference}} on {{Artificial Intelligence}} and {{Statistics}}},
  author = {Gorbunov, Eduard and Hanzely, Filip and Richtarik, Peter},
  year = 2020,
  series = {{{PMLR}}},
  volume = {108},
  pages = {680--690},
  publisher = {JMLR},
  urldate = {2022-10-08},
  langid = {english}
}

@inproceedings{gower_sgd_2019,
  title = {{{SGD}}: {{General}} Analysis and Improved Rates},
  booktitle = {Proceedings of the International Conference on Machine Learning},
  author = {Gower, Robert Mansel and Loizou, Nicolas and Qian, Xun and Sailanbayev, Alibek and Shulgin, Egor and Richt{\'a}rik, Peter},
  year = 2019,
  series = {{{PMLR}}},
  volume = {97},
  pages = {5200--5209},
  publisher = {JMLR}
}

@article{gower_stochastic_2021,
  title = {Stochastic Quasi-Gradient Methods: {{Variance}} Reduction via {{Jacobian}} Sketching},
  shorttitle = {Stochastic Quasi-Gradient Methods},
  author = {Gower, Robert M. and Richt{\'a}rik, Peter and Bach, Francis},
  year = 2021,
  journal = {Mathematical Programming},
  volume = {188},
  number = {1},
  pages = {135--192},
  urldate = {2022-11-08},
  langid = {english}
}

@inproceedings{graves_practical_2011,
  title = {Practical Variational Inference for Neural Networks},
  booktitle = {Advances in {{Neural Information Processing Systems}}},
  author = {Graves, Alex},
  year = 2011,
  volume = {24},
  pages = {2348--2356},
  publisher = {Curran Associates, Inc.},
  urldate = {2026-01-21}
}

@book{hewitt_conservation_1921,
  title = {The Conservation of the Wild Life of {{Canada}}},
  author = {Hewitt, C. Gordon},
  year = 1921,
  publisher = {Charles Scribner's Sons}
}

@inproceedings{hinton_keeping_1993,
  title = {Keeping the Neural Networks Simple by Minimizing the Description Length of the Weights},
  booktitle = {Proceedings of the Annual Conference on {{Computational}} Learning Theory},
  author = {Hinton, Geoffrey E. and {van Camp}, Drew},
  year = 1993,
  pages = {5--13},
  publisher = {ACM Press},
  urldate = {2022-12-23},
  langid = {english}
}

@article{ho_perturbation_1983,
  title = {Perturbation Analysis and Optimization of Queueing Networks},
  author = {Ho, Y. C. and Cao, X.},
  year = 1983,
  journal = {Journal of Optimization Theory and Applications},
  volume = {40},
  number = {4},
  pages = {559--582},
  urldate = {2025-05-03},
  langid = {english}
}

@inproceedings{hoffman_blackbox_2020,
  title = {Black-Box Variational Inference as a Parametric Approximation to {{Langevin}} Dynamics},
  booktitle = {Proceedings of the {{International Conference}} on {{Machine Learning}}},
  author = {Hoffman, Matthew and Ma, Yian},
  year = 2020,
  series = {{{PMLR}}},
  volume = {119},
  pages = {4324--4341},
  publisher = {JMLR},
  urldate = {2022-12-21},
  langid = {english}
}

@inproceedings{huix_theoretical_2024,
  title = {Theoretical {{Guarantees}} for {{Variational Inference}} with {{Fixed-Variance Mixture}} of {{Gaussians}}},
  booktitle = {Proceedings of the {{International Conference}} on {{Machine Learning}}},
  author = {Huix, Tom and Korba, Anna and Durmus, Alain and Moulines, Eric},
  year = 2024,
  series = {{{PMLR}}},
  volume = {235},
  pages = {20700--20721},
  publisher = {JMLR}
}

@article{jordan_introduction_1999,
  title = {An Introduction to Variational Methods for Graphical Models},
  author = {Jordan, Michael I. and Ghahramani, Zoubin and Jaakkola, Tommi S. and Saul, Lawrence K.},
  year = 1999,
  journal = {Machine Learning},
  volume = {37},
  number = {2},
  pages = {183--233},
  urldate = {2021-05-09}
}

@article{jordan_variational_1998,
  title = {The Variational Formulation of the {{Fokker--Planck}} Equation},
  author = {Jordan, Richard and Kinderlehrer, David and Otto, Felix},
  year = 1998,
  journal = {SIAM Journal on Mathematical Analysis},
  volume = {29},
  number = {1},
  pages = {1--17},
  urldate = {2024-04-04},
  langid = {english}
}

@book{kery_bayesian_2012,
  title = {Bayesian Population Analysis Using {{WinBUGS}}: A Hierarchical Perspective},
  shorttitle = {Bayesian Population Analysis Using {{WinBUGS}}},
  author = {K{\'e}ry, Manuel Marton Marc and Schaub, Michael},
  year = 2012,
  edition = {1},
  publisher = {Academic Press},
  address = {Waltham, MA},
  langid = {english}
}

@article{khaled_unified_2023,
  title = {Unified Analysis of Stochastic Gradient Methods for Composite Convex and Smooth Optimization},
  author = {Khaled, Ahmed and Sebbouh, Othmane and Loizou, Nicolas and Gower, Robert M. and Richt{\'a}rik, Peter},
  year = 2023,
  journal = {Journal of Optimization Theory and Applications},
  volume = {199},
  eprint = {2006.11573},
  primaryclass = {cs, math, stat},
  pages = {499--540},
  urldate = {2023-05-07},
  archiveprefix = {arXiv}
}

@article{khan_bayesian_2023,
  title = {The {{Bayesian}} Learning Rule},
  author = {Khan, Mohammad Emtiyaz and Rue, H{\aa}vard},
  year = 2023,
  journal = {Journal of Machine Learning Research},
  volume = {24},
  number = {281},
  pages = {1--46},
  urldate = {2026-01-02}
}

@inproceedings{khan_fast_2018,
  title = {Fast yet Simple Natural-Gradient Descent for Variational Inference in Complex Models},
  booktitle = {Proceedings of the {{International Symposium}} on {{Information Theory}} and {{Its Applications}}},
  author = {Khan, Mohammad Emtiyaz and Nielsen, Didrik},
  year = 2018,
  pages = {31--35},
  publisher = {IEEE Press},
  address = {Singapore},
  urldate = {2026-01-08}
}

@inproceedings{kim_convergence_2023,
  type = {Techreport},
  title = {On the Convergence of Black-Box Variational Inference},
  booktitle = {Advances in {{Neural Information Processing Systems}}},
  author = {Kim, Kyurae and Oh, Jisu and Wu, Kaiwen and Ma, Yian and Gardner, Jacob R.},
  year = 2023,
  volume = {36},
  eprint = {2305.15349},
  primaryclass = {cs, eess, math, stat},
  pages = {44615--44657},
  publisher = {Curran Associates Inc.},
  urldate = {2023-06-03},
  archiveprefix = {arXiv},
  copyright = {All rights reserved}
}

@inproceedings{kim_linear_2024,
  title = {Linear Convergence of Black-Box Variational Inference: Should We Stick the Landing?},
  booktitle = {Proceedings of the {{International Conference}} on {{Artificial Intelligence}} and {{Statistics}}},
  author = {Kim, Kyurae and Ma, Yian and Gardner, Jacob R.},
  year = 2024,
  series = {{{PMLR}}},
  volume = {238},
  pages = {235--243},
  publisher = {JMLR},
  copyright = {All rights reserved}
}

@inproceedings{kim_nearly_2025,
  title = {Nearly Dimension-Independent Convergence of Mean-Field Black-Box Variational Inference},
  booktitle = {Advances in {{Neural Information Processing Systems}}},
  author = {Kim, Kyurae and Ma, Yi-An and Campbell, Trevor and Gardner, Jacob R.},
  year = 2025,
  volume = {38  (to appear)},
  eprint = {2505.21721},
  primaryclass = {stat},
  publisher = {Curran Associates, Inc.},
  urldate = {2026-01-03},
  archiveprefix = {arXiv}
}

@inproceedings{kingma_autoencoding_2014,
  title = {Auto-Encoding Variational {{Bayes}}},
  booktitle = {Proceedings of the {{International Conference}} on {{Learning Representations}}},
  author = {Kingma, Diederik P. and Welling, Max},
  year = 2014,
  address = {Banff, AB, Canada},
  urldate = {2023-10-07},
  langid = {english}
}

@article{kucukelbir_automatic_2017,
  title = {Automatic Differentiation Variational Inference},
  author = {Kucukelbir, Alp and Tran, Dustin and Ranganath, Rajesh and Gelman, Andrew and Blei, David M.},
  year = 2017,
  journal = {Journal of Machine Learning Research},
  volume = {18},
  number = {14},
  pages = {1--45}
}

@article{kullback_information_1951,
  title = {On Information and Sufficiency},
  author = {Kullback, S. and Leibler, R. A.},
  year = 1951,
  journal = {The Annals of Mathematical Statistics},
  volume = {22},
  number = {1},
  pages = {79--86},
  publisher = {Institute of Mathematical Statistics},
  urldate = {2023-11-12}
}

@article{kumar_optimization_2025,
  title = {Optimization Guarantees for Square-Root Natural-Gradient Variational Inference},
  author = {Kumar, Navish and M{\"o}llenhoff, Thomas and Khan, Mohammad Emtiyaz and Lucchi, Aurelien},
  year = 2025,
  journal = {Transactions on Machine Learning Research},
  urldate = {2026-01-14},
  langid = {english}
}

@techreport{lacoste-julien_simpler_2012,
  type = {{{arXiv}} Preprint},
  title = {A Simpler Approach to Obtaining an $\mathrm{O}(1/t)$ Convergence Rate for the Projected Stochastic Subgradient Method},
  author = {{Lacoste-Julien}, Simon and Schmidt, Mark and Bach, Francis},
  year = 2012,
  number = {arXiv:1212.2002},
  eprint = {1212.2002},
  primaryclass = {cs, math, stat},
  urldate = {2023-02-16},
  archiveprefix = {arXiv}
}

@inproceedings{lambert_variational_2022,
  title = {Variational Inference via {{Wasserstein}} Gradient Flows},
  booktitle = {Advances in {{Neural Information Processing Systems}}},
  author = {Lambert, Marc and Chewi, Sinho and Bach, Francis and Bonnabel, Silv{\`e}re and Rigollet, Philippe},
  year = 2022,
  volume = {35},
  pages = {14434--14447},
  publisher = {Curran Associates, Inc.},
  urldate = {2023-10-26},
  langid = {english}
}

@inproceedings{lin_fast_2019,
  title = {Fast and Simple Natural-Gradient Variational Inference with Mixture of Exponential-Family Approximations},
  booktitle = {Proceedings of the {{International Conference}} on {{Machine Learning}}},
  author = {Lin, Wu and Khan, Mohammad Emtiyaz and Schmidt, Mark},
  year = 2019,
  series = {{{PMLR}}},
  volume = {97},
  pages = {3992--4002},
  publisher = {JMLR},
  urldate = {2023-10-26},
  langid = {english}
}

@techreport{lin_steins_2025,
  type = {{{arXiv}} Preprint},
  title = {Stein's Lemma for the Reparameterization Trick with Exponential Family Mixtures},
  author = {Lin, Wu and Khan, Mohammad Emtiyaz and Schmidt, Mark},
  year = 2025,
  number = {arXiv:1910.13398},
  eprint = {1910.13398},
  primaryclass = {stat},
  institution = {arXiv},
  urldate = {2026-01-07},
  archiveprefix = {arXiv}
}

@article{liu_siegels_1994,
  title = {Siegel's Formula via {{Stein}}'s Identities},
  author = {Liu, Jun S.},
  year = 1994,
  journal = {Statistics \& Probability Letters},
  volume = {21},
  number = {3},
  pages = {247--251},
  urldate = {2026-01-08}
}

@inproceedings{luu_stochastic_2025,
  title = {Stochastic Variance-Reduced {{Gaussian}} Variational Inference on the {{Bures-Wasserstein}} Manifold},
  booktitle = {Proceedings of the {{International Conference}} on {{Learning Representations}}},
  author = {Luu, Hoang Phuc Hau and Yu, Hanlin and Williams, Bernardo and Hartmann, Marcelo and Klami, Arto},
  year = 2025,
  urldate = {2026-01-14},
  langid = {english}
}

@inproceedings{magnusson_posteriordb_2025,
  title = {Posteriordb: {{Testing}}, {{Benchmarking}} and {{Developing Bayesian Inference Algorithms}}},
  shorttitle = {Posteriordb},
  booktitle = {Proceedings of {{The International Conference}} on {{Artificial Intelligence}} and {{Statistics}}},
  author = {Magnusson, M{\aa}ns and Torgander, Jakob and B{\"u}rkner, Paul-Christian and Zhang, Lu and Carpenter, Bob and Vehtari, Aki},
  year = 2025,
  series = {{{PMLR}}},
  volume = {258},
  pages = {1198--1206},
  publisher = {JMLR},
  urldate = {2026-01-17},
  langid = {english}
}

@inproceedings{miller_reducing_2017,
  title = {Reducing Reparameterization Gradient Variance},
  booktitle = {Advances in {{Neural Information Processing Systems}}},
  author = {Miller, Andrew and Foti, Nick and D' Amour, Alexander and Adams, Ryan P},
  year = 2017,
  volume = {30},
  pages = {3708--3718},
  publisher = {Curran Associates, Inc.},
  urldate = {2022-11-10}
}

@article{mohamed_monte_2020,
  title = {Monte {{Carlo}} Gradient Estimation in Machine Learning},
  author = {Mohamed, Shakir and Rosca, Mihaela and Figurnov, Michael and Mnih, Andriy},
  year = 2020,
  journal = {Journal of Machine Learning Research},
  volume = {21},
  number = {132},
  pages = {1--62},
  urldate = {2022-11-15}
}

@article{nemirovski_robust_2009,
  title = {Robust Stochastic Approximation Approach to Stochastic Programming},
  author = {Nemirovski, A. and Juditsky, A. and Lan, G. and Shapiro, A.},
  year = 2009,
  journal = {SIAM Journal on Optimization},
  volume = {19},
  number = {4},
  pages = {1574--1609},
  urldate = {2023-05-10},
  langid = {english}
}

@article{opper_variational_2009,
  title = {The Variational {{Gaussian}} Approximation Revisited},
  author = {Opper, Manfred and Archambeau, C{\'e}dric},
  year = 2009,
  journal = {Neural Computation},
  volume = {21},
  number = {3},
  pages = {786--792},
  urldate = {2026-01-21},
  langid = {english}
}

@book{parikh_proximal_2014,
  title = {Proximal Algorithms},
  author = {Parikh, Neal and Boyd, Stephen P.},
  year = 2014,
  series = {Foundations and {{Trends}}\textregistered{} in {{Optimization}}},
  volume = {1},
  publisher = {Now Publishers},
  address = {Norwell, MA},
  langid = {english}
}

@article{patil_pymc_2010,
  title = {{{PyMC}}: {{Bayesian}} Stochastic Modelling in {{Python}}},
  shorttitle = {{{{\textbf{PyMC}}}}},
  author = {Patil, Anand and Huard, David and Fonnesbeck, Christopher},
  year = 2010,
  journal = {Journal of Statistical Software},
  volume = {35},
  number = {4},
  pages = {1--81},
  urldate = {2020-07-23},
  langid = {english}
}

@article{peterson_explorations_1989,
  title = {Explorations of the Mean Field Theory Learning Algorithm},
  author = {Peterson, Carsten and Hartman, Eric},
  year = 1989,
  journal = {Neural Networks},
  volume = {2},
  number = {6},
  pages = {475--494},
  urldate = {2022-12-23},
  langid = {english}
}

@article{price_useful_1958,
  title = {A Useful Theorem for Nonlinear Devices Having {{Gaussian}} Inputs},
  author = {Price, R.},
  year = 1958,
  journal = {IRE Transactions on Information Theory},
  volume = {4},
  number = {2},
  pages = {69--72},
  urldate = {2026-01-07}
}

@inproceedings{ranganath_black_2014,
  title = {Black Box Variational Inference},
  booktitle = {Proceedings of the International Conference on Artificial Intelligence and Statistics},
  author = {Ranganath, Rajesh and Gerrish, Sean and Blei, David},
  year = 2014,
  series = {{{PMLR}}},
  volume = {33},
  pages = {814--822},
  publisher = {JMLR}
}

@inproceedings{rezende_stochastic_2014,
  title = {Stochastic Backpropagation and Approximate Inference in Deep Generative Models},
  booktitle = {Proceedings of the {{International Conference}} on {{Machine Learning}}},
  author = {Rezende, Danilo Jimenez and Mohamed, Shakir and Wierstra, Daan},
  year = 2014,
  series = {{{PMLR}}},
  volume = {32},
  pages = {1278--1286},
  publisher = {JMLR},
  urldate = {2023-11-08},
  langid = {english}
}

@inproceedings{rezende_variational_2015,
  title = {Variational Inference with Normalizing Flows},
  booktitle = {Proceedings of the {{International Conference}} on {{Machine Learning}}},
  author = {Rezende, Danilo and Mohamed, Shakir},
  year = 2015,
  series = {{{PMLR}}},
  volume = {37},
  pages = {1530--1538},
  publisher = {JMLR},
  urldate = {2025-05-07},
  langid = {english}
}

@article{robbins_stochastic_1951,
  title = {A Stochastic Approximation Method},
  author = {Robbins, Herbert and Monro, Sutton},
  year = 1951,
  journal = {The Annals of Mathematical Statistics},
  volume = {22},
  number = {3},
  pages = {400--407},
  urldate = {2019-09-17},
  langid = {english}
}

@article{roualdes_bridgestan_2023,
  title = {{{BridgeStan}}: {{Efficient}} in-Memory Access to the Methods of a {{Stan}} Model},
  shorttitle = {{{BridgeStan}}},
  author = {Roualdes, Edward A. and Ward, Brian and Carpenter, Bob and Seyboldt, Adrian and Axen, Seth D.},
  year = 2023,
  journal = {Journal of Open Source Software},
  volume = {8},
  number = {87},
  pages = {5236},
  urldate = {2026-01-12},
  langid = {english}
}

@article{rubinstein_sensitivity_1992,
  title = {Sensitivity Analysis of Discrete Event Systems by the ``Push out'' Method},
  author = {Rubinstein, Reuven Y.},
  year = 1992,
  journal = {Annals of Operations Research},
  volume = {39},
  number = {1},
  pages = {229--250},
  urldate = {2025-05-03},
  langid = {english}
}

@inproceedings{salim_wasserstein_2020,
  title = {The {{Wasserstein}} Proximal Gradient Algorithm},
  booktitle = {Advances in {{Neural Information Processing Systems}}},
  author = {Salim, Adil and Korba, Anna and Luise, Giulia},
  year = 2020,
  volume = {33},
  pages = {12356--12366},
  publisher = {Curran Associates, Inc.},
  urldate = {2026-01-07}
}

@article{salimans_fixedform_2013,
  title = {Fixed-Form Variational Posterior Approximation through Stochastic Linear Regression},
  author = {Salimans, Tim and Knowles, David A.},
  year = 2013,
  journal = {Bayesian Analysis},
  volume = {8},
  number = {4},
  pages = {837--882},
  publisher = {International Society for Bayesian Analysis},
  urldate = {2023-10-15}
}

@techreport{schmidt_fast_2013,
  type = {{{arXiv}} Preprint},
  title = {Fast Convergence of Stochastic Gradient Descent under a Strong Growth Condition},
  author = {Schmidt, Mark and Roux, Nicolas Le},
  year = 2013,
  number = {arXiv:1308.6370},
  urldate = {2022-11-08}
}

@article{shalev-shwartz_pegasos_2011,
  title = {Pegasos: Primal Estimated Sub-Gradient Solver for {{SVM}}},
  shorttitle = {Pegasos},
  author = {{Shalev-Shwartz}, Shai and Singer, Yoram and Srebro, Nathan and Cotter, Andrew},
  year = 2011,
  journal = {Mathematical Programming},
  volume = {127},
  number = {1},
  pages = {3--30},
  urldate = {2024-01-22},
  langid = {english}
}

@inproceedings{shamir_stochastic_2013,
  title = {Stochastic Gradient Descent for Non-Smooth Optimization: {{Convergence}} Results and Optimal Averaging Schemes},
  shorttitle = {Stochastic Gradient Descent for Non-Smooth Optimization},
  booktitle = {Proceedings of the {{International Conference}} on {{Machine Learning}}},
  author = {Shamir, Ohad and Zhang, Tong},
  year = 2013,
  series = {{{PMLR}}},
  volume = {28},
  pages = {71--79},
  publisher = {JMLR},
  urldate = {2023-02-16},
  langid = {english}
}

@article{solomon_traumatic_1953,
  title = {Traumatic Avoidance Learning: {{Acquisition}} in Normal Dogs},
  shorttitle = {Traumatic Avoidance Learning},
  author = {Solomon, Richard L. and Wynne, Lyman C.},
  year = 1953,
  journal = {Psychological Monographs: General and Applied},
  volume = {67},
  number = {4},
  pages = {1--19},
  urldate = {2026-01-19},
  langid = {english}
}

@techreport{spiegelhalter_bugs_1996,
  title = {{{BUGS}} Examples Volume 1},
  author = {Spiegelhalter, David and Thomas, Andrew and Best, Nicky and Gilks, Wally},
  year = 1996,
  address = {Cambridge, U.K.},
  institution = {Institute of Public Health}
}

@article{stein_estimation_1981,
  title = {Estimation of the Mean of a Multivariate Normal Distribution},
  author = {Stein, Charles M.},
  year = 1981,
  journal = {The Annals of Statistics},
  volume = {9},
  number = {6},
  pages = {1135--1151},
  publisher = {Institute of Mathematical Statistics},
  urldate = {2026-01-08}
}

@techreport{stich_unified_2019,
  type = {{{arXiv}} Preprint},
  title = {Unified Optimal Analysis of the (Stochastic) Gradient Method},
  author = {Stich, Sebastian U.},
  year = 2019,
  number = {arXiv:1907.04232},
  eprint = {1907.04232},
  primaryclass = {cs, math, stat},
  urldate = {2022-05-15},
  archiveprefix = {arXiv}
}

@inproceedings{sun_natural_2025,
  title = {Natural Gradient {{VI}}: {{Guarantees}} for Non-Conjugate Models},
  shorttitle = {Natural Gradient {{VI}}},
  booktitle = {Advances in {{Neural Information Processing Systems}}},
  author = {Sun, Fangyuan and Fatkhullin, Ilyas and He, Niao},
  year = 2025,
  volume = {38  (to appear)},
  eprint = {2510.19163},
  primaryclass = {cs},
  publisher = {Curran Associates, Inc.},
  urldate = {2026-01-19},
  archiveprefix = {arXiv}
}

@misc{symbol-1_answer_2022,
  title = {Answer to ``{{Meaningful}} Lower-Bound of $\sqrt{a^2+b} - a$ When $a \gg b>0$''},
  author = {{Symbol-1}},
  year = 2022,
  journal = {Mathematics Stack Exchange},
  url = {https://math.stackexchange.com/q/4360503}
}

@inproceedings{talamon_variational_2025,
  title = {Variational Inference with Mixtures of Isotropic {{Gaussians}}},
  booktitle = {Advances of {{Neural Information Processing Systems}}},
  author = {Talamon, Marguerite Petit and Lambert, Marc and Korba, Anna},
  year = 2025,
  volume = {38 (to appear)},
  publisher = {Curran Associates, Inc.},
  urldate = {2026-01-08},
  langid = {english}
}

@article{tan_analytic_2025,
  title = {Analytic Natural Gradient Updates for {{Cholesky}} Factor in {{Gaussian}} Variational Approximation},
  author = {Tan, Linda S L},
  year = 2025,
  journal = {Journal of the Royal Statistical Society Series B: Statistical Methodology},
  volume = {87},
  number = {4},
  pages = {930--956},
  urldate = {2026-01-24},
  copyright = {https://creativecommons.org/licenses/by/4.0/},
  langid = {english}
}

@article{taylor_forecasting_2018,
  title = {Forecasting at Scale},
  author = {Taylor, Sean J. and Letham, Benjamin},
  year = 2018,
  journal = {The American Statistician},
  volume = {72},
  number = {1},
  pages = {37--45},
  urldate = {2025-01-29},
  langid = {english}
}

@inproceedings{titsias_doubly_2014,
  title = {Doubly Stochastic Variational {{Bayes}} for Non-Conjugate Inference},
  booktitle = {Proceedings of the {{International Conference}} on {{Machine Learning}}},
  author = {Titsias, Michalis and {L{\'a}zaro-Gredilla}, Miguel},
  year = 2014,
  series = {{{PMLR}}},
  volume = {32},
  pages = {1971--1979},
  publisher = {JMLR},
  urldate = {2022-12-19},
  langid = {english}
}

@inproceedings{vaswani_fast_2019,
  title = {Fast and Faster Convergence of {{SGD}} for Over-Parameterized Models and an Accelerated Perceptron},
  booktitle = {Proceedings of the {{International Conference}} on {{Artificial Intelligence}} and {{Statistics}}},
  author = {Vaswani, Sharan and Bach, Francis and Schmidt, Mark},
  year = 2019,
  series = {{{PMLR}}},
  volume = {89},
  pages = {1195--1204},
  publisher = {JMLR},
  urldate = {2022-11-08},
  langid = {english}
}

@book{villani_optimal_2009,
  title = {Optimal {{Transport}}},
  author = {Villani, C{\'e}dric},
  year = 2009,
  series = {Grundlehren Der Mathematischen {{Wissenschaften}}},
  volume = {338},
  publisher = {Springer Berlin Heidelberg},
  address = {Berlin, Heidelberg},
  urldate = {2024-03-19}
}

@inproceedings{wang_joint_2024,
  type = {Techreport},
  title = {Joint Control Variate for Faster Black-Box Variational Inference},
  booktitle = {Proceedings of the {{International Conference}} on {{Artificial Intelligence}} and {{Statistics}}},
  author = {Wang, Xi and Geffner, Tomas and Domke, Justin},
  year = 2024,
  series = {{{PMLR}}},
  volume = {238},
  eprint = {2210.07290},
  primaryclass = {cs, stat},
  pages = {1639--1647},
  publisher = {JMLR},
  urldate = {2023-07-03},
  archiveprefix = {arXiv}
}

@inproceedings{wibisono_sampling_2018,
  title = {Sampling as Optimization in the Space of Measures: {{The Langevin}} Dynamics as a Composite Optimization Problem},
  shorttitle = {Sampling as Optimization in the Space of Measures},
  booktitle = {Proceedings of the  {{Conference On Learning Theory}}},
  author = {Wibisono, Andre},
  year = 2018,
  series = {{{PMLR}}},
  volume = {75},
  pages = {2093--3027},
  publisher = {JMLR},
  urldate = {2024-03-11},
  langid = {english}
}

@phdthesis{wilson_lyapunov_2018,
  title = {Lyapunov {{Arguments}} in {{Optimization}}},
  author = {Wilson, Ashia},
  year = 2018,
  address = {Berkeley, California},
  school = {University of California, Berkeley}
}

@techreport{wingate_automated_2013,
  type = {{{arXiv}} Preprint},
  title = {Automated Variational Inference in Probabilistic Programming},
  author = {Wingate, David and Weber, Theophane},
  year = 2013,
  number = {arXiv:1301.1299},
  eprint = {1301.1299},
  primaryclass = {stat},
  urldate = {2025-04-26},
  archiveprefix = {arXiv}
}

@book{wright_optimization_2021,
  title = {Optimization for Data Analysis},
  author = {Wright, Stephen J. and Recht, Benjamin},
  year = 2021,
  publisher = {Cambridge University Press},
  address = {New York},
  lccn = {QA76.9.B45 W75 2021}
}

@inproceedings{wu_understanding_2024,
  title = {Understanding {{Stochastic Natural Gradient Variational Inference}}},
  booktitle = {Proceedings of the 41st {{International Conference}} on {{Machine Learning}}},
  author = {Wu, Kaiwen and Gardner, Jacob R.},
  year = 2024,
  series = {{{PMLR}}},
  volume = {235},
  pages = {53398--53421},
  publisher = {JMLR},
  urldate = {2026-01-19},
  langid = {english}
}

@inproceedings{xu_variance_2019,
  title = {Variance Reduction Properties of the Reparameterization Trick},
  booktitle = {Proceedings of the {{International Conference}} on {{Artificial Intelligence}} and {{Statistics}}},
  author = {Xu, Ming and Quiroz, Matias and Kohn, Robert and Sisson, Scott A.},
  year = 2019,
  series = {{{PMLR}}},
  volume = {89},
  pages = {2711--2720},
  publisher = {JMLR},
  urldate = {2022-12-19},
  langid = {english}
}

@techreport{yi_bridging_2023,
  type = {{{arXiv Preprint}}},
  title = {Bridging the Gap between Variational Inference and {{Wasserstein}} Gradient Flows},
  author = {Yi, Mingxuan and Liu, Song},
  year = 2023,
  number = {arXiv:2310.20090},
  eprint = {2310.20090},
  primaryclass = {stat.ML},
  archiveprefix = {arXiv}
}

\clearpage
\appendix
\onecolumn

{\hypersetup{linkbordercolor=black,linkcolor=black}
\tableofcontents
}

\clearpage

\section{Benchmark Problems}\label{section:benchmark_problems}

\begin{table}[h!]
    \vspace{-2ex}
    \caption{Overview of Benchmark Problems}
    \label{tab:problems}
    \vspace{-1ex}
    \centering
    \def\arraystretch{1.5}
    \begin{tabular}{lp{0.55\columnwidth}rc}
        \multicolumn{1}{c}{\textbf{Name}} & \multicolumn{1}{c}{\textbf{Description}} & \(d\) & \multicolumn{1}{c}{\textbf{Reference}} 
        \\ \midrule
        Dogs & Logistic regression model of the traumatic learning behavior of dogs. The dataset is from the Solomon-Wynne experiment. (model: \(\mathtt{dogs}\); dataset: \(\mathtt{dogs}\)) & 3 & \makecell[t]{\citet{gelman_data_2021}\\\citet{solomon_traumatic_1953}}
        \\
        Basketball & Hidden Markov model of NBA basketball player SportVU tracking data during a drive event. (model: \(\mathtt{hmm\_drive\_1}\); dataset: \(\mathtt{bball\_drive\_event\_1}\)) & 6 & \makecell[t]{\citet{ali_tagging_2019}}
        \\
        Lynx & Lotka-Volterra model of a lynx-hare population. The dataset is the number of pelts collected by the Hudson's Bay Company in the years 1910--1920. (model: \(\mathtt{lotka\_volterra}\); dataset: \(\mathtt{hudson\_lynx\_hare}\)) & 8 & \makecell[t]{\citet{carpenter_predatorprey_2018}\\\citet{hewitt_conservation_1921}}
        \\
        NES2000 & Linear model of political party identification. The data set is from the 2000 National Election Study. (model: \(\mathtt{nes2000}\); dataset: \(\mathtt{nes}\)) & 10 & \makecell[t]{\citet{gelman_data_2021}}
        \\
        Bones & Latent trait model for multiple ordered categorical responses for quantifying skeletal maturity from radiograph maturity ratings with missing entries. (model: \(\mathtt{bones\_model}\); dataset: \(\mathtt{bones\_data}\)) & 13 & \makecell[t]{\citet{spiegelhalter_bugs_1996}}
        \\
        SISLOB &
        Loss model of insurance claims.
        The model is the single line-of-business, single insurer (SISLOB) variant, where the dataset is the ``ppauto'' line of business, part of the ``Schedule P loss data'' provided by the Casualty Actuarial Society.
        (model: \(\mathtt{losscurve\_sislob}\); dataset: \(\mathtt{loss\_curves}\)) & 15 & \makecell[t]{\citet{cooney_modelling_2017}}
        \\
        Pilots & Linear mixed effects model with varying intercepts for estimating the psychological effect of pilots when performing flight simulations on various airports.
        (model: \(\mathtt{pilots}\); dataset: \(\mathtt{pilots}\)) & 18 & \makecell[t]{\citet{gelman_data_2021}}
        \\
        Downloads & Prophet time series model applied to the download count of \texttt{rstan} over time. The model is an additive combination of
        \begin{enumerate*}[label=(\roman*)]
            \item a trend model,
            \item a model of seasonality, and
            \item a model for events such as holidays.
        \end{enumerate*}
        (model: \(\mathtt{prophet}\); dataset: \(\mathtt{rstan\_downloads}\)) 
        & 62 & \makecell[t]{\citealt{taylor_forecasting_2018}\\\citealt{bales_selecting_2019}}
        \\
        Rats & Linear mixed effects model with varying slopes and intercepts for modeling the weight of young rats over five weeks. (model: \(\mathtt{rats\_model}\); data: \(\mathtt{rats\_data}\)) & 65 & \makecell[t]{\citet{spiegelhalter_bugs_1996}\\\citet{gelfand_illustration_1990}}
        \\
        Radon & Multilevel mixed effects model with log-normal likelihood and varying intercepts for modeling the radon level measured in U.S. households. We use the Minnesota state subset. (model: \(\mathtt{radon\_hierarchical\_intercept\_centered}\); dataset: \(\mathtt{radon\_mn}\)) & 90 & \makecell[t]{\citealt{magnusson_posteriordb_2025}\\\citealt{gelman_bayesian_2014}}
        \\
        Butterfly & Multispecies occupancy model with correlation between sites. The dataset contains counts of butterflies from twenty grassland sites in south-central Sweden (model: \(\mathtt{butterfly}\); dataset: \(\mathtt{multi\_occupancy}\)) & 106 & \makecell[t]{\citet{dorazio_estimating_2006}}
        \\
        Birds & Mixed effects model with a Poisson likelihood and varying intercepts for modeling the occupancy of the Coal tit (\textit{Parus ater}) bird species during the breeding season in Switzerland. (model: \(\mathtt{GLMM1\_model}\); dataset: \(\mathtt{GLMM\_data}\)) & 237 & \makecell[t]{\citet{kery_bayesian_2012}} 
    \end{tabular}
    \vspace{-2ex}
\end{table}

The benchmark problems used in the experiments of \cref{section:experiments} and \cref{section:additional_results} are organized in \cref{tab:problems}.

\clearpage

\section{Additional Experimental Results}\label{section:additional_results}

\begin{figure}[h!]
    \centering
    \subfloat[Basketball]{
        \includegraphics[width=0.9\linewidth]{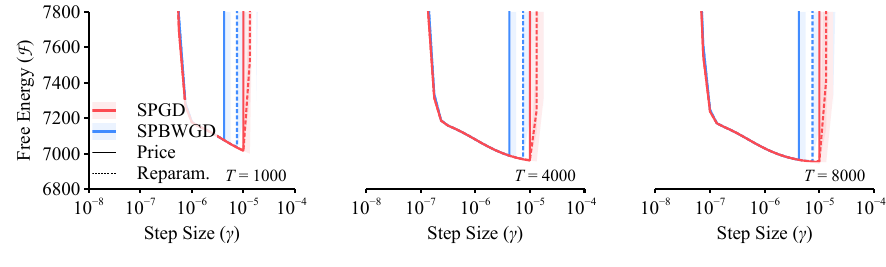}
        \vspace{-1ex}
    }\\
    \subfloat[Lynx]{
        \includegraphics[width=0.9\linewidth]{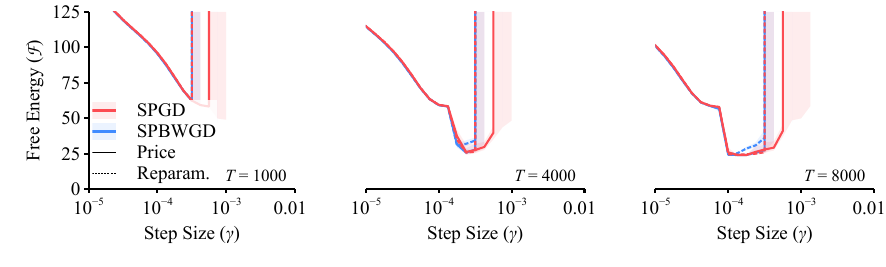}
        \vspace{-1ex}
    }\\
    \subfloat[Dogs]{
        \includegraphics[width=0.9\linewidth]{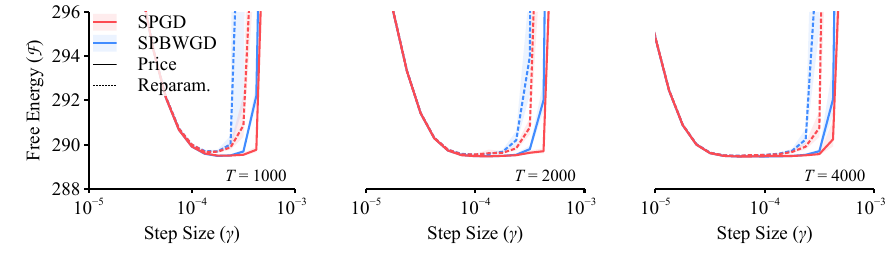}
        \vspace{-1ex}
    }\\
    \subfloat[NES2000]{
        \includegraphics[width=0.9\linewidth]{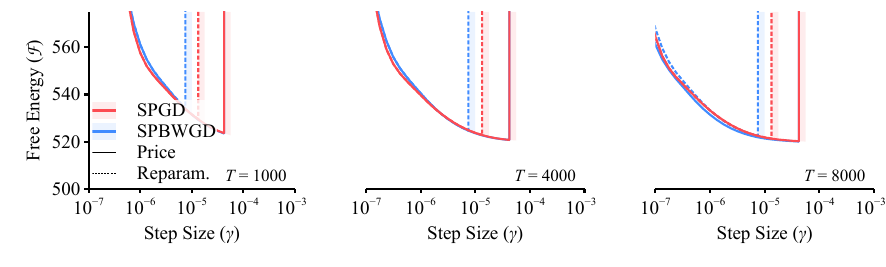}
        \vspace{-1ex}
    }\\
    \vspace{-1ex}
    \caption{
    \textbf{Variational free energy ($\mathcal{F}$) of the last iterate $q_T$ versus step size $\gamma$.}
    The solid lines are the mean estimated over 32 independent repetitions, while the shaded regions are the corresponding 95\% bootstrap confidence intervals.
    }
    \vspace{-1ex}
\end{figure}

\begin{figure}[h!]
    \centering
    \subfloat[Bones]{
        \includegraphics[width=0.9\linewidth]{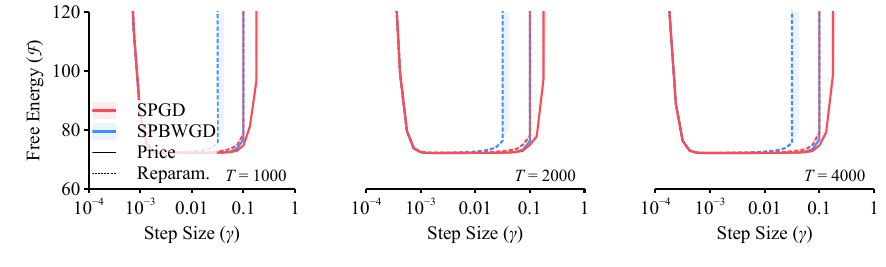}
        \vspace{-1ex}
    }\\
    \subfloat[SISLOB]{
        \includegraphics[width=0.9\linewidth]{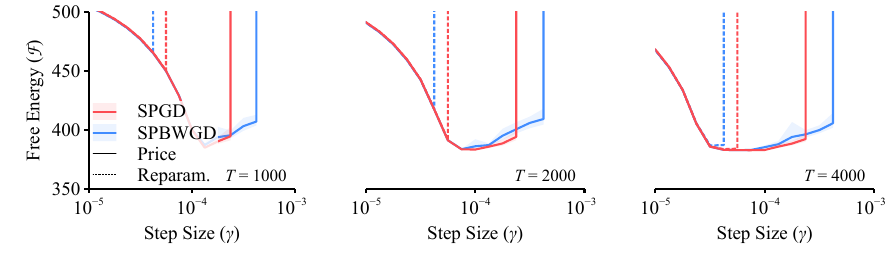}
        \vspace{-1ex}
    }\\
    \subfloat[Pilots]{
        \includegraphics[width=0.9\linewidth]{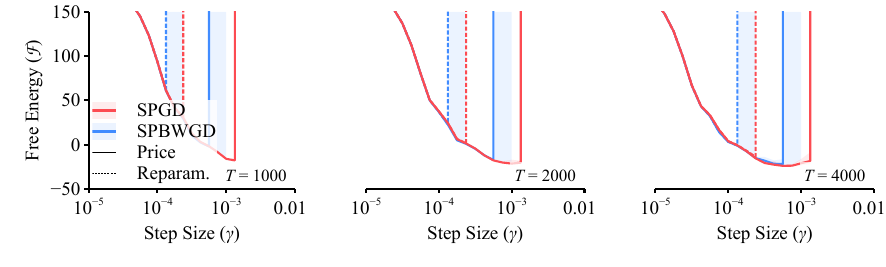}
        \vspace{-1ex}
    }\\
    \subfloat[Downloads]{
        \includegraphics[width=0.9\linewidth]{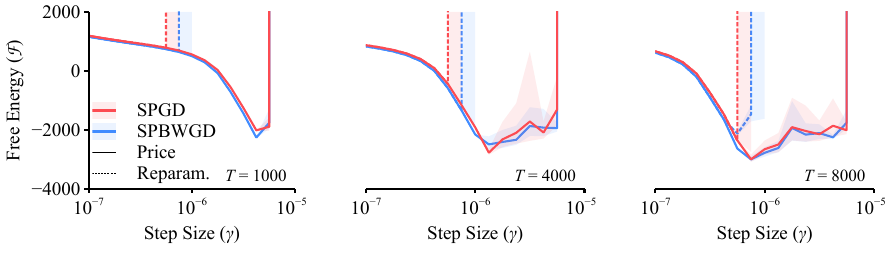}
        \vspace{-1ex}
    }\\
    \caption{
    \textbf{Variational free energy ($\mathcal{F}$) of the last iterate $q_T$ versus step size $\gamma$.}
    The solid lines are the mean estimated over 32 independent repetitions, while the shaded regions are the corresponding 95\% bootstrap confidence intervals.
    }
\end{figure}

\begin{figure}[h!]
    \subfloat[Rats]{
        \includegraphics[width=0.9\linewidth]{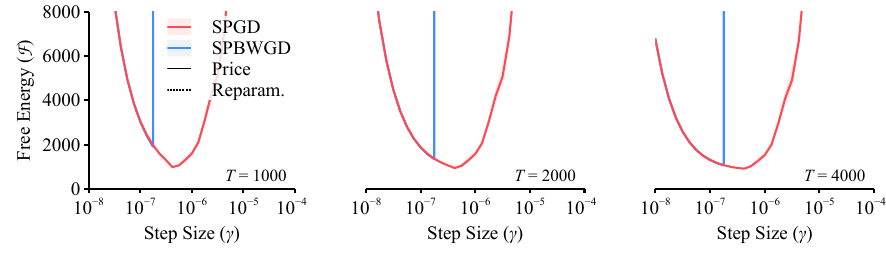}
        \vspace{-1ex}
    }\\
    \subfloat[Radon]{
        \includegraphics[width=0.9\linewidth]{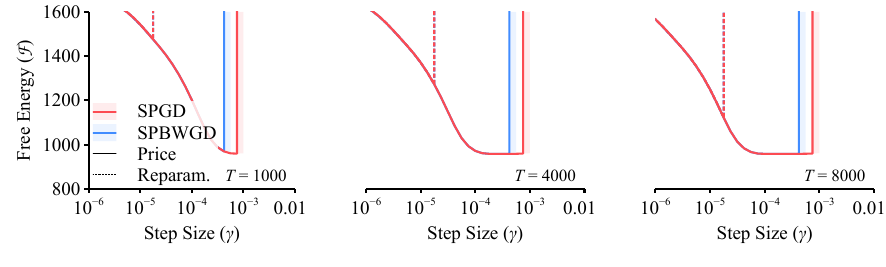}
        \vspace{-1ex}
    }\\
    \subfloat[Butterfly]{
        \includegraphics[width=0.9\linewidth]{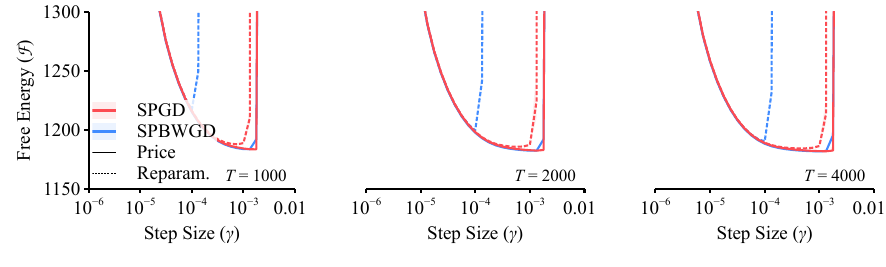}
        \vspace{-1ex}
    }\\
    \subfloat[Birds]{
        \includegraphics[width=0.9\linewidth]{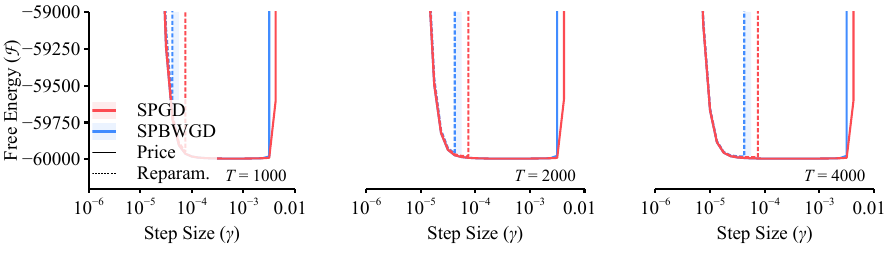}
        \vspace{-1ex}
    }
    \centering
    \caption{
    \textbf{(continued) Variational free energy ($\mathcal{F}$) of the last iterate $q_T$ versus step size $\gamma$.}
    In the case of the Rats problem, methods using first-order estimators didn't converge for any step size between $10^{-8}$ and $10^{0}$, which is why the dotted lines are not visible.
    The solid lines are the mean estimated over 32 independent repetitions, while the shaded regions are the corresponding 95\% bootstrap confidence intervals.
    }
\end{figure}

\clearpage
\twocolumn
\section{Auxiliary Results}

\subsection{Properties of the Bregman Divergence}\label{section:bregman_properties}

Under \cref{assumption:potential_convex_smooth}, it is well known~\citep[Lemma 2.14 \& 2.25]{garrigos_handbook_2023} that the Bregman divergence is related to the Euclidean distance via the inequality
\[
    \frac{\mu}{2} \norm{x - y}_2^2
    \quad\leq\quad
    \mathrm{D}_U\lt(x, y\rt)
    \quad\leq\quad
    \frac{L}{2} \norm{x - y}_2^2 \; .
    \label{eq:bregman_divergence_convex_smooth}
\]
When $x$ and $y$ are replaced with two coupled random variables $X \sim p$ and $Y \sim q$, where $(X, Y) \sim \psi^*$ and $\psi^* \in \Psi\lt(p, q\rt)$ is the optimal coupling between $p$ and $q$, we can also notice that $\mathbb{E}_{ (X,Y) \sim \psi^*}\lt[\mathrm{D}_U\lt(X, Y\rt)\rt]$ is related to the Wasserstein distance as 
\[
    \frac{\mu}{2} {\mathrm{W}_2\lt(p, q\rt)}^2
    \leq
    \mathbb{E}_{ (X,Y) \sim \psi^*}\lt[\mathrm{D}_U\lt(X, Y\rt)\rt]
    \leq
    \frac{L}{2} {\mathrm{W}_2\lt(p, q\rt)}^2 \; .
    \label{eq:bregman_divergence_and_wasserstein_relation}
\]

In addition, under \cref{assumption:potential_convex_smooth}, it is well known~\citep[Lemma 2.29]{garrigos_handbook_2023} that, for any $x, y \in \mathbb{R}^d$, 
\[
    \norm{ \nabla U(x) - \nabla U(y) }_2^2
    \quad\leq\quad
    2 L \, \mathrm{D}_U\lt(x, y\rt) \; .
    \label{eq:bregman_divergence_cocoersivity}
\]

\newpage

\subsection{Miscellaneous Results}

The following is the multivariate version of Stein's identity~\citep{stein_estimation_1981}.

\begin{proposition}[Stein's Identity]\label{thm:steinidentity}
    Suppose $X \sim \mathrm{Normal}(\mu, \Sigma)$.
    Then, for any differentiable function $g : \mathbb{R}^d \to \mathbb{R}^d$, where, for all $i = 1, \ldots, d$, $\partial g / \partial x_i$ is continuous almost everywhere and $\mathbb{E} \abs{\partial g / \partial x_i} < +\infty$, we have
    \[
        \mathbb{E}\big[ \lt(X - \mu\rt) {g\lt(X\rt)}^{\top} \big]
        =
        \Sigma \,
        \mathbb{E}{\nabla g\lt(X\rt)} \; .
        \nonumber
    \]
\end{proposition}
\begin{proof}
    This is a direct consequence of integration by parts.
    See the full proof by \citet[Lemma 1]{liu_siegels_1994}.
\end{proof}

The next proposition is the multivariate analog of the Poincar\'e inequality specialized to Gaussians.

\begin{proposition}[Gaussian Poincar\'e Inequality]\label{thm:gpoincare_vector}
    Suppose $X \sim \mathrm{Normal}(\mu,\Sigma)$ and $g:\mathbb{R}^d\to\mathbb{R}^d$ is some continuously differentiable function satisfying 
    $\mathbb{E}\|\nabla g(X)\|_{\mathrm{F}}^2 < +\infty$, where $\nabla g(x)\in\mathbb{R}^{d\times d}$ is its Jacobian.
    Then
    \[
    \mathbb{E}\left\|g(X) - \mathbb{E}[g(X)]\right\|^2
        \leq
        \mathbb{E}\operatorname{tr} \nabla g(X)\Sigma\nabla g(X)^{\top} \; .
        \nonumber
    \]
\end{proposition}
\vspace{-1ex}

Formally, we say a probability measure $\nu$ satisfies a $C_{\mathrm{PI}}$-Poincar\'e inequality if there exists some $C_{\mathrm{PI}} \in (0, +\infty)$ such that, for any Lipschitz smooth function $f:\mathbb{R}^d\rightarrow\mathbb{R}$, the following inequality holds:
\[
\mathrm{Var}_{\nu}(f) \leq\frac{1}{C_{\mathrm{PI}}}\mathbb{E}_{\nu}\|\nabla f\|^2 \; .
\label{eq:vanilla_poincare}
\]
Then the result is a basic consequence of \cref{eq:vanilla_poincare}.

\vspace{-1ex}
\begin{proof}[Proof of \cref{thm:gpoincare_vector}]
Denote $g$ by $(g_1, ..., g_d)^\top$ and the coloring transform $M(\epsilon) \triangleq \Sigma^{\nicefrac{1}{2}} \epsilon + m$ such that $X = M(\epsilon)$ where $\epsilon \sim \mathrm{Normal}(0_d, \mathrm{I}_d)$.
Through $M$, we can convert the expectation to be over a standard Gaussian, which satisfies \cref{eq:vanilla_poincare} with $C_{\mathrm{PI}} = 1$~\citep{beckner_generalized_1989}.
\[
 &
 \mathbb{E}\left\|g(X) - \mathbb{E}[g(X)]\right\|^2 
 \nonumber
 \\
 &= {\textstyle\sum_{i=1}^d}\mathbb{E}(g_i(X) - \mathbb{E}[g_i(X)])^2
 \nonumber
 \\
 &= {\textstyle\sum_{i=1}^d}\mathbb{E}(g_i( M(\epsilon)) - \mathbb{E}[g_i(M(\epsilon))])^2
 \nonumber
 \\
 &\leq {\textstyle\sum_{i=1}^d}\mathbb{E}\|\nabla g_i(M(\epsilon))\|^2 
 &&\text{(\cref{eq:vanilla_poincare})}
 \nonumber
\]
By the chain rule, 
\[
    \nabla_{\epsilon} g(M(\epsilon)) = \nabla g(M(\epsilon)) \nabla M(\epsilon) = \nabla g(X)\Sigma^{\nicefrac{1}{2}} \; .
    \nonumber
\]
Then
\[
 \mathbb{E}\left\|g(X) - \mathbb{E}[g(X)]\right\|^2 
 &\leq \mathbb{E} \operatorname{tr}\nabla_{\epsilon} g(M(\epsilon))\nabla_{\epsilon} {g(M(\epsilon))}^\top \nonumber\\
 &= \mathbb{E} \operatorname{tr}\nabla g(X)\Sigma\nabla g(X)^{\top} \; . \nonumber
\]
\end{proof}

\citet{diao_forwardbackward_2023} also demonstrate that \cref{thm:gpoincare_vector} can be seen as a consequence of the Brescamp-Lieb inequality~\citep{brascamp_extensions_2002}, which generalizes \cref{thm:gpoincare_vector} to all strictly log-concave $\mathrm{Law}(X)$.

\newpage
In addition, we prove an upper bound on the trace of a product of positive definite matrices.

\begin{proposition}\label{thm:trace_of_product_upper_bound}
Suppose $A, B \in \mathbb{S}_{\succ 0}^d$ with $\norm{A}_2 < +\infty$.
Then
\[
    \operatorname{tr} A B A \quad\leq\quad \norm{A}_2 \operatorname{tr} A B \; .
    \nonumber
\]
\end{proposition}
\begin{proof}
Since $A$ is positive definite, there exists a collection of eigenvalues ${(\lambda_i)}_{i = 1}^d$ and eigenvectors ${(v_i)}_{i = 1}^d$ such that
\[
    A = {\textstyle\sum}_{i=1}^d \lambda_i v_i v_i^{\top} \; .
    \label{eq:trace_upper_bound_spectral_representation_of_A}
\]
Given this representation, 
\[
    \mathrm{tr} A B A
    &=
    \mathrm{tr} A^2 B
    \nonumber
    &&\text{(Cyclic property of $\operatorname{tr}$)}
    \\
    &=
    \operatorname{tr} \sum_{i=1}^d \lambda_i^2 v_i v_i^{\top}  B
    &&\text{(\cref{eq:trace_upper_bound_spectral_representation_of_A})}
    \nonumber
    \\
    &=
    \sum_{i=1}^d \lambda_i^2 \operatorname{tr}  v_i v_i^{\top}  B
    &&\text{(Linearity of $\operatorname{tr}$)}
    \nonumber
    \\
    &=
    \sum_{i=1}^d \lambda_i^2 \operatorname{tr} v_i^{\top} B v_i 
    &&\text{(Cyclic property of $\operatorname{tr}$)}
    \nonumber
\]
Now, from the fact that $A$ and $B$ are positive definite and that $\norm{A}_2 < + \infty$, for all $i = 1, \ldots, d$, we have
\[
    \lambda_i^2
    \leq
    \big(\max_{j=1, \ldots, d} \lambda_{j}\big) \lambda_i
    \nonumber
    \quad\text{and}\quad
    v_i^{\top} B v_i
    > 0 \; .
\]
Therefore, 
\[
    \mathrm{tr} A B A
    &\leq
    \big(\max_{i = 1, \ldots, d} \lambda_i \big) \sum_{i=1}^d \lambda_i \operatorname{tr} v_i^{\top} B v_i 
    \nonumber
    \\
    &=
    \big(\max_{i = 1, \ldots, d} \lambda_i \big) \sum_{i=1}^d  \operatorname{tr} \lambda_i v_i  v_i^{\top} B
    \nonumber
    \\
    &=
    \big(\max_{i = 1, \ldots, d} \lambda_i \big) \sum_{i=1}^d \operatorname{tr} A B
    &&\text{(\cref{eq:trace_upper_bound_spectral_representation_of_A})}
    \nonumber
    \\
    &=
    \norm{A}_2 \operatorname{tr} A B \; .
    \nonumber
    &&\text{($A \succ 0$)}
\]
\end{proof}

\newpage

\subsection{Stationarity Condition}

The following proposition characterizes the properties of the minimizer of $\mathcal{F}$, which also corresponds to the stationary point of the algorithms considered in this work.

\begin{proposition}[Stationary condition]\label{thm:stationary_condition}
    Suppose $q_* = \mathrm{Normal}(\mu_*, \Sigma_*) \in \argmin_{q \in \mathrm{BW}(\mathbb{R}^d)} \mathcal{F}(q)$. Then
    \[
    \mathbb{E}_{q_*}\nabla U = 0,\quad \mathbb{E}_{q_*}\nabla^2U = \Sigma_*^{-1}.\nonumber
    \]
\end{proposition}
\begin{proof}
    The Bures-Wasserstein gradient of $\mathcal{F}$ can be derived~\citep[Appendix C.1]{lambert_variational_2022} as, for each $q = \mathrm{Normal}(m, \Sigma) \in \mathrm{BW}(\mathbb{R}^d)$, 
    \[
        \nabla_{\mathrm{BW}} \mathcal{F}\lt(q\rt)
        \;=\;
        x \mapsto \mathbb{E}_{q}\nabla U + \lt( \mathbb{E}_{q} \nabla^2 U - \Sigma^{-1}\rt) \lt(x - m\rt)
        \nonumber
    \]
    Solving for the first-order stationarity condition immediately yields the result.
\end{proof}

\newpage

\subsection{Bound on the Covariance-Weighted Hessian Norm}

A crucial step in our analysis of the variance of the stochastic gradients is to bound the quantity $\mathbb{E}_{Z \sim q} \mathrm{tr} \; \nabla^2 U\lt(Z\rt) \, \Sigma \, \nabla^2 U\lt(Z\rt) $.
Specifically, we need to relate it to some notion of ``growth'' of $\mathcal{E}$.


The main technical contributions of our new results come from the fact that we relate $\mathbb{E}_{Z \sim q} \mathrm{tr} \; \nabla^2 U\lt(Z\rt) \, \Sigma \, \nabla^2 U\lt(Z\rt)$ with the Bregman divergence of $U$, $\mathrm{D}_U$.

\begin{lemma}\label{thm:covariance_weighted_hessian_bound}
Suppose \cref{assumption:potential_convex_smooth} holds and denote $q_* = \argmin_{q \in \mathrm{BW}(\mathbb{R}^d)} \mathcal{F}\lt(q\rt)$.
Then, for any $q = \mathrm{Normal}(m, \Sigma)$ and any coupling $\psi \in \Psi(q, q_*)$, 
\[
    &
    \mathbb{E}_{Z \sim q} \mathrm{tr} \; \nabla^2 U\lt(Z\rt) \, \Sigma \, \nabla^2 U\lt(Z\rt)
    \nonumber
    \\
    &\quad\qquad\leq
    L \lt( 2 \sqrt{\kappa} + \kappa \rt) \mathbb{E}_{ (X,X_*) \sim \psi} \lt[ \mathrm{D}_U\lt(X,X_*\rt) \rt]
    + 
    3 d L
    \; .
    \nonumber
\]
\end{lemma}

This contrasts with the previous analysis by \citet[Lemma 5.6]{diao_forwardbackward_2023}, who bounded $\mathbb{E}_{Z \sim q} \mathrm{tr} \; \nabla^2 U\lt(Z\rt) \, \Sigma \, \nabla^2 U\lt(Z\rt)$ by the Wasserstein distance ${\mathrm{W}_2\lt(q_t, q_*\rt)}^2$.
From \cref{eq:bregman_divergence_and_wasserstein_relation}, it is apparent that the Bregman divergence results in a tighter bound; it avoids paying for an additional factor of $L$, which is the source of the $\kappa$-factor improvement in \cref{thm:wasserstein_proximal_gradient_descent_pricestein}.


The first steps in the proof closely mirror the steps of \citet{diao_forwardbackward_2023} up to the error decomposition:
\[
    &
    \mathbb{E}_{Z \sim q} \mathrm{tr} \; \nabla^2 U\lt(Z\rt) \, \Sigma \, \nabla^2 U\lt(Z\rt) \; .
    \nonumber
\shortintertext{Applying \cref{thm:trace_of_product_upper_bound} with $\norm{\nabla^2 U}_2 \leq L$,}
    &\leq
    L
    \operatorname{tr} {\mathbb{E}_{q}\lt[\nabla^2 U\rt]} \Sigma 
    \nonumber
\shortintertext{and by Stein's identity (\cref{thm:steinidentity}),}
    &=
    L \,
    \mathbb{E}_{X \sim q}
    \inner{ \nabla U\lt(X\rt), X - m }
    \nonumber
    \\
    &=
    L
    \underbrace{
    \mathbb{E}
    \inner{ \nabla U\lt(X_*\rt), X_* - m_* }
    }_{\triangleq E_1}
    \nonumber
    \\
    &\qquad
    +
    L
    \underbrace{
    \mathbb{E}
    \inner{ \nabla U\lt(X\rt) - \nabla U\lt(X_*\rt), \lt( X - m \rt) - \lt( X_* - m_* \rt) }
    }_{\triangleq E_2}
    \nonumber
    \\
    &\qquad
    +
    L
    \underbrace{
    \mathbb{E}
    \inner{ \nabla U\lt(X_*\rt), \lt( X - m \rt) - \lt( X_* - m_* \rt) }
    }_{\triangleq E_3}
    \nonumber
    \\
    &\qquad
    +
    L
    \underbrace{
    \mathbb{E}
    \inner{ \nabla U\lt(X\rt) - \nabla U\lt(X_*\rt), X_* - m_* }
    }_{\triangleq E_4}
    \; .
    \label{eq:covariance_weighted_hessian_bound_eq1}
\]
Each error term $E_1, E_2, E_3, E_4$, however, will be bounded by the Bregman divergence instead of the squared Euclidean.
For this, we will repeatedly use the following result:
\newpage
\begin{lemma}[name={}]
    \label{thm:upper_bound_coupling_distance_with_potential}%
    Suppose \cref{assumption:potential_convex_smooth} holds.
    Then, for any two random variables $X, X_*$ on $\mathbb{R}^d$ satisfying $\mathbb{E}X = m$ and $\mathbb{E}X_* = m_*$, where $\norm{m}_2 < +\infty$ and $\norm{m_*}_2 < +\infty$, we have
    \[
        \mathbb{E}\norm{ (X - m) - (X_* - m_*) }_2^2
        \leq
        \frac{2}{\mu} \,
        \mathbb{E}\lt[\mathrm{D}_U\lt(Z, Z_*\rt)\rt] \; .
    \nonumber
    \]
\end{lemma}
\begin{proof}
\[
    &
    \mathbb{E}\norm{ (X - m) - (X_* - m_*) }_2^2
    \nonumber
    \\
    &\;\leq
    \mathbb{E}\norm{ X - X_* }_2^2
    &&\text{($\operatorname{tr}\mathrm{Var}(X) \leq \mathbb{E}\norm{X}_2^2$)}
    \nonumber
    \\
    &\;\leq
    \frac{2}{\mu}
    \mathbb{E}\lt[\mathrm{D}_U\lt(Z, Z_*\rt)\rt] \; .
    &&\text{(\cref{eq:bregman_divergence_convex_smooth})}
    \nonumber
\]
\end{proof}

Let's proceed to the proof of \cref{thm:covariance_weighted_hessian_bound}.

\begin{proof}[Proof of \cref{thm:covariance_weighted_hessian_bound}]
From \cref{eq:covariance_weighted_hessian_bound_eq1}, we have
\[
    \mathbb{E}_{Z \sim q} \mathrm{tr} \; \nabla^2 U\lt(Z\rt) \, \Sigma \, \nabla^2 U\lt(Z\rt)
    \leq
    L \lt(E_1 + E_2 + E_3 + E_4\rt) \; .
    \nonumber
\]
The error terms can be bounded as follows.

For $E_1$, Stein's identity (\cref{thm:steinidentity}) states that
\[
    E_1 
    &= 
    \mathbb{E}_{X_* \sim q_*}
    \inner{ \nabla U\lt(X_*\rt), X_* - m_* } 
    \nonumber
    \\
    &= 
    \operatorname{tr}{\lt( \mathbb{E}_{q_*}\lt[\nabla^2 U\rt]\Sigma_* \rt) }
    \nonumber
\shortintertext{and the stationary condition (\cref{thm:stationary_condition}) yields}
    &= 
    \operatorname{tr}{\lt( {\Sigma_*}^{-1} \Sigma_* \rt) }
    \nonumber
    \\
    &= 
    d \; .
    \nonumber
\]

For $E_2$, Young's inequality yields
\[
    E_2 
    &=
    \mathbb{E}_{ (X, X_*) \sim \psi} \big\langle \nabla U\lt(X\rt) - \nabla U\lt(X_*\rt), 
    \nonumber
    \\
    &\qquad\qquad\qquad\qquad
    \lt( X - m \rt) - \lt( X_* - m_* \rt) \big\rangle
    \nonumber
    \\
    &\leq
    \frac{1}{2 \sqrt{L \mu}} \, \mathbb{E}_{ (X, X_*) \sim \psi} \norm{ \nabla U\lt(X\rt) - \nabla U\lt(X_*\rt)}_2^2
    \nonumber
    \\
    &\qquad
    +
    \frac{\sqrt{L \mu}}{2} \, \mathbb{E}_{(X, X_*) \sim \psi} \norm{ \lt( X - m \rt) - \lt( X_* - m_* \rt) }_2^2 \; .
    \nonumber
\shortintertext{Then apply \cref{eq:bregman_divergence_cocoersivity}}
    &\leq
    \sqrt{\kappa} \, \mathbb{E}_{ (X, X_*) \sim \psi}\lt[\mathrm{D}_U\lt(X, X_*\rt)\rt]
    \nonumber
    \\
    &\qquad
    +
    \frac{\sqrt{L \mu}}{2} \mathbb{E}_{(X, X_*) \sim \psi} \norm{ \lt( X - m \rt) - \lt( X_* - m_* \rt) }_2^2
    \nonumber
\shortintertext{and \cref{thm:upper_bound_coupling_distance_with_potential} such that}
    &\leq
    \sqrt{\kappa} \, \mathbb{E}_{ (X, X_*) \sim \psi}\lt[\mathrm{D}_U\lt(X, X_*\rt)\rt]
    \nonumber
    \\
    &\qquad
    +
    \sqrt{\kappa}
    \mathbb{E}_{ (X, X_*) \sim \psi}\lt[ \mathrm{D}_U\lt(X, X_*\rt) \rt]
    \nonumber
    \\
    &\leq
    2 \sqrt{\kappa} \, \mathbb{E}_{ (X, X_*) \sim \psi}\lt[\mathrm{D}_U\lt(X, X_*\rt)\rt] \; .
    \nonumber
\]

For $E_3$, we begin by applying Young's inequality as
\[
    E_3
    &=
    \mathbb{E}_{ (X, X_*) \sim \psi}
    \inner{ \nabla U\lt(X_*\rt), \lt( X - m \rt) - \lt( X_* - m_* \rt) }
    \nonumber
    \\
    &\leq
    \frac{1}{L} \,
    \mathbb{E}_{ (X, X_*) \sim \psi} \norm{ \nabla U\lt(X_*\rt)}_2^2
    \nonumber
    \\
    &\qquad
    +
    \frac{L}{4}
    \mathbb{E}_{ (X, X_*) \sim \psi} \norm{ \lt( X - m \rt) - \lt( X_* - m_* \rt) }_2^2 \; .
    \nonumber
\shortintertext{Applying \cref{thm:upper_bound_coupling_distance_with_potential} yields}
    &\leq
    \frac{1}{L}
    \mathbb{E}_{ (X, X_*) \sim \psi} \norm{ \nabla U\lt(X_*\rt)}_2^2
    +
    \frac{\kappa}{2} \,
    \mathbb{E}\lt[\mathrm{D}_U\lt(X, X_*\rt)\rt] \;. 
    \nonumber
\shortintertext{The stationary condition $\mathbb{E}_{q_*}\lt[\nabla U\rt] = 0$ yields}
    &=
    \frac{1}{L}
    \mathbb{E}_{ (X, X_*) \sim \psi} \norm{ \nabla U\lt(X_*\rt) - \mathbb{E}_{q_*}\lt[\nabla U\rt] }_2^2
    \nonumber
    \\
    &\qquad
    +
    \frac{\kappa}{2} \,
    \mathbb{E}\lt[\mathrm{D}_U\lt(X, X_*\rt)\rt] \; .
    \nonumber
\shortintertext{Then, by the Gaussian Poincar\'e inequality (\cref{thm:gpoincare_vector}),}
    &\leq
    \frac{1}{L}
    \mathbb{E}_{q_*} \operatorname{tr}\big( {\lt(\nabla^2 U\rt)}^2 \Sigma_* \big)
    +
    \frac{\kappa}{2} \, \; 
    \mathbb{E}\lt[\mathrm{D}_U\lt(X, X_*\rt)\rt] \; .
    \nonumber
\shortintertext{Since $U$ is $L$-smooth under \cref{assumption:potential_convex_smooth},}
    &\leq
    \operatorname{tr}\big( \mathbb{E}_{q_*}\lt[\nabla^2 U\rt] \Sigma_* \big)
    +
    \frac{\kappa}{2} \,
    \mathbb{E}_{(X,X_*) \sim \psi}\lt[\mathrm{D}_U\lt(X, X_*\rt)\rt]
    \nonumber
\shortintertext{and applying the stationary condition (\cref{thm:stationary_condition}),}
    &=
    \operatorname{tr}\big( \Sigma_*^{-1} \Sigma_* \big)
    +
    \frac{\kappa}{2} \,
    \mathbb{E}_{(X,X_*) \sim \psi}\lt[\mathrm{D}_U\lt(X, X_*\rt)\rt]
    \nonumber
    \\
    &=
    d
    +
    \frac{\kappa}{2} \, \mathbb{E}_{(X,X_*) \sim \psi}\lt[\mathrm{D}_U\lt(X, _*\rt)\rt] \; .
    \nonumber
\]

For $E_4$, we again begin with Young's inequality.
\[
    E_4
    &=
    \mathbb{E}_{ (X,X_*) \sim \psi }
    \inner{ \nabla U\lt(X\rt) - \nabla U\lt(X_*\rt), X_* - m_* }
    \nonumber
    \\
    &\leq
    \frac{1}{4 \mu} \,
    \mathbb{E}_{ (X,X_*) \sim \psi }
    \norm{\nabla U\lt(X\rt) - \nabla U\lt(X_*\rt) }_2^2
    \nonumber
    \\
    &\qquad
    +
    \mu \,
    \mathbb{E}_{X_* \sim q_*} \norm{X_* - m_* }_2^2 \; .
    \nonumber
\shortintertext{From \cref{eq:bregman_divergence_cocoersivity},}
    &\leq
    \frac{\kappa}{2} \,
    \mathbb{E}_{ (X,X_*) \sim \psi }\lt[ \mathrm{D}_U\lt(X, X_*\rt) \rt]
    \nonumber
    \\
    &\qquad
    +
    \mu \,
    \mathbb{E} \norm{X_* - m_* }_2^2
    \nonumber
    \\
    &=
    \frac{\kappa}{2} \, \mathbb{E}_{ (X,X_*) \sim \psi }\lt[ \mathrm{D}_U\lt(X, X_*\rt) \rt]
    +
    \mu \operatorname{tr}\lt( \Sigma_* \rt)
    \nonumber
\shortintertext{and the stationary condition (\cref{thm:stationary_condition}),}
    &=
    \frac{\kappa}{2} \, \mathbb{E}_{ (X,X_*) \sim \psi }\lt[ \mathrm{D}_U\lt(X, X_*\rt) \rt]
    +
    \mu \operatorname{tr}\big( {(\mathbb{E}_{q_*} \nabla^2 U)}^{-1} \big) \; .
    \nonumber
\shortintertext{Finally, from the fact that ${(\mathbb{E}_{q_*} \nabla^2 U)}^{-1} \preceq (1/\mu) \mathrm{I}_d$,}
    &\leq
    \frac{\kappa}{2} \, \mathbb{E}_{ (X,X_*) \sim \psi }\lt[ \mathrm{D}_U\lt(X, X_*\rt) \rt]
    +
    d \; .
    \nonumber
\]

Combining all the results, 
\[
    &
    \mathbb{E}_{Z \sim q} \mathrm{tr} \; \nabla^2 U\lt(Z\rt) \, \Sigma \, \nabla^2 U\lt(Z\rt)
    \nonumber
    \\
    &\;\leq
    L \{ E_1 + E_2 + E_3 + E_4 \}
    \nonumber
    \\
    &\;\leq
    L
    \mathbb{E}_{ (X,X_*) \sim \psi}
    \bigg\{
    d + 2 \sqrt{\kappa} \, \mathrm{D}_U\lt(X,X_*\rt)
    \nonumber
    \\
    &\qquad
    + \lt( d + \frac{\kappa}{2} \, \mathrm{D}_U\lt(X,X_*\rt) \rt)
    + \lt( \frac{\kappa}{2} \, \mathrm{D}_U\lt(X,X_*\rt) + d \rt)
    \bigg\}
    \nonumber
    \\
    &\;=
    L \lt( 2 \sqrt{\kappa} + \kappa \rt) \mathbb{E}_{ (X,X_*) \sim \psi} \lt[ \mathrm{D}_U\lt(X,X_*\rt) \rt]
    + 
    3 d L
    \nonumber
    \; .
\]
\end{proof}

\newpage

\subsection{Variance Bound for Bonnet's Gradient Estimator}

\subsubsection{\cref{thm:potential_gradient_variance_bound}}

The following result will be used to bound the variance of Bonnet's gradient estimator~\citep{bonnet_transformations_1964} for the location component $m$ of any $q = \mathrm{Normal}(m, \Sigma) \in \mathrm{BW}(\mathbb{R}^d)$
{%
\setlength{\belowdisplayskip}{1ex} \setlength{\belowdisplayshortskip}{1ex}
\setlength{\abovedisplayskip}{1ex} \setlength{\abovedisplayshortskip}{1ex}
\[
    \widehat{\nabla_m^{\text{bonnet}} \mathcal{E}}(q; \epsilon)
    \triangleq
    \nabla U(Z) \, ,
    \quad\text{where}\quad
    Z \sim q \; ,
    \nonumber
\]
}%
where $\epsilon$ is the randomness needed to sample $Z$ from $q$.

\begin{lemma}[name={},restate=thmwassersteingradientvariancelocation,label={thm:potential_gradient_variance_bound}]
Suppose \cref{assumption:potential_convex_smooth} holds and $q_* = \mathrm{Normal}(\mu_*, \Sigma_*) \in \argmin_{q \in \mathrm{BW}(\mathbb{R}^d)} \mathcal{F}(q)$. 
Then, for any $q \in \mathrm{BW}(\mathbb{R}^d)$ and any coupling $\psi \in \Psi\lt(q, q_*\rt)$, 
{%
\setlength{\belowdisplayskip}{1.5ex} \setlength{\belowdisplayshortskip}{1.5ex}
\setlength{\abovedisplayskip}{1.5ex} \setlength{\abovedisplayshortskip}{1.5ex}
\[
    &
    \mathbb{E}_{Z \sim q} \norm*{\nabla U\lt(Z\rt) - \mathbb{E}_{q_*} \nabla U }_2^2
    \nonumber
    \\
    &\qquad\qquad\leq
    4 L \, \mathbb{E}_{(X, X_*) \sim \psi} \lt[ \mathrm{D}_U\lt(X, X_*\rt) \rt]
    +
    2 d L  \; .
    \nonumber
\]
}%
\end{lemma}

\vspace{-1.5ex}
This serves a similar purpose as the usual ``variance transfer'' lemma used to analyze SGD on expected risk minimization-type problems~\citep[Lemma 8.21]{garrigos_handbook_2023}.
Notice that the result does not specify the coupling $\psi$ and generalizes to all couplings in $\Psi(q, q_*)$.

\vspace{-1.5ex}

\begin{proof}[Proof of \cref{thm:potential_gradient_variance_bound}]
First, decompose the gradient variance using Young's inequality.
{%
\setlength{\belowdisplayskip}{.5ex} \setlength{\belowdisplayshortskip}{.5ex}
\setlength{\abovedisplayskip}{1.5ex} \setlength{\abovedisplayshortskip}{1.5ex}
\[
    &
    \mathbb{E}_{Z \sim q} \norm*{\nabla U\lt(Z\rt) - \mathbb{E}_{q_*} \nabla U }_2^2
    \nonumber
    \\
    &=
    \mathbb{E}_{(Z, Z_*) \sim \psi}
    \norm*{\nabla U\lt(Z\rt) - \nabla U\lt(Z_*\rt) + \nabla U\lt(Z_*\rt) - \mathbb{E}_{q_*} \nabla U }_2^2
    \nonumber
    \\
    &\leq
    2 \, \underbrace{
        \mathbb{E}_{(Z, Z_*) \sim \psi}\norm*{\nabla U\lt(Z\rt) - \nabla U\lt(Z_*\rt)}
    }_{\triangleq V_{\text{mult}}}
    \nonumber
    \\
    &\qquad
    + 
    2 \, \underbrace{
    \mathbb{E}_{Z_* \sim q_*}\norm*{\nabla U\lt(Z_*\rt) - \mathbb{E}_{q_*} \nabla U }_2^2
    }_{\triangleq V_{\text{add}}}
    \nonumber
\]
}%
The multiplicative noise follows from the $L$-smoothness of $U$.
By applying \cref{eq:bregman_divergence_convex_smooth}, 
{%
\setlength{\belowdisplayskip}{0ex} \setlength{\belowdisplayshortskip}{0ex}
\setlength{\abovedisplayskip}{1.5ex} \setlength{\abovedisplayshortskip}{1.5ex}
\[
    V_{\text{mult}}
    &=
    \mathbb{E}_{(Z, Z_*) \sim \psi}\norm*{\nabla U\lt(Z\rt) - \nabla U\lt(Z_*\rt)}_2^2
    \nonumber
    \\
    &\leq
    2 L \, \mathbb{E}_{(Z, Z_*) \sim \psi} \big( U\lt(Z\rt) - U\lt(Z_*\rt) 
    \nonumber
    \\
    &\qquad\qquad\qquad
    - \inner{ \nabla U\lt(Z_*\rt), Z_t - Z_* } \big)
    \nonumber
    \\
    &=
    2 L \, \mathbb{E}_{(Z, Z_*) \sim \psi} \mathrm{D}_U\lt(Z, Z_*\rt) \; .
    \label{eq:wass_location_t_mult_bound}
\]
}%

The additive noise, on the other hand, follows from the Gaussian Poincar\'e inequality and the $L$-smoothness of $U$.
\[
    V_{\text{add}}
    &=
    \mathbb{E}_{Z_* \sim q_*}\norm*{\nabla U\lt(Z_*\rt) - \mathbb{E}_{q_*} \nabla U }_2^2
    \nonumber
\shortintertext{Due to the Gaussian Poincar\'e inequality (\cref{thm:gpoincare_vector}),}
    &\leq
    \mathbb{E}_{Z_* \sim q_*} \operatorname{tr}{ (\nabla^2 U(Z_*)) \Sigma_* (\nabla^2 U(Z_*)) }
    \nonumber
\shortintertext{and applying \cref{thm:trace_of_product_upper_bound} with $\norm{\nabla^2 U}_2 \leq L$,}
    &\leq
    L \operatorname{tr}(\Sigma_* \mathbb{E}_{q_*}\lt[\nabla^2 U\rt]) \; .
    \nonumber
\shortintertext{The stationary condition (\cref{thm:stationary_condition}) yields}
    &=
    L \operatorname{tr}(\Sigma_* \Sigma_*^{-1})
    \nonumber
    \\
    &=
    d L \; .
    \label{eq:wass_location_t_add_bound}
\]
Combining \cref{eq:wass_location_t_mult_bound,eq:wass_location_t_add_bound} yields the statement.
\end{proof}

\newpage

\subsection{Lyapunov Convergence Lemma}\label{section:lyapunov}

In this section, we will provide an auxiliary result that will be used throughout this work for analyzing the convergence of stochastic first-order algorithms on strongly convex objectives.

\subsubsection{\cref{thm:lyapunov_convergence_lemma}}

Historically, convergence guarantees for stochastic first-order methods on strongly convex objectives come in two flavors: results based on a fixed step size or a decreasing step size schedule.
Consider some distance metric $\mathrm{d}(\cdot, \cdot)$, denote the iterates generated by SGD as ${(x_t)}_{t = 0}^T$ and the unique global optimum as $x_*$.
For any target accuracy level $\epsilon > 0$, to obtain an $\epsilon$-accurate solution such that $\mathbb{E}\, {\mathrm{d}(x_T, x_*)}^2 \leq \epsilon$, SGD with a \textit{fixed step size} results in an iteration complexity of $\mathrm{O}(1/\epsilon \log (\Delta_0 / \epsilon) )$~\citep[Corollary 5.9]{garrigos_handbook_2023}.
Notice that the dependence on $\epsilon$ is $\mathrm{O}(1/\epsilon \log 1/\epsilon)$, while the dependence on the initial distance ${\mathrm{d}(x_0, x_*)}$ is logarithmic.
On the other hand, a \textit{decreasing step size schedule}~\citep{shamir_stochastic_2013,lacoste-julien_simpler_2012,gower_sgd_2019} is able to improve the dependence on $\epsilon$ to $\mathrm{O}(1/\epsilon)$.
For strongly convex and smooth objectives,~\citet{gower_sgd_2019} showed that the two-stage schedule in \cref{eq:stepsize_schedule} can achieve an iteration complexity of $\mathrm{O}(1/\epsilon + {\mathrm{d}(x_0, x_*)}^2 /\sqrt{\epsilon} )$.
Unfortunately, however, the dependence on the distance ${\mathrm{d}(x_0, x_*)}$ is now polynomial instead of logarithmic.

An improvement was presented by \citet{stich_unified_2019} by relying on a step size schedule that optimizes the dependence on both $\epsilon$ and $\mathrm{d}\lt(x_0, x_*\rt)$ simultaneously.
However, the step size schedule of \citeauthor{stich_unified_2019} requires knowing the maximum number of iterations $T$, which means $T$ must be fixed before executing the optimization run.
As such, the schedule does not provide an \textit{any-time} convergence guarantee.
Since then, \citet[Proposition 2.9]{kim_nearly_2025} presented a refined schedule that does not rely on $T$, and therefore provides any-time convergence guarantees.
We present this result in a more general form following the Lyapunov style of convergence analysis~\citep{wilson_lyapunov_2018,dieuleveut_stochastic_2023,bansal_potentialfunction_2019}.

\newpage

Consider some dynamical system generating the state variable sequence ${(x_t)}_{t \geq 0}$, where, for each $t \geq 0$, $x_t \in \mathcal{X}$ controlled by some sequence of step sizes ${(\gamma_t)}_{t \geq 0}$.
In our case, ${(x_t)}_{t \geq 0}$ will be an iterate sequence generated by some corresponding optimization algorithm.
Suppose there exists some Lyapunov function $V : \mathcal{X} \to \mathbb{R}_{\geq 0}$ quantifying the energy of the dynamical system.
Denoting $V_{t} \triangleq V(x_t)$, our interest is the sufficient number of iterations $T$ and the conditions on the step size sequence ${(\gamma_t)}_{t \geq 0}$ for the system ${(x_t)}_{t \geq 0}$ to achieve $\epsilon$-Lyapunov stability: $V_T = V(x_T) < \epsilon$.
This is given by the following proposition.

\begin{proposition}\label{thm:lyapunov_convergence_lemma}
    Consider a sequence of Lyapunov function values ${(V_{t})}_{t = 0}^T$ associated with some dynamical system controlled by some bounded step size sequence ${(\gamma_t)}_{t=0}^T$, where $\gamma_{t} \leq \gamma_{\mathrm{max}}$ for some $\gamma_{\mathrm{max}} \in (0,+\infty)$.
    Suppose there exist some constants $\mu \in (0, +\infty)$ and $b \in [0, +\infty)$ such that, for all $t \geq 0$, the sequence satisfies the Lyapunov condition
    \[
        V_{t+1} - V_t \quad\leq\quad - \mu \gamma_t  V_t + b \gamma_t^2 \; .
        \label{eq:lyapunov_onestep_condition}
    \]
    Then, if the step size schedule in \cref{eq:stepsize_schedule} is used with some $\gamma_{0} \leq \gamma_{\mathrm{max}}$ and the remaining parameters set as
    \[
        \tau &= t_* + \frac{2}{\gamma_0 \mu}
        \nonumber
        \\
        t_* &= 
        \lt\lceil
        \frac{1}{\log 1/\rho}
        \log\lt(
            \frac{\mu}{\gamma_0 b} V_0
        \rt) 
        \rt\rceil
        \; ,
        \nonumber
    \]
    for any $\epsilon > 0$, we have 
    \[
        T \geq
        \max\lt\{ B_{\mathrm{var}}, B_{\mathrm{bias}} \rt\}
        \quad\Rightarrow\quad
        V_T \leq \epsilon
        \; ,
    \nonumber
    \]
    where $\rho \triangleq 1 - \mu \gamma_0$, 
    \[
        B_{\mathrm{var}}
        &= 
        \frac{4 b}{\mu} \frac{1}{\mu \epsilon} 
        \nonumber
        + 
        \frac{4 \sqrt{b}}{\sqrt{\mu}}
        \frac{1}{\sqrt{ \mu \gamma_0}}
        \\
        &\quad\times
        \lt\{
        \log\lt( \frac{\mu V_0 }{b \gamma_0} \rt) 
        +
        \mu \gamma_0
        +
        \sqrt{2}
        \rt\}
        \frac{1}{\sqrt{\mu \epsilon}} 
    \nonumber
        \\
        B_{\mathrm{bias}}
        &= 
        \frac{1}{\mu \gamma_0}
        \log\lt( 2 V_0 \frac{1}{\epsilon} \rt) \; .
    \nonumber
    \]
\end{proposition}

\newpage
\subsubsection{Proof of \cref{thm:lyapunov_convergence_lemma}}

Under the two-stage step size schedule in \cref{eq:stepsize_schedule}, the sequence exhibits two different regimes of convergence behavior. 
In the first stage, where $\gamma_t = \gamma_0$, we have the following:

\begin{lemma}\label{thm:lyapunov_convergence_lemma_constant_stepsize}
Suppose there exists some $\mu \in (0, +\infty)$ and $b \in [0, +\infty)$ such that, for some $t_* \geq 0$ and for all $t \in \{0, \ldots, t_*\}$, \cref{eq:lyapunov_onestep_condition} holds and the step size schedule is constant such that $\gamma_t = \gamma_0$.
Then
\[
    V_{t_*} 
    \quad\leq \quad
    \rho^{t_*} V_{0} + \frac{b}{\mu} \gamma_0 \; .
    \nonumber
\]    
\end{lemma}
\begin{proof}
Under the stated assumptions, \cref{eq:lyapunov_onestep_condition} implies
\[
    V_{t+1} \quad\leq\quad \rho V_{t} + b \gamma_0^2 \; .
    \nonumber
\]    
Unrolling this over $t = 0, \ldots, t_* - 1$, 
\[
    V_{T} 
    &\leq
    \rho^{T} V_{0} + b \gamma_0^2 \sum_{t=0}^{t_* - 1} {\lt(1 - \mu \gamma_0\rt)}^{t} 
    \nonumber
    \\
    &\leq
    {\rho}^{t_*} V_{0} + \frac{b}{\mu} \gamma_0
    \; ,
    \nonumber
\]    
where the last step follows from the geometric series sum formula.
\end{proof}

For the second stage, we have a decreasing step size schedule.
It is well known in the stochastic optimization literature that the step size schedule in \cref{eq:stepsize_schedule} yields a rate that is equivalent to $\mathrm{O}(1/T)$ asymptotically in $T$ for both non-smooth~\citep{lacoste-julien_simpler_2012,shamir_stochastic_2013,stich_unified_2019} and smooth objectives~\citep{gower_sgd_2019}.
\begin{lemma}\label{thm:lyapunov_convergence_lemma_decreasing_stepsize}
Suppose, for all $t \in \{t_*, \ldots, T-1\}$, there exists some $\mu \in (0, +\infty)$ and $b \in [0, +\infty)$ such that \cref{eq:lyapunov_onestep_condition} holds and the step size schedule satisfies, for all $t \in \{t_*, \ldots, T-1\}$, some $\tau > 0$,
\[
    \gamma_t = \frac{1}{\mu} \frac{2 (t + \tau) + 1}{ {(t + \tau + 1)}^2 } \; .
    \nonumber
\]
Then
\[
    V_{T} 
    \leq
    \frac{ {\lt(t_* + \tau\rt)}^2 }{ {\lt(T + \tau\rt)}^2 }
    V_{t_*}
    +
    \frac{4 b}{\mu^2} \, \frac{T - t_*}{{(T + \tau)}^2}
    \nonumber
    \; .
\]    
\end{lemma}
\begin{proof}
Under the assumptions, \cref{eq:lyapunov_onestep_condition} becomes
\[
    V_{t+1} - V_t &\leq 
    - \frac{2 (t + \tau) + 1}{ {\lt(t + \tau + 1\rt)}^{2} } V_t + \frac{b}{\mu^2} \frac{ {\lt( 2 (t + \tau) + 1 \rt)}^2 }{ {\lt(t + \tau + 1\rt)}^{4} }
    \nonumber
\]
Multiplying ${(t + \tau + 1)}^2$ to both hand sides,
\[
    &
    {\lt(t + \tau + 1\rt)}^2 V_{t+1} 
    - {\lt(t + \tau + 1\rt)}^2 V_t 
    \nonumber
    \\
    &\quad\leq - \lt(2 (t + \tau) + 1\rt) V_t + \frac{b}{\mu^2} \frac{ {\lt( 2 (t + \tau) + 1 \rt)}^2 }{ {\lt(t + \tau + 1\rt)}^{2} }  \; ,
    \nonumber
\]
and moving the $V_t$ term on the right-hand side to the left,
\[
    {\lt(t + \tau + 1\rt)}^2 V_{t+1} - {(t + \tau)}^2 V_t 
    &\leq \frac{b}{\mu^2} \frac{ {\lt( 2 (t + \tau) + 1 \rt)}^2 }{ {\lt(t + \tau + 1\rt)}^{2} } 
    \; .
    \nonumber
\]
Summing up the inequality over $t \in \{t_*, \ldots, T - 1\}$ yields
\[
    {(T + \tau)}^2 V_{T} - {\lt(t_* + \tau\rt)}^2 V_{0}
    &\leq 
    \sum_{t = t_*}^{T-1} \frac{b}{\mu^2} \frac{ {\lt( 2 (t + \tau) + 1 \rt)}^2 }{ {\lt(t + \tau + 1\rt)}^{2} } 
    \nonumber
    \\
    &\leq 
    \sum_{t = t_*}^{T-1} \frac{4 b}{\mu^2} 
    \nonumber
    \\
    &= 
    \frac{4 b}{\mu^2} \lt( T - t_* \rt) \; .
    \nonumber
\]
Re-organizing yields the inequality in the statement.
\end{proof}

When $t_* < T$, where both stages of the steps size kick in, we have to combine both \cref{thm:lyapunov_convergence_lemma_constant_stepsize} and \cref{thm:lyapunov_convergence_lemma_decreasing_stepsize}.
For \cref{eq:stepsize_schedule}, applying \cref{thm:lyapunov_convergence_lemma_constant_stepsize} for $t \in \{0, \ldots, t_* - 1\}$ yields
\[
    V_{t_*} \leq \rho^{t_*} V_0 + \frac{b}{\mu} \gamma_0 \; .
    \nonumber
\]
Then, for $t \in \{t_*, \ldots , T - 1\}$, since $\tau = 2/(\gamma_0 \mu)$, the schedule is bounded by $\gamma_0$ as 
\[
     \gamma_t 
     &= 
     \frac{1}{\mu} \frac{2 (t + \tau) + 1}{ {(t + \tau + 1)}^2 }
     \nonumber
     \\
     &\leq \frac{2}{\mu} \frac{1}{t + \tau + 1}
     \nonumber
     \\
     &\leq \frac{2}{\mu} \frac{1}{\tau}
     \nonumber
     \\
     &=
     \frac{2}{\mu} \frac{ \gamma_0 \mu }{ 2 }
     \nonumber
     \\
     &=
     \gamma_0 
     \\
     &\leq
     \gamma_{\mathrm{max}}
     \; .
     \nonumber
\]
Therefore, we can apply \cref{thm:lyapunov_convergence_lemma_decreasing_stepsize}, which yields
\[
    V_{T} 
    &\leq
    V_{t_*} \, \frac{ {(t_* + \tau)}^2 }{{(T + \tau)}^2}
    +
    \frac{4 b}{\mu^2} \, \frac{T - t_*}{{(T + \tau)}^2} \; .
    \nonumber
\]
Combining the two bounds,
\[
    V_T 
    &\leq
    \lt\{ \rho^{t_*} V_0 + \frac{b}{\mu} \gamma_0 \rt\} \frac{ {\lt(t_* + \tau\rt)}^2 }{ {(T + \tau)}^2 } 
    + \frac{4 b}{\mu^2} \frac{T - t_*}{{(T + \tau)}^2} 
    \; .
    \label{eq:lyapunov_combined_general_bound}
\]
For any fixed $t_*$, this already results in an asymptotic rate of $\mathrm{O}(1/T)$.
Optimizing the bound with respect to the switching time $t_*$, however, requires a bit of work.


\newpage
\begin{proof}[Proof of \cref{thm:lyapunov_convergence_lemma}]
For the total number of iterations $T$, we have to separately consider the case where $T \leq t_*$ and $T > t_*$.
That is, the cases where the second stage doesn't kick in at all, and when it does.

First, in the case of $t_* \geq T$, we have the implications
\[
    T &\leq t_*
    && \Leftrightarrow &
    T &\leq \lt\lceil \frac{1}{\log 1/\rho} \log\lt(\frac{\mu}{\gamma_0 b}  V_0 \rt) \rt\rceil
    \nonumber
    \\
    &
    &&\Rightarrow &
    T - 1 &\leq \frac{1}{\log 1/\rho} \log\lt(\frac{\mu}{\gamma_0 b}  V_0 \rt)
    \nonumber
    \\
    &
    && \Leftrightarrow &
    V_0 \rho^{T - 1} &\geq \frac{b}{\mu} \gamma_0 
    \label{eq:lyapunov_preasymptotic_regime_restated_case1}
    \; .
\]
Furthermore, the second stage never kicks in. 
Therefore, we can invoke \cref{thm:lyapunov_convergence_lemma_constant_stepsize}, which yields
\[
    V_T
    \quad\leq\quad
    \rho^{T} V_0 + \frac{b}{\mu} \gamma_0
    \quad\leq\quad
    2 \rho^{T - 1} V_0 \; .
    &&\text{(\cref{eq:lyapunov_preasymptotic_regime_restated_case1})}
    \nonumber
\]
Let's identify the condition on $T$ to ensure $V_T \leq \epsilon$. 
This follows from 
\[
    V_T &\leq \epsilon
    &&\Leftrightarrow &
    2 \rho^{T} V_0
    &\leq
    \epsilon
    \nonumber
    \\
    &
    && \Leftrightarrow &
    T 
    &\geq
    \frac{1}{\log \lt(1/\rho\rt)}
    \log\lt(
    2 V_0
    \frac{1}{\epsilon}
    \rt)
    \nonumber
    \\
    &
    && \Leftarrow &
    T
    &\geq
    \frac{1}{1 - \rho}
    \log\lt(
    2 V_0
    \frac{1}{\epsilon}
    \rt) \; ,
    \nonumber
\]
where the last step used the inequality $\log (1/\rho) \geq 1 - \rho $.
Therefore, in this regime, 
\[
    T
    \;\geq\;
    B_{\mathrm{bias}}
    \qquad\Rightarrow\qquad
    V_T \leq \epsilon \; .
    \label{eq:lyapunov_Tlowerbound_case1}
\]

Let's turn to $t_* < T$.
Notice that
\[
    t_*  &= \lt\lceil \frac{1}{\log (1/\rho)} \log\lt( \frac{\mu V_0 }{b \gamma_0} \rt) \rt\rceil
    \nonumber
    \\
    &>
    \frac{1}{\log (1/\rho)} \log\lt( \frac{\mu V_0 }{b \gamma_0} \rt)
    \; .
    \nonumber
\]
Therefore, 
\[
    V_0 \rho^{t_*} > \frac{b}{\mu} \gamma_0  \; .
    \nonumber
\]
Then \cref{eq:lyapunov_combined_general_bound} can be developed as
\[
    V_T
    &\leq
    \lt\{ \frac{2 b}{\mu} \gamma_0 \rt\} \frac{{\lt(t_* + \tau\rt)}^2}{{(T + \tau)}^2} + \frac{4 b}{\mu^2} \frac{T - t_*}{{(T + \tau)}^2} \; .
    \nonumber
    \\
    &\leq
    \frac{2 b}{\mu} \gamma_0 \frac{ {\lt(t_* + \tau\rt)}^2 }{T^2} + \frac{4 b}{\mu^2} \frac{1}{T} \; .
    \nonumber
\shortintertext{Substituting our choice of $\tau = 2/(\gamma_0 \mu)$,} 
    &=
    \frac{2 b}{\mu} \gamma_0 {\lt(t_* + \frac{2}{\gamma_0 \mu} \rt)}^2 \frac{1}{T^2} + \frac{4 b}{\mu^2} \frac{1}{T} \; ,
    \nonumber
\shortintertext{and applying Young's inequality,}
    &\leq
    4
    \lt(
    \frac{b \gamma_0}{\mu} t_*^2
    +
    2
    \frac{b}{\gamma_0 \mu^3} 
    \rt)
    \frac{1}{T^2} + \frac{4 b}{\mu^2} \frac{1}{T}
    \nonumber
    \\
    &=
    \alpha \frac{1}{T^2} + \beta \frac{1}{T} \; ,
    \nonumber
\]
where the upper bound is now a quadratic function in $1/T$ with the coefficients
\[
    \alpha =  \frac{4 b \gamma_0}{\mu} t_*^2 + \frac{8 b}{\gamma_0 \mu^{3}} 
    \qquad\text{and}\qquad
    \beta = \frac{4b}{\mu^2} \; .
    \nonumber
\]

To ensure the condition $V_T \leq \epsilon$ we must now identify the smallest $T > 0$ that satisfies the inequality 
\[
    \alpha \frac{1}{T^2} + \beta \frac{1}{T} \quad\leq\quad \epsilon \; .
    \nonumber
\]
By the basic properties of quadratics, the quadratic formula yields the condition
\[
    & &
    \frac{1}{T} \quad&\leq\quad \frac{-\beta + \sqrt{\beta^2 + 4 \alpha \epsilon}}{2 \alpha} 
    \nonumber
    \\
    &  &
    \quad&\leq\quad 
    \frac{4 \alpha \epsilon}{4 \alpha \sqrt{ \beta^2 + 4 \alpha \epsilon} } 
    &&\text{\citep{symbol-1_answer_2022}}
    \nonumber
    \\
    &  &
    \quad&\leq\quad 
    \frac{\epsilon}{\sqrt{ \beta^2 + 4 \alpha \epsilon} } 
    \nonumber
    \\
    & \Leftarrow &
    T
    \quad&\geq\quad 
    \beta \frac{1}{\epsilon} + 2 \sqrt{\alpha} \frac{1}{\sqrt{\epsilon}} \; ,
    \nonumber
\]
where the last step used the inequality $\sqrt{\alpha + \beta} \leq \sqrt{\alpha} + \sqrt{\beta}$.
That is, in this regime, 
\[
    T \geq \frac{4 b}{\mu^2} \frac{1}{\epsilon} 
    + 2 \lt\{
    \sqrt{ \frac{4 b \gamma_0}{\mu} } t_* 
    + 
    \sqrt{
    \frac{8 b}{\gamma_0 \mu^{3}} 
    }
    \rt\} \frac{1}{\sqrt{\epsilon}}
    \;\;\Rightarrow\;\;
    V_T \leq \epsilon
    \nonumber
\]
Since
\[
    t_*  
    &< \frac{1}{\log (1/\rho)} \log\lt( \frac{\mu V_0 }{b \gamma_0} \rt) + 1 
    \nonumber
    \\
    &< \frac{1}{\mu \gamma_0} \log\lt( \frac{\mu V_0 }{b \gamma_0} \rt) + 1  \; ,
    &&\text{($\log 1/\rho \geq 1 - \rho$)}
    \nonumber
\]
it is sufficient to ensure
\[
    T &\geq 
    \frac{4 b}{\mu^2} \frac{1}{\epsilon} 
    + 
    \frac{4 \sqrt{b}}{\mu}
    \frac{1}{\sqrt{ \mu \gamma_0}}
    \lt\{
    \log\lt( \frac{\mu V_0 }{b \gamma_0} \rt) 
    +
    \mu \gamma_0
    +
    \sqrt{2}
    \rt\}
    \frac{1}{\sqrt{\epsilon}} 
    \nonumber
    \\
    &= B_{\mathrm{var}}
    \; .
    \label{eq:lyapunov_Tlowerbound_case2}
\]

The last step is to combine the results from case $t_* \geq T$ and $t_* < T$.
That is, to ensure that, for any $\epsilon > 0$, $V_T \leq \epsilon$ holds in all cases, it suffices to ensure the conditions in  \cref{eq:lyapunov_Tlowerbound_case1,eq:lyapunov_Tlowerbound_case2} hold simultaneously.
Taking $T$ to be the maximum of both $B_{\mathrm{bias}}$ and $B_{\mathrm{var}}$ is sufficient.
\end{proof}

\newpage

\subsection{General Convergence Analysis of Proximal Bures--Wasserstein Gradient Descent}

We first provide a general convergence result for SPBWGD.
In particular, instead of assuming a specific gradient estimator, we will assume only that the estimator is unbiased and satisfies a specific variance bound.

\begin{assumption}\label{assumption:wasserstein_gradient_variance}
    Denote the global optimum as $q_* \in \argmin_{q \in \mathrm{BW}(\mathbb{R}^d)} \mathcal{F}\lt(q\rt)$.
    There exist some constants $L_{\epsilon} \in (0, +\infty)$ and $\sigma^2 \in [0, +\infty)$ such that, for all $q \in \mathrm{BW}(\mathbb{R}^d)$, the stochastic estimator of the Bures--Wasserstein gradient is unbiased and satisfies the inequality
    \[
        &
        \mathbb{E}_{(X, X_*) \sim \psi^*, \epsilon \sim \varphi}\norm{ \widehat{\nabla_{\mathrm{BW}} \mathcal{E}}\lt(q; \epsilon\rt)\lt(X\rt) - \nabla_{\mathrm{BW}} \mathcal{E}\lt(q_*\rt)\lt(X_*\rt) }_2^2  
    \nonumber
        \\
        &\qquad\qquad\qquad\qquad\leq
        4 L_{\epsilon} \mathrm{D}_{\mathcal{E}}\lt(q, q_*\rt) + 2 \sigma^2
        \; ,
    \nonumber
    \]
    where $\psi^* \in \Psi\lt(q, q_*\rt)$ is the optimal coupling between $q$ and $q_*$.
\end{assumption}

Here, the term $L_{\epsilon} \mathrm{D}_{\mathcal{E}}\lt(q, q_*\rt)$ represents the ``multiplicative noise,'' while $\sigma^2$ represents the ``additive noise.''
Typically, the $L_{\epsilon}$ factor restricts the largest step size we can use, essentially playing the same role as the Lipschitz smoothness constant $L$.
In fact, $L_{\epsilon} \geq L$ always holds by Jensen's inequality.
On the other hand, $\sigma^2$ is the factor dominating the asymptotic complexity of the algorithm.

\cref{assumption:wasserstein_gradient_variance} is closely related to the ``convex expected smoothness'' condition used for the analysis of SPGD~\citep[Assumption 4.1]{gorbunov_unified_2020}. (See also Assumption 3.2 of \citealt{khaled_unified_2023} and Eq. (65) of \citealt{garrigos_handbook_2023}.)
Furthermore, the special case of $\sigma^2 = 0$ is historically referred to as the ``interpolation condition'' in the optimization literature~\citep{vaswani_fast_2019}. 
When $\sigma^2 = 0$, it is generally expected that the iteration complexity of the stochastic algorithm improves dramatically~\citep{schmidt_fast_2013,vaswani_fast_2019}.
We will see that our analysis covers this special case.

\newpage
\subsubsection{\cref{thm:wasserstein_proximal_gradient_descent_convergence}}

Given a gradient estimator satisfying \cref{assumption:wasserstein_gradient_variance} and a potential satisfying \cref{assumption:potential_convex_smooth}, we have the following iteration complexity guarantee.

\begin{proposition}\label{thm:wasserstein_proximal_gradient_descent_convergence}
    Suppose \cref{assumption:potential_convex_smooth} holds and the chosen stochastic estimator of the Bures--Wasserstein gradient satisfies \cref{assumption:wasserstein_gradient_variance}.
    Then, for any $q_0 \in \mathrm{BW}(\mathbb{R}^d)$,
    running stochastic Bures--Wasserstein proximal gradient descent with the step size schedule in \cref{eq:stepsize_schedule} with 
    \[
        \gamma_0 &= \frac{1}{4 L_{\epsilon}}
    \nonumber
        \\
        \tau &= \frac{2}{ \mu \gamma_0 }
        \nonumber
        \\
        t_* 
        &= 
        \lt\lceil
            \frac{1}{\log(1/(1 - \mu \gamma_0))}
            \log\lt(
            \frac{1}{2 \gamma_0 \sigma^2} \Delta_0^2
            \rt) 
        \rt\rceil
        \; ,
    \nonumber
    \]
    where we denote $\Delta^2 = \mu {\mathrm{W}_2\lt(q_0, q_*\rt)}^{2}$, guarantees that 
    \[
        T \geq
        \max\lt\{  B_{\mathrm{var}},  B_{\mathrm{bias}}
        \rt\}
        \quad\Rightarrow\quad
        \mu \mathbb{E}[ {\mathrm{W}_2\lt(q_T, q_*\rt)}^2 ] \leq \epsilon \; ,
    \nonumber
    \]
    where
    \[
        B_{\mathrm{var}}
        &=
    \frac{8 \sigma^2}{\mu} \frac{1}{\epsilon} 
    + 
    2 \frac{ \sqrt{2 \sigma^2} }{ \sqrt{\mu} }
    \frac{\sqrt{L_{\epsilon}}}{\sqrt{\mu}}
    \nonumber
    \\
    &\qquad\times
    \lt\{
        \log \lt( \frac{2 L_{\epsilon}}{\sigma^2 } \Delta^2 \rt)
        +
        \frac{\mu}{4 L_{\epsilon}}
        +
        \sqrt{2}
    \rt\} \frac{1}{\sqrt{\epsilon}}
    \nonumber
        \\
        B_{\mathrm{bias}}
        &=
        \frac{4 L_{\epsilon}}{\mu} \log\lt( 2 \Delta^2 \frac{1}{\epsilon} \rt) \; .
    \nonumber
\]
\end{proposition}

This result is general in the sense that, to analyze the behavior of any gradient estimator, one only needs to establish \cref{assumption:wasserstein_gradient_variance} and substitute the corresponding constants into \cref{thm:wasserstein_proximal_gradient_descent_convergence}.
Indeed, we even retrieve ``linear convergence'' ($\log 1/\epsilon$ complexity) under the interpolation condition ($\sigma^2 = 0$).

\begin{corollary}[Linear Convergence under Interpolation]
    Suppose the assumptions of \cref{thm:wasserstein_proximal_gradient_descent_convergence} hold with $\sigma^2 = 0$.
    Then the same result holds with
    \[
        T \geq \frac{4 L_{\epsilon}}{\mu} \log\lt( 2 \mu {\mathrm{W}_2\lt(q_0, q_*\rt)}^{2} \frac{1}{\epsilon} \rt) \; .
    \nonumber
    \]
\end{corollary}

\newpage
\subsubsection{Proof of \cref{thm:wasserstein_proximal_gradient_descent_convergence}}
The proof will establish a Wasserstein contraction.
That is, we will establish an inequality of the form of 
\[
    \mathbb{E} {\mathrm{W}_2(q_{t+1}, q_*)}^2
    \leq
    \rho_t \mathbb{E} {\mathrm{W}_2(q_{t}, q_*)}^2 + \gamma_t^2 b
    \nonumber
\]
for some corresponding constants $\rho_t \in (0, 1)$ and $b_t < +\infty$.
Informally, this amounts to obtaining a contraction between the two sequences generated as, for each $t \geq 0$,
\[
    q_{t+1} &= 
    \operatorname{JKO}_{\gamma_t \mathcal{H}}\big(
    {\big( \mathrm{Id} - \gamma_t \widehat{\nabla_{\mathrm{BW}} \mathcal{E}}\lt(q_t; \epsilon_t\rt) \big)}_{\# q_t}
    \big)
    \nonumber
    \\
    q_{*} &= 
    \operatorname{JKO}_{\gamma_t \mathcal{H}}\big(
    {\big( \mathrm{Id} - \gamma_t \nabla_{\mathrm{BW}} \mathcal{E}\lt(q_*\rt) \big)}_{\# q_*}
    \big) \; ,
    \label{eq:wasserstein_stationary_sequence}
\]
which is reminiscent of the synchronous coupling approach of analyzing sampling algorithms~\citep{durmus_highdimensional_2019,dalalyan_further_2017}. (See also \S 4.1 of \citealt{chewi_logconcave_2024}.)
More importantly, this differs from the strategy of \citet{diao_forwardbackward_2023}, who analyzed changes in the functionals $\mathcal{E}$ and $\mathcal{H}$ rather than the distance between iterates.
We argue that our approach is more natural for strongly convex potentials and better aligned with contemporary analysis strategies of PSGD~\citep[\S 12.2]{garrigos_handbook_2023}.

To proceed with the Wasserstein contraction strategy, we need to establish that the sequence simulated by \cref{eq:wasserstein_stationary_sequence} is stationary and therefore represents the minimizer of $\mathcal{F}$.

\begin{lemma}[restate={[name=Restated]thmwassersteinproximalgradientdescentstationarity}]\label{thm:wasserstein_proximal_gradient_descent_stationarity}
    Suppose $q_* \in \argmin_{q \in \mathrm{BW}\lt(\mathbb{R}^d\rt)} \lt\{\mathcal{F} = \mathcal{E} + \mathcal{H} \rt\}$.
    Then $q_*$ is a fixed point of the composition of the Bures--Wasserstein gradient descent step and the JKO operator such that
    \[
        q_* = \operatorname{JKO}_{\gamma_t \mathcal{H}}( {\lt(\mathrm{Id} - \gamma_t \nabla_{\mathrm{BW}} \mathcal{E}\lt(q_*\rt)\rt)}_{\# q_*} ) \; .
    \nonumber
    \]
\end{lemma}
\begin{proof}
The proof is deferred to \cref{section:proof_wasserstein_proximal_gradient_descent_stationarity}.
\end{proof}

\newpage
Secondly, we need to ensure that the proximal step does not push the iterates farther apart.
In the analysis of the canonical Euclidean proximal operator, this is represented by the non-expansiveness (1-Lipschitzness) of proximal operators.
For the JKO operator, however, directly establishing and relying on non-expansiveness as follows appears to be new.

\begin{lemma}[restate={[name=Restated]thmnonexpansivenessjko}]\label{thm:wasserstein_proximal_operator_non_expansiveness}
    Suppose the functional $\mathcal{G} : \mathcal{P}_2(\mathbb{R}^d) \to (-\infty, +\infty]$ satisfies the following:
    \begin{enumerate}[label=(\alph*)]
        \item $\mathcal{G}$ admits a Bures--Wasserstein gradient for all $p \in \mathrm{BW}(\mathbb{R}^d)$, 
        \item The output of $\operatorname{JKO}_{\mathcal{G}}(p)$ is unique for all $p \in \mathrm{BW}(\mathbb{R}^d)$, and
        \item $\mathcal{G}$ is convex along generalized geodesics such that, for any $p, q \in \mathrm{BW}\lt(\mathbb{R}^d\rt)$ and $\nu \in \mathrm{BW}\lt(\mathbb{R}^d\rt)$,
    \[
        &\mathcal{G}(p)-\mathcal{G}(q) \geq 
        \nonumber
        \\
        &\qquad
        \mathbb{E}_{\nu} \inner{ \nabla_{\mathrm{BW}} \mathcal{G}(q) \circ M^*_{\nu \mapsto q}, M^*_{\nu \mapsto p} - M^*_{\nu \mapsto q} }
        \; ,
        \nonumber
    \]
    where $M^*_{\nu \mapsto p}$ and $M^*_{\nu \mapsto q}$ are the optimal transport maps from $\nu$ to $p$ and $q$, respectively.
    \end{enumerate}
    
    Then, for any $p, q \in \mathrm{BW}\lt(\mathbb{R}^d\rt)$, the Bures--Wasserstein JKO operator associated with $\mathcal{G}$ satisfies
    \[
        \mathrm{W}_2\lt(
            \operatorname{JKO}_{\mathcal{G}}\lt(p\rt), \operatorname{JKO}_{\mathcal{G}}\lt(q\rt)
        \rt)
        \quad\leq\quad
        \mathrm{W}_2\lt(p, q\rt) \; .
        \nonumber
    \]
\end{lemma}
\begin{proof}
The proof is deferred to \cref{section:proof_nonexpansiveness_jko}.
\end{proof}

One might think that \cref{thm:wasserstein_proximal_operator_non_expansiveness} imposes assumptions that are stronger than typically expected for analyzing a proximal operator.
In particular, we assumed the existence of the Bures--Wasserstein gradient and therefore ruled out non-Bures--Wasserstein-differentiable functionals.
This is because, in this work, we only consider $\mathcal{G} = \gamma_t \mathcal{H}$, which is Bures--Wasserstein-differentiable.
To rule out technicalities associated with non-differentiability beyond the scope of this paper, we opted for slightly stronger assumptions.
However, we conjecture that \cref{thm:wasserstein_proximal_operator_non_expansiveness} should be generalizable to functionals with only Fr\'echet subdifferentials.
(Refer to the work of~\citealt{salim_wasserstein_2020} for further discussions on the JKO operator.)

For our use case of $\mathcal{G} = \gamma_t \mathcal{H}$, (a) was established by \citet{lambert_variational_2022}, (b) is immediate by inspecting the closed form solution of $\mathrm{JKO}_{\gamma_t \mathcal{H}}$ established by~\citet{wibisono_sampling_2018}, and (c) was shown by~\citet[Lemma 3.2]{diao_forwardbackward_2023}.

\newpage
Since we now have \cref{thm:wasserstein_proximal_operator_non_expansiveness,thm:wasserstein_proximal_gradient_descent_stationarity}, we can show that the Bures--Wasserstein gradient descent step always makes the iterates go closer to each other up to some noise-induced perturbation.

\begin{lemma}\label{thm:wasserstein_proximal_gradient_descent_one_step_inequality}
Suppose \cref{assumption:potential_convex_smooth} holds and the chosen stochastic gradient estimator of the Bures--Wasserstein gradient satisfies \cref{assumption:wasserstein_gradient_variance}.
Then, for any $t \geq 0$ and any $\gamma_t \leq 1/(2 L_{\epsilon})$, the iterates generated by SPBWGD satisfy
\[
    \mathbb{E}[{\mathrm{W}_2\lt(q_{t+1}, q_*\rt)}^2 ]
    \leq
    \lt(1 - \mu \gamma_t\rt) \mathbb{E}[ {\mathrm{W}_2\lt(q_{t}, q_*\rt)}^2 ]
    + 2 \gamma_t^2 \sigma^2 \; .
    \nonumber
\]
\end{lemma}

This follows from the fact that the Bures--Wasserstein gradient is coercive under \cref{assumption:potential_convex_smooth}, which is established in \cref{thm:wasserstein_coercivity}.

\begin{proof}
Denote
\[
    q_{t+\nicefrac{1}{2}}^*
    =
    { (\mathrm{Id} - \gamma_t \nabla_{\mathrm{BW}} \mathcal{E}\lt(q_*\rt) ) }_{\# q_*}  \; .
    \nonumber
\]
Given \cref{thm:wasserstein_proximal_operator_non_expansiveness,thm:wasserstein_proximal_gradient_descent_stationarity}, we have that 
\[
    &
    {\mathrm{W}_2\lt(q_{t+1}, q_*\rt)}
\nonumber
    \\
    &=
    {\mathrm{W}_2\big( 
    \operatorname{JKO}_{\gamma_t \mathcal{H}}( q_{t+\nicefrac{1}{2}} ) ,
    \operatorname{JKO}_{\gamma_t \mathcal{H}}( q_{t+\nicefrac{1}{2}}^* ) \big)}
\nonumber
    \\
    &\leq
    {\mathrm{W}_2\big(  q_{t+\nicefrac{1}{2}} , q_{t+\nicefrac{1}{2}}^* \big)}
\nonumber
    \\
    &=
    \mathrm{W}_2\big( 
    {( \mathrm{Id} - \gamma_t \widehat{\nabla \mathcal{E}} \lt(q_t; \epsilon_t\rt)) }_{\# q_t} ,
    { (\mathrm{Id} - \gamma_t \nabla \mathcal{E}\lt(q_*\rt) ) }_{\# q_*}
    \big)  \; .
\nonumber
\]
Denote the optimal coupling $\psi^*_t \in \Psi(q_t, q_*)$ between $q_{t}$ and $q_*$.
Then, for the random variables $(X_t, X_*) \sim \psi_t^*$ and by the definition of the Wasserstein distance, 
\[
    &
    {\mathrm{W}_2\big( 
    {( \mathrm{Id} - \gamma_t \widehat{\nabla_{\mathrm{BW}} \mathcal{E}} \lt(q_t; \epsilon_t\rt)) }_{\# q_t} ,
    { (\mathrm{Id} - \gamma_t \nabla_{\mathrm{BW}} \mathcal{E}\lt(q_*\rt) ) }_{\# q_*}
    \big)}^2
\nonumber
    \\
    &=
    \mathbb{E}{\lVert}
        X_{t} - \gamma_t \widehat{\nabla_{\mathrm{BW}} \mathcal{E}}\lt(q_t; \epsilon_t\rt)\lt(X_t\rt) 
        \nonumber
        \\
        &\qquad\qquad\qquad
        - X_* + \gamma_t \nabla_{\mathrm{BW}} \mathcal{E}\lt(q_*\rt)\lt(X_*\rt) 
    {\rVert}_2^2 \; .
\nonumber
\]
Expanding the square, 
\[
    &
    {\mathrm{W}_2\lt(q_{t+1}, q_*\rt)}^2
\nonumber
    \\
    &\leq
    \mathbb{E}\Big[
    \norm{X_t - X_*}_2^2
\nonumber
    \\
    &\;
    - 2 \gamma_t \inner{ \widehat{\nabla_{\mathrm{BW}} \mathcal{E}}\lt(q_t\rt)\lt(X_t; \epsilon_t\rt) - \nabla_{\mathrm{BW}} \mathcal{E}\lt(q_*\rt)\lt(X_*\rt) , X_t - X_* }
\nonumber
    \\
    &
    \;
    + \gamma_t^2 \norm{\widehat{\nabla_{\mathrm{BW}} \mathcal{E}}\lt(q_t\rt)\lt(X_t; \epsilon_t\rt) - \nabla_{\mathrm{BW}} \mathcal{E}\lt(q_*\rt)\lt(X_*\rt)}_2^2
    \Big]
\nonumber
    \\
    &=
    {\mathrm{W}_2\lt(q_t, q_*\rt)}^2
\nonumber
    \\
    &\;
    - 2 \gamma_t \mathbb{E} \inner{ \widehat{\nabla_{\mathrm{BW}} \mathcal{E}}\lt(q_t\rt)\lt(X_t; \epsilon_t\rt) - \nabla_{\mathrm{BW}} \mathcal{E}\lt(q_*\rt)\lt(X_*\rt) , X_t - X_* }
\nonumber
    \\
    &
    \;
    + \gamma_t^2 \mathbb{E} \norm{\widehat{\nabla_{\mathrm{BW}} \mathcal{E}}\lt(q_t\rt)\lt(X_t; \epsilon_t\rt) - \nabla_{\mathrm{BW}} \mathcal{E}\lt(q_*\rt)\lt(X_*\rt)}_2^2
\nonumber
    \; .
\] 
Taking expectation conditional on the filtration $\mathscr{F}_t$ of the iterates generated up to iteration $t$, the unbiasedness of the gradient estimator yields
\[
    &
    \mathbb{E}\lt[ {\mathrm{W}_2\lt(q_{t+1}, q_*\rt)} \mid \mathscr{F}_t \rt]
\nonumber
    \\
    &\;\leq
    {\mathrm{W}_2\lt(q_t, q_*\rt)}^2
\nonumber
    \\
    &\quad
    - 
    2 \gamma_t \mathbb{E} \inner*{ \nabla_{\mathrm{BW}} \mathcal{E}\lt(q_t\rt)\lt(X_t\rt) - \nabla_{\mathrm{BW}} \mathcal{E}\lt(q_*\rt)\lt(X_*\rt) , X_t - X_* }
\nonumber
    \\
    &\quad
    + \gamma_t^2 \, \mathbb{E}\lt[ \norm{\widehat{\nabla_{\mathrm{BW}} \mathcal{E}}\lt(q_t\rt)\lt(X_t; \epsilon_t\rt) - \nabla_{\mathrm{BW}} \mathcal{E}\lt(q_*\rt)\lt(X_*\rt)}_2^2 \rt] \; .
\nonumber
\shortintertext{Under \cref{assumption:wasserstein_gradient_variance}, the gradient variance is bounded as}
    &\;\leq
    {\mathrm{W}_2\lt(q_t, q_*\rt)}^2
\nonumber
    \\
    &\quad
    - 2 \gamma_t \mathbb{E}\inner{ \nabla_{\mathrm{BW}} \mathcal{E}\lt(q_t\rt)\lt(X_t\rt) - \nabla_{\mathrm{BW}} \mathcal{E}\lt(q_*\rt)\lt(X_*\rt) , X_t - X_* }
\nonumber
    \\
    &
    \quad+
    \gamma_t^2 \, \lt(4 L_{\epsilon} \mathrm{D}_{\mathcal{E}}\lt(q, q_*\rt) + 2 \sigma^2 \rt) \; ,
\nonumber
\shortintertext{while the Bures--Wasserstein gradient is coercive in the sense of  \cref{thm:wasserstein_coercivity} such that}
    &\;\leq
    {\mathrm{W}_2\lt(q_t, q_*\rt)}^2
    - 2 \gamma_t \lt(  \frac{\mu}{2} {\mathrm{W}_2\lt(q_t, q_*\rt)}^2 + \mathrm{D}_{\mathcal{E}}\lt(q_t, q_*\rt) \rt)
\nonumber
    \\
    &
    \qquad+
    \gamma_t^2 \, \lt(4 L_{\epsilon} \mathrm{D}_{\mathcal{E}}\lt(q, q_*\rt) + 2 \sigma^2 \rt) \; .
\nonumber
    \\
    &\;=
    \lt(1 - \mu \gamma_t\rt) {\mathrm{W}_2\lt(q_t, q_*\rt)}^2
\nonumber
    \\
    &\qquad
    - 2 \gamma_t \lt(1 - 2 L_{\epsilon} \gamma_t \rt) \mathrm{D}_{\mathcal{E}}\lt(q, q_*\rt)  
    + 2 \gamma_t^2 \sigma^2
\nonumber
    \\
    &\;\leq
    \lt(1 - \mu \gamma_t\rt) {\mathrm{W}_2\lt(q_t, q_*\rt)}^2
    + 2 \gamma_t^2 \sigma^2 \; .
\nonumber
\] 
The last step follows from the non-negativity of $\mathrm{D}_{\mathcal{E}}$ and the step size limit $\gamma_t \leq 1/(2 L_{\epsilon})$.
Taking full expectation yields the statement.
\end{proof}

Notice that \cref{thm:wasserstein_proximal_gradient_descent_one_step_inequality} implies \cref{eq:lyapunov_onestep_condition} under a specific choice of Lyapunov function.
Therefore, \cref{thm:wasserstein_proximal_gradient_descent_convergence} follows by invoking \cref{thm:lyapunov_convergence_lemma}.

\begin{proof}[Proof of \cref{thm:wasserstein_proximal_gradient_descent_convergence}]
\cref{thm:wasserstein_proximal_gradient_descent_one_step_inequality} implies the Lyapunov decrease condition in \cref{eq:lyapunov_onestep_condition} holds for all $t \geq 0$ with $\gamma_{\mathrm{max}} = \gamma_0 = 1/(2 L_{\epsilon})$, $V_t = \mathbb{E}[{\mathrm{W}_2\lt(q_{t}, q_*\rt)}^2 ]$, and $b = 2 \sigma^2$.
Therefore, we can invoke \cref{thm:lyapunov_convergence_lemma} with $\gamma_0 = 1/(2 L_{\epsilon})$, $\tau = 2/(\gamma_0 \mu)$, and
\[
    t_* 
    = 
    \lt\lceil
        \frac{1}{\log(1/(1 - \mu \gamma_0))}
        \log\lt(
        \frac{\mu}{2 \gamma_0 \sigma^2} V_0
        \rt) 
    \rt\rceil
    \; .
\nonumber
\]
Then we have
\[
    T &\geq \max\{ B_{\mathrm{var}}, B_{\mathrm{bias}} \}
    \quad\Rightarrow\quad
    \mathbb{E}[ {\mathrm{W}_2\lt(q_T, q_*\rt)}^2 ] \leq \epsilon
\nonumber
\]
with the constants
\[
    B_{\mathrm{var}}
    &=
    \frac{8 \sigma^2}{\mu} \frac{1}{\mu \epsilon} 
    + 
    2 \frac{ \sqrt{2 \sigma^2} }{ \sqrt{\mu} }
    \frac{\sqrt{L_{\epsilon}}}{ \sqrt{\mu} }
    \nonumber
    \\
    &\quad\times
    \lt\{
        \log \lt( \frac{2 L_{\epsilon}}{\sigma^2 } \mu V_0 \rt)
        +
        \frac{\mu}{4 L_{\epsilon}}
        +
        \sqrt{2}
    \rt\} \frac{1}{\sqrt{\mu \epsilon}}
\nonumber
    \\
    B_{\mathrm{bias}}
    &=
    4 \frac{L_{\epsilon}}{\mu} \log\lt( 2 \mu V_0  \frac{1}{\mu \epsilon} \rt)
\nonumber
    \; .
\]
where
Adjusting for the dimensionless condition $\mu \mathbb{E}[{\mathrm{W}_2\lt(q_T, q_*\rt)}^2 ] \leq \epsilon$ yields the statement.
\end{proof}

\newpage

\subsection{General Convergence Analysis of Proximal Gradient Descent}

\subsubsection{\cref{thm:parameterspace_proximal_gradient_descent_convergence}}

We present the Euclidean-space analog of \cref{thm:wasserstein_proximal_gradient_descent_convergence}.
The result is essentially a typical non-asymptotic convergence analysis of vanilla PSGD~\citep{nemirovski_robust_2009}.
For the gradient estimator, we assume the following general condition:

\begin{assumption}\label{assumption:parameterspace_gradient_variance}
    Denote the global optimum as $\lambda_* \in \argmin_{\lambda \in \Lambda} \mathcal{F}\lt(q_{\lambda}\rt)$.
    There exists some constants $L_{\epsilon}, \sigma^2 \in (0, +\infty)$ such that, for all $\lambda \in \Lambda$, the stochastic estimator of the gradient of the energy $\widehat{\nabla_{\lambda} \mathcal{E}}$ satisfies the inequality
    \[
        &
        \mathbb{E}_{\epsilon \sim \varphi}\norm{ \widehat{\nabla_{\lambda} \mathcal{E}}\lt(q_{\lambda}; \epsilon\rt) - \nabla_{\lambda_*} \mathcal{E}\lt(q_{\lambda_*}\rt) }_2^2  
\nonumber
        \\
        &\qquad\qquad\qquad\leq
        4 L_{\mathrm{\epsilon}} \mathrm{D}_{\lambda \mapsto \mathcal{E}(q_{\lambda})}\lt(\lambda, \lambda_*\rt) + 2 \sigma^2
        \; .
\nonumber
    \]
\end{assumption}

This assumption is a special case of Assumption 4.1 of \citet{gorbunov_unified_2020}.
(\citealt{khaled_unified_2023} analyze SPGD on non-strongly-convex objectives under the same assumption.)
Furthermore, it is a basic consequence of the ``expected smoothness'' assumption~\citep{gower_stochastic_2021}.
Under \cref{assumption:parameterspace_gradient_variance} and a fixed step size schedule (for all $t \geq 0$, $\gamma_t = \gamma$ for some $\gamma > 0$), it is known that (\citealp[Theorem 12.10]{garrigos_handbook_2023}; \citealp[Corollary A.1]{gorbunov_unified_2020}) the iterates of SPGD satisfy the bound
\[
    \mathbb{E}\norm{\lambda_T - \lambda_*}_2^2 
    \leq
    \lt(1 - \mu \gamma\rt) \norm{\lambda_0 - \lambda_*}_2^2 + \frac{2 \sigma^2}{\mu} \gamma \; .
\nonumber
\]
where $\lambda_* = \argmin_{\lambda \in \Lambda} \mathcal{F}(q_{\lambda})$ is the global minimizer.
This implies an iteration complexity of $\mathrm{O}(\epsilon^{-1} \log (\Delta \epsilon^{-1}) )$ for ensuring $\forall \epsilon > 0 , \, \mathbb{E}\norm{\lambda_T - \lambda_*}_2^2 \leq \epsilon $~\citep[Corollary 12.10]{garrigos_handbook_2023}.

Later on, \citet{domke_provable_2023,kim_convergence_2023} removed the factor of $\log \epsilon^{-1}$ by employing the two-step step size schedule in~\cref{eq:stepsize_schedule} by \citet{gower_sgd_2019}, originally developed for vanilla SGD.
Their choice of switching time $t_*$, however, results in a worse dependence on $\Delta = \norm{\lambda_0 - \lambda_*}_2$ in the iteration complexity $\mathrm{O}(\epsilon^{-1} + \Delta^2 \epsilon^{-1/2})$.
By employing the switching time stated in \cref{section:lyapunov}, we can simultaenously ensure a $\mathrm{O}(\epsilon^{-1})$ dependence on $\epsilon$ and a $\mathrm{O}(\log \Delta)$ dependence on $\Delta$ in an any-time fashion.

\newpage
\begin{proposition}\label{thm:parameterspace_proximal_gradient_descent_convergence}
    Suppose \cref{assumption:potential_convex_smooth} holds, the variational family is parametrized as in \cref{assumption:parametrization}, and the chosen stochastic estimator of the parameter gradient satisfies \cref{assumption:parameterspace_gradient_variance}.
    Then, for any $\lambda_0 \in \Lambda$, running stochastic proximal gradient descent with the step size schedule in \cref{eq:stepsize_schedule} with 
    \[
        \gamma_0 &= \frac{1}{4 L_{\epsilon}} \nonumber
        \\
        \tau &= \frac{2}{\mu \gamma_0} \nonumber
        \\
        t_* 
        &= 
        \lt\lceil
            \frac{1}{\log(1/(1 - \mu \gamma_0))}
            \log\lt(
            \frac{1}{2 \gamma_0 \sigma^2} \Delta^2
            \rt) 
        \rt\rceil
        \; ,
    \nonumber
    \]
    where we denote $\Delta^2 = \mu \norm{\lambda_0 - \lambda_*}^{2}$, guarantees that
    \[
        T \geq
        \max\lt\{  B_{\mathrm{var}} , B_{\mathrm{bias}} \rt\}
        \quad\Rightarrow\quad
        \mu \mathbb{E}[ {\mathrm{W}_2\lt(q_{\lambda_T}, q_{\lambda_*}\rt)}^2 ] \leq \epsilon ,
\nonumber
    \]
    where
    \[
        B_{\mathrm{var}}
        &=
    \frac{8 \sigma^2}{\mu} \frac{1}{\epsilon} 
    + 
    2 \frac{ \sqrt{2 \sigma^2} }{ \sqrt{\mu} }
    \frac{\sqrt{L_{\epsilon}}}{\sqrt{\mu}}
    \nonumber
    \\
    &\qquad\times
    \lt\{
        \log \lt( \frac{2 L_{\epsilon}}{\sigma^2 } \Delta^2 \rt)
        +
        \frac{\mu}{4 L_{\epsilon}}
        +
        \sqrt{2}
    \rt\} \frac{1}{\sqrt{\epsilon}}
\nonumber
        \\
        B_{\mathrm{bias}}
        &=
        \frac{4 L_{\epsilon}}{\mu} \log\lt( 2 \Delta^2 \frac{1}{\epsilon} \rt) \; .
\nonumber
    \]
\end{proposition}

\newpage
\subsubsection{Proof of \cref{thm:parameterspace_proximal_gradient_descent_convergence}}

On a high level, the proof establishes a contraction of the Euclidean distance in parameter space $\mathbb{E}\norm{\lambda - \lambda_*}_2^2$.
This follows from the non-expansiveness of the proximal operator and the fact that the gradient descent step on the energy results in a contraction due to coercivity.
The properties of the proximal operator are summarized as follows:

\begin{lemma}\label{thm:proximal_stationary_nonexpansive}
    Denote $\lambda_* \in \argmin_{\lambda \in \Lambda} \mathcal{F}\lt(q_{\lambda}\rt)$, where $q_{\lambda}$ is parametrized as in \cref{assumption:parametrization}.
    Then the proximal operator of the Boltzmann entropy $ \lambda \mapsto \mathcal{H}\lt(q_{\lambda}\rt)$ is non-expansive such that, for any $\gamma > 0$, and any $\lambda, \lambda' \in \Lambda$,
    \[
        \norm{ 
        \operatorname{prox}_{\lambda \mapsto \gamma \mathcal{H}(q_{\lambda})} 
        - \operatorname{prox}_{\lambda \mapsto \gamma \mathcal{H}(q_{\lambda'})} 
        }_2^2
        \leq
        \norm{  \lambda - \lambda'}_2^2 \; ,
\nonumber
    \]
    while $\lambda_*$ is a fixed point of the composition with a gradient descent step on the energy $\mathcal{E}$ such that
    \[
        \lambda_* = \operatorname{prox}_{\lambda \mapsto \gamma \mathcal{H}(q_{\lambda})}\lt( \lambda_* - \gamma \nabla_{\lambda_*} \mathcal{E}(q_{\lambda_*}) \rt) \; .
\nonumber
    \]
\end{lemma}
\begin{proof}
    Under \cref{assumption:parametrization} $\lambda \mapsto \mathcal{H}\lt(q_{\lambda}\rt)$ is closed and convex on $\Lambda$~\citep[Lemma 19]{domke_provable_2023}.
    Therefore, the proximal operator of $\lambda \mapsto \gamma_t \mathcal{H}(q_{\lambda})$ is well-defined in the sense that its variational problem $\minimize_{\lambda \in \Lambda} \mathcal{H}(q_{\lambda}) + (1/\gamma_t) \norm{\lambda - \lambda'}_2^2$ has a unique solution due to strong convexity.
    Both non-expansiveness and the fact that $\lambda_*$ is a fixed point of the composition with a gradient descent step on $\mathcal{E}$ are proven by~\citet{garrigos_handbook_2023} in Lemma 8.17 and 8.18, respectively.
\end{proof}

Using this result, we obtain the one-step contraction.
\begin{lemma}\label{thm:parameter_proximal_gradient_descent_onestep}
Suppose \cref{assumption:potential_convex_smooth} holds, the variational family is parametrized as \cref{assumption:parametrization}, and the chosen stochastic estimator of the parameter gradient satisfies \cref{assumption:parameterspace_gradient_variance}.
Then, for each $t \geq 0$ and any $\gamma_t \leq 1/(2 L_{\epsilon})$, 
\[
    \mathbb{E}\norm{\lambda_{t+1} - \lambda_*}_2^2
    \leq
    \lt( 1 - \mu \gamma_t \rt) \mathbb{E}\norm{\lambda_{t} - \lambda_*}_2^2
    + 2 \gamma_t^2 \sigma^2 \; .
    \nonumber
\]
\end{lemma}
\begin{proof}
First, denote
\[
    \lambda_{t+ \nicefrac{1}{2}}^* \triangleq \lambda_{*} - \gamma_t \nabla_{\lambda_*} \mathcal{E}\lt(q_{\lambda_*}\rt) \; .
\nonumber
\]
From \cref{thm:proximal_stationary_nonexpansive},
\[
    &
    \norm{\lambda_{t+1} - \lambda_*}_2^2
\nonumber
    \\
    &=
    \norm{ \mathrm{prox}_{\lambda \mapsto \gamma_t \mathcal{H}\lt(q_{\lambda}\rt)}\lt(\lambda_{t+\nicefrac{1}{2}}\rt) - \mathrm{prox}_{\lambda \mapsto \gamma_t \mathcal{H}\lt(q_{\lambda}\rt)}(\lambda_{t+\nicefrac{1}{2}}^*)  }_2^2
\nonumber
    \\
    &\leq
    \norm{ \lambda_{t + \nicefrac{1}{2}} - \lambda_{t + \nicefrac{1}{2}}^*  }_2^2
\nonumber
    \\
    &=
    \norm{ \lambda_{t} - \gamma_t \widehat{\nabla_{\lambda_t} \mathcal{E}}\lt(q_{\lambda_t}; \epsilon_t\rt) - \lambda_{*} - \gamma_t \nabla_{\lambda_*} \mathcal{E}\lt(q_{\lambda_*}\rt)  }_2^2
    \; .
\nonumber
\]
Expanding the square and taking expectation conditional on the filtration $\mathscr{F}_t$ generated by the iterates generated up to iteration $t$,
\[
    &
    \mathbb{E}\lt[ \norm{\lambda_{t+1} - \lambda_*}_2^2 \mid \mathscr{F}_t \rt]
\nonumber
    \\
    &\leq
    \norm{\lambda_{t} - \lambda_*}_2^2
\nonumber
    \\
    &\qquad
    - 
    2 \gamma_t \inner*{ \mathbb{E}\lt[ \widehat{\nabla_{\lambda_t}  \mathcal{E}}\lt(q_{\lambda_t}; \epsilon_t\rt) \mid \mathscr{F}_t \rt] - \mathcal{E}\lt(q_{\lambda_*}\rt), \lambda_t - \lambda_* }
\nonumber
    \\
    &\qquad
    +
    \gamma_t^2
    \mathbb{E}\lt[
    \norm{ \widehat{\nabla_{\lambda_t} \mathcal{E}}\lt(q_{\lambda_t}; \epsilon_t\rt) - \nabla_{\lambda_*} \mathcal{E}\lt(q_{\lambda_*}\rt)  }_2^2
    \mid \mathscr{F}_t \rt] \; .
\nonumber
\shortintertext{Since the gradient estimator is unbiased,}
    &=
    \norm{\lambda_{t} - \lambda_*}_2^2
    - 
    2 \gamma_t \inner*{ \nabla_{\lambda_t} \mathcal{E}\lt(q_{\lambda_t}\rt)  - \mathcal{E}\lt(q_{\lambda_*}\rt), \lambda_t - \lambda_* }
\nonumber
    \\
    &\qquad
    +
    \gamma_t^2
    \mathbb{E}\lt[
    \norm{ \widehat{\nabla_{\lambda_t} \mathcal{E}}\lt(q_{\lambda_t}\rt) - \nabla_{\lambda_*} \mathcal{E}\lt(q_{\lambda_*}\rt)  }_2^2
    \mid \mathscr{F}_t \rt]  \; .
\nonumber
\shortintertext{Under \cref{assumption:parameterspace_gradient_variance}, the gradient variance can be bounded as}
    &\leq
    \norm{\lambda_{t} - \lambda_*}_2^2
    - 
    2 \gamma_t \inner*{ \nabla_{\lambda_t} \mathcal{E}\lt(q_{\lambda_t}\rt)  - \mathcal{E}\lt(q_{\lambda_*}\rt), \lambda_t - \lambda_* }
\nonumber
    \\
    &\qquad
    +
    \gamma_t^2 \lt( 4 L_{\epsilon} \mathrm{D}_{\lambda \mapsto \mathcal{E}\lt(q_{\lambda}\rt)}\lt(\lambda_t, \lambda_*\rt) + 2 \sigma^2 \rt)
\nonumber
\shortintertext{and from the fact that the gradient is coercive due to \cref{thm:parameter_coercivity},}
    &\leq
    \norm{\lambda_{t} - \lambda_*}_2^2
    - 
    2 \gamma_t \lt( \frac{\mu}{2} \norm{\lambda_t - \lambda_*}_2^2 + \mathrm{D}_{\lambda \mathcal{E}(q_{\lambda})}(\lambda_t, \lambda_*)  \rt)
\nonumber
    \\
    &\qquad
    +
    \gamma_t^2 \lt( 4 L_{\epsilon} \mathrm{D}_{\lambda \mapsto \mathcal{E}\lt(q_{\lambda}\rt)}\lt(\lambda_t, \lambda_*\rt) + 2 \sigma^2 \rt)
\nonumber
    \\
    &=
    \lt(1 - \mu \gamma_t \rt) \norm{\lambda_{t} - \lambda_*}_2^2
\nonumber
    \\
    &\qquad
    - 
    2 \gamma_t \lt(1 - 2 L_{\epsilon} \gamma_t \rt) \mathrm{D}_{\lambda \mapsto \mathcal{E}(q_{\lambda})}(\lambda_t, \lambda_*)
    +
    2 \gamma_t \sigma^2
\nonumber
    \\
    &\leq
    \lt(1 - \mu \gamma_t \rt) \norm{\lambda_{t} - \lambda_*}_2^2
    +
    2 \gamma_t \sigma^2 
\nonumber
    \; .
\]
The last step follows from the non-negativity of the Bregman divergence and the step size limit $\gamma_t \leq 1/(2 L_{\epsilon})$.
Taking full expectation yields the result.
\end{proof}

\cref{thm:parameter_proximal_gradient_descent_onestep} satisfies \cref{eq:lyapunov_onestep_condition} for a specific Lyapunov function.
Therefore, \cref{thm:parameterspace_proximal_gradient_descent_convergence} follows by invoking \cref{thm:lyapunov_convergence_lemma}.

\begin{proof}[Proof of \cref{thm:parameterspace_proximal_gradient_descent_convergence}]
\cref{thm:parameter_proximal_gradient_descent_onestep} implies the Lyapunov condition \cref{eq:lyapunov_onestep_condition} for all $t \geq 0$ with $\gamma_{\mathrm{max}} = 1/(4 L_{\epsilon})$, $V_t = \mathbb{E}\norm{\lambda_{t} - \lambda_*}_2^2$, and $b = 2 \sigma^2$.
Invoking \cref{thm:lyapunov_convergence_lemma} with  $\gamma_0 = 1/(4 L_{\epsilon})$, $\tau = 2/(\mu \gamma_0)$, and
\[
    t_* 
    = 
    \lt\lceil
        \frac{1}{\log(1/(1 - \mu \gamma_0))}
        \log\lt(
        \frac{\mu}{2 \gamma_0 \sigma^2} V_0
        \rt) 
    \rt\rceil
    \; .
\nonumber
\]
we have 
\[
    T \geq \max\{B_{\mathrm{var}}, B_{\mathrm{bias}}\}
    \quad\Rightarrow\quad
    \mathbb{E}\norm{\lambda_T - \lambda_*}^2 \leq \epsilon
\nonumber
\]
with the constants
\[
    B_{\mathrm{var}} 
    &= 
    \frac{8 \sigma^2}{\mu} \frac{1}{\mu \epsilon} 
    + 
    2 \frac{ \sqrt{2 \sigma^2} }{ \sqrt{\mu} }
    \frac{\sqrt{L_{\epsilon}}}{\sqrt{\mu}}
    \nonumber
    \\
    &\qquad\qquad\times
    \lt\{
        \log \lt( \frac{2 L_{\epsilon}}{\sigma^2 } \mu V_0 \rt)
        +
        \frac{\mu}{4 L_{\epsilon}}
        +
        \sqrt{2}
    \rt\} \frac{1}{\sqrt{\mu \epsilon}}
\nonumber
    \\
    B_{\mathrm{bias}}
    &=
    4 \frac{L_{\epsilon}}{\mu} \log\lt( 2 \mu V_0  \frac{1}{\mu \epsilon} \rt) \; .
\nonumber
\]

Now, denote the coupling $\psi^{\mathrm{rep}} \in \Psi\lt(q_{\lambda_T}, q_{\lambda_*}\rt)$ associated with the transport map $M^{\mathrm{rep}}_{q_{\lambda_T} \mapsto q_{\lambda_*}}$ and the squared Euclidean distance-optimal coupling $\psi^* \in \Psi\lt(q_T, q_*\rt)$.
Then, from the identity \cref{eq:parameter_distance_coupling_distance_equivalence}, we have the ordering
\[
    \norm{\lambda_T - \lambda_*}^2   
    &=
    \mathbb{E}_{(Z_T, Z_*) \sim \psi^{\mathrm{rep}}} \norm{Z_T - Z_*}_2^2
\nonumber
    \\
    &\geq
    \mathbb{E}_{(Z_T, Z_*) \sim \psi^{*}} \norm{Z_T - Z_*}^2_2
\nonumber
    \\
    &=
    {\mathrm{W}_2\lt(q_{\lambda_T}, q_{\lambda_*}\rt)}^2 \; .
\nonumber
\]
Therefore,
\[
\mathbb{E}\norm{\lambda_T - \lambda_*}^2 \leq \epsilon
\quad\Rightarrow\quad
\mathbb{E} {\mathrm{W}_2\lt(q_{\lambda_T}, q_{\lambda_*}\rt)}^2 \leq \epsilon \; .
\nonumber
\]
Adjusting for the condition $\mu \mathbb{E}[{\mathrm{W}_2\lt(q_T, q_*\rt)}^2 ] \leq \epsilon$ yields the statement.
\end{proof}

\clearpage
\section{Deferred Proofs}

\subsection{Stochastic Proximal Bures--Wasserstein Gradient Descent}

\subsubsection{Stationary Point of SPBWGD (Proof of \cref{thm:wasserstein_proximal_gradient_descent_stationarity})}\label{section:proof_wasserstein_proximal_gradient_descent_stationarity}

\thmwassersteinproximalgradientdescentstationarity*

\begin{proof}
Recall that, from \cref{thm:stationary_condition},
\[
    \mathbb{E}_{q^*}[\nabla U] = 0, \quad \mathbb{E}_{q^*}[\nabla^2U] = \Sigma_*^{-1} \; .
    \nonumber
\]
and denote 
\[
    q_{t+\nicefrac{1}{2}}^* 
    &= \mathcal{N}(m_{t + \nicefrac{1}{2}}^*, \Sigma_{t + \nicefrac{1}{2}}^*) 
    = 
    { (\mathrm{Id} - \gamma_t \nabla_{\mathrm{BW}} \mathcal{E}\lt(q_*\rt) ) }_{\# q_*}  \; ,
    \nonumber
    \\
    q_* &= \mathcal{N}(m_*, \Sigma_*) \; .
    \nonumber
\]
The Bures--Wasserstein gradient descent step satisfies
\[
    m_{t+\nicefrac{1}{2}}^* &= m_* - \gamma_t \mathbb{E}_{q^*}[\nabla U] = m_*
    \nonumber
    \\
    M_{t+\nicefrac{1}{2}}^* &= \mathrm{I}_d - \gamma_t \mathbb{E}_{q^*}[\nabla^2U] = \mathrm{I}_d - \gamma_t\ \Sigma_*^{-1}
    \nonumber
    \\
    \Sigma_{t+\nicefrac{1}{2}}^* &= M_{t+\nicefrac{1}{2}}^* \Sigma_* M_{t+\nicefrac{1}{2}}^*
    \nonumber
    \\
    &= (\mathrm{I}_d - \gamma_t \Sigma_*^{-1})\Sigma_*(\mathrm{I}_d - \gamma_t\ \Sigma_*^{-1}) 
    \nonumber
    \\
    &= (\mathrm{I}_d - \gamma_t\Sigma_*^{-1})^2\Sigma_* \; ,
    \nonumber
\]
where the last step follows from the fact that the factors in the product all share the same eigenvectors.
In addition, $q_{t+1}^* = \mathcal{N}(m_{t+1}^*, \Sigma_{t+1}^*) \triangleq \operatorname{JKO}_{\gamma_t \mathcal{H}}(q_{t+\nicefrac{1}{2}}^*)$. 
Since the proximal step preserves the mean, \textit{i.e.}, $m_{t+1}^* = m_*$, it suffices to prove that $\Sigma_{t+1}^* = \Sigma_*$. 
From the update of $\Sigma^*_{t+\nicefrac12}$, we have
\[
\Sigma_{t+\nicefrac{1}{2}}^* + 4\gamma_t \mathrm{I}_d &= (\mathrm{I}_d - \gamma_t \Sigma_*^{-1})\Sigma_*(\mathrm{I}_d - \gamma_t\ \Sigma_*^{-1}) + 4\gamma_t \mathrm{I}_d \nonumber\\
&=\Sigma_* + 2\gamma_t \mathrm{I}_d + \gamma_t^2\Sigma_*^{-1}\nonumber\\
&= \Sigma_*(\mathrm{I}_d + \gamma_t\Sigma_*^{-1})^2.
\nonumber
\]
Then we obtain
\[
\Sigma_{t+1}^* 
&= \frac{1}{2}\bigg\{\Sigma_{t+\nicefrac{1}{2}}^* + 2\gamma_t \mathrm{I}_d 
\nonumber
\\
&\qquad\qquad\qquad
+ \left[\Sigma_{t+\nicefrac{1}{2}}^*\left(\Sigma_{t+\nicefrac{1}{2}}^* + 4\gamma_t \mathrm{I}_d\right)\right]^{\nicefrac{1}{2}}\bigg\} \nonumber\\
&= \frac12 \Big\{\Sigma_* + \gamma_t^2\Sigma_*^{-1} 
\nonumber
\\
&\quad\qquad
+ \left[(\mathrm{I}_d - \gamma_t\Sigma_*^{-1})^2\Sigma_*^2(\mathrm{I}_d + \gamma_t\Sigma_*^{-1})^2\right]^{\nicefrac{1}{2}}\Big\} \nonumber\\
&= 
\frac12\big\{\Sigma_* + \gamma_t^2\Sigma_*^{-1} 
\nonumber
\\
&\qquad\qquad\qquad
+ (\mathrm{I}_d - \gamma_t\Sigma_*^{-1})\Sigma_*(\mathrm{I}_d + \gamma_t\Sigma_*^{-1})\big\} \nonumber\\
& = \frac12\left\{\Sigma_* + \gamma_t^2\Sigma_*^{-1} + \Sigma_* - \gamma_t^2\Sigma_*^{-1}\right\} \nonumber\\
& = \Sigma_*.
\nonumber
\]
Thus $q_{t+1}^* = q_*$ and we conclude that $q_*$ is a fixed point.
\end{proof}

\newpage

\subsubsection{Non-Expansiveness of the JKO Operator (Proof of \cref{thm:wasserstein_proximal_operator_non_expansiveness})}\label{section:proof_nonexpansiveness_jko}

\thmnonexpansivenessjko*

The proof utilizes the regularity properties of the JKO operator established by \citet{salim_wasserstein_2020,ambrosio_gradient_2005}.
However, to avoid technicalities related to non-differentiability, we restrict our interest to Bures--Wasserstein-differentiable functionals $\mathcal{G}$.

\begin{proof}
Denote 
\[
p &\triangleq  \mathrm{Normal}(m_p,\Sigma_p)
&
q &\triangleq \mathrm{Normal}(m_q,\Sigma_q)
\nonumber
\\
p' &\triangleq \operatorname{JKO}_{\mathcal{G}}\lt(p\rt)
&
q' &\triangleq \operatorname{JKO}_{\mathcal{G}}\lt(q\rt) \; .
\nonumber
\]
Under the assumptions on $\mathcal{G}$, the JKO updates can be expressed as
\[
X' &= X - \nabla_{\mathrm{BW}} \mathcal{G}(p)(X') \, ;  &X' &\sim p' \nonumber \\ 
Y' &= Y - \nabla_{\mathrm{BW}} \mathcal{G}(q)(Y') \, ;  &Y' &\sim q' 
\; ,
\nonumber
\]
where $(X, Y)$ are optimally coupled.
Now, consider some $Z \sim \nu  = \mathrm{Normal}(0, \mathrm{I}_d)$.
We know that the optimal transport maps $M^*_{\nu \mapsto p'}$ and $M^*_{\nu \mapsto q'}$ exist~\citep[Theorem 4.1]{villani_optimal_2009} uniquely~\citep{brenier_polar_1991}.
From the assumption on $\mathcal{G}$, we have
\[
&
\mathcal{G}(p')-\mathcal{G}(q') 
\nonumber
\\
&\qquad\geq 
\mathbb{E}_{\nu}\inner{ \nabla_{\mathrm{BW}} \mathcal{G}(q')\circ M^*_{\nu \mapsto q'}, M^*_{\nu \mapsto p'} - M^*_{\nu \mapsto q'} }
\nonumber 
\\
&\qquad= 
\mathbb{E} \inner{ \nabla_{\mathrm{BW}} \mathcal{G}(q')(Y'), X' - Y'},
\nonumber 
\\
&
\mathcal{G}(q')-\mathcal{G}(p') 
\nonumber
\\
&\qquad\geq 
\mathbb{E}_{\nu}\inner{
\nabla_{\mathrm{BW}} \mathcal{G}(p')\circ M^*_{\nu \mapsto p'}, M^*_{\nu \mapsto q'} - M^*_{\nu \mapsto p'}
}
\nonumber \\
&\qquad= \mathbb{E}\inner{ \nabla_{\mathrm{BW}} \mathcal{G}(p')(X'), Y' - X' } \; .
\nonumber
\]
Summing up the two inequalities above yields
\[
\mathbb{E} \inner{ \nabla_{\mathrm{BW}} \mathcal{G}(p')(X') - \nabla_{\mathrm{BW}} \mathcal{G}(q')(Y'), X' - Y' }  \geq 0 \; .
\label{eq:wasserstein_subgradient_monotonicity}
\]

Rearranging the JKO updates,
\[
X' - Y' + \nabla_{\mathrm{BW}} \mathcal{G}(p')(X') - \nabla_{\mathrm{BW}} \mathcal{G}(q')(Y') = X - Y \; .
\nonumber
\]
Taking the inner product with $X' - Y'$ for both sides and taking the expectation, we have
\[
&
\mathbb{E} \inner{ \nabla_{\mathrm{BW}} \mathcal{G}(p')(X') - \nabla_{\mathrm{BW}} \mathcal{G}(q')(Y'), X' - Y'} 
\nonumber
\\
&= \mathbb{E}\left[\inner{ X-Y, X'-Y' } - \norm{X' - Y'}^2_2\right] \geq 0 \; ,
\label{eq:nonexpansiveness_basic_crucial_inequality}
\]
where the inequality follows from \cref{eq:wasserstein_subgradient_monotonicity}.
Now, from Cauchy-Schwarz and Young's inequality, we know that 
\[
\mathbb{E}\|X' - Y'\|_2^2 &\leq \mathbb{E}\langle X-Y,X'-Y'\rangle
\nonumber
\\
&\leq 
\frac{1}{2}\mathbb{E}\|X - Y\|_2^2 + \frac{1}{2}\mathbb{E}\|X' - Y'\|_2^2 \; .
\label{eq:nonexpansiveness_basic_inequality}
\]
Combining \cref{eq:nonexpansiveness_basic_inequality,eq:nonexpansiveness_basic_crucial_inequality},
\[
\mathbb{E}\left[\|X' - Y'\|_2^2\right] \leq \mathbb{E}\left[\|X - Y\|_2^2\right]
\; .
\nonumber
\]
Thus
\[
{\mathrm{W}_2\lt(
\operatorname{JKO}_{\mathcal{G}}\lt(p\rt), \operatorname{JKO}_{\mathcal{G}}\lt(q\rt)
\rt) }^2
&\leq 
\mathbb{E}\left[\|X' - Y'\|_2^2\right] 
\nonumber
\\
&\leq 
\mathbb{E}\left[\|X - Y\|_2^2\right] 
\nonumber
\\
&= \mathrm{W}_2\lt(p, q\rt)^2 \; .
\nonumber
\]
\end{proof}

\newpage

\subsubsection{Iteration Complexity (Proof of \cref{thm:wasserstein_proximal_gradient_descent_pricestein})}\label{section:proof_wasserstein_proximal_gradient_descent_steinprice}

\thmwassersteinproximalgradientdescentpricestein*

This is a corollary of \cref{thm:wasserstein_proximal_gradient_descent_convergence}, where the sufficient conditions are established in \cref{thm:wasserstein_gradient_variance_bound,thm:wasserstein_bregman_equivalence}.
    
\begin{proof}
For any $q \in \mathrm{BW}(\mathbb{R}^d)$, denote $\psi^* \in \Psi(q, q_*)$, the coupling between $q$ and $q_*$ optimal in terms of squared Euclidean distance.
Under the stated conditions, we can invoke \cref{thm:wasserstein_gradient_variance_bound}, where the generic coupling $\psi$ in the statement of \cref{thm:wasserstein_gradient_variance_bound} can be set as $\psi = \psi^*$.
That is,
\[
    &
    \mathbb{E} \norm{\widehat{\nabla_{\mathrm{BW}}^{\text{bonnet--price}} \mathcal{E}}\lt(q; \epsilon\rt)\lt(X_t\rt) - \nabla_{\mathrm{BW}} \mathcal{E}\lt(q_*\rt)\lt(X_*\rt)}_2^2
\nonumber
    \\
    &\qquad\leq
    10 L \kappa \, \mathbb{E}_{(X,X_*) \sim \psi^*}\lt[ \mathrm{D}_{U}\lt(X, X_*\rt) \rt]
    +
    10 d L 
\nonumber
    \\
    &\qquad=
    10 L \kappa \,  \mathrm{D}_{\mathcal{E}}\lt(q, q_*\rt)
    +
    10 d L  \; .
\nonumber
\]
This establishes
\[
    \text{\cref{thm:wasserstein_gradient_variance_bound} \& \cref{thm:wasserstein_bregman_equivalence}} \quad\Rightarrow\quad
    \text{\cref{assumption:wasserstein_gradient_variance}}
    \nonumber
\]
with the constants $L_{\epsilon} = 5/2 L \kappa$ and $\sigma^2 = 5 d L$.
Then 
\[
    \text{\cref{assumption:potential_convex_smooth} \& \cref{assumption:wasserstein_gradient_variance}} 
    \;\Rightarrow\;
    \text{\cref{thm:wasserstein_proximal_gradient_descent_convergence}}
    \; .
\nonumber
\]
Substituting for the constants $L_{\epsilon}$ and $\sigma^2$, the parameters of the step size schedule in \cref{eq:stepsize_schedule} become
\[
    \gamma_0 &= \frac{1}{10 L \kappa}
\nonumber
    \\
    \tau &= 8 \kappa
\nonumber
    \\
    t_* 
    &= 
    \lt\lceil
    \frac{1}{\log(1/(1 - \nicefrac{1}{10 \kappa^2}))}
    \log\lt(
    \frac{\kappa}{d} \Delta^2
    \rt) 
    \rt\rceil
    \; ,
\nonumber
\]
which guarantee 
\[
    T \geq \max\{B_{\mathrm{var}}, B_{\mathrm{bias}}\}
    \quad\Rightarrow\quad
    \mu \mathbb{E}[{\mathrm{W}_2\lt(q_T, q_*\rt)}^2] \leq \epsilon
\nonumber
\]
with the constants
\[
        B_{\mathrm{var}}
        &=
        40 d \kappa \frac{1}{\epsilon} 
        + 
        10 \sqrt{d} \kappa^{3/2}
        \nonumber
        \\
        &\quad\qquad\qquad\times
        \lt\{
            \log \lt( \frac{\kappa }{d} \Delta^2 \rt)
            +
            \frac{1}{10 \kappa^2}
            +
            \sqrt{2}
        \rt\} \frac{1}{\sqrt{\epsilon}}
        \nonumber
        \\
        B_{\mathrm{bias}}
        &=
        10 \kappa^2 \log\lt( 2 \Delta^2 \frac{1}{\epsilon} \rt) 
        \; .
\nonumber
\]
\end{proof}

\clearpage

\subsubsection{Variance Bound on the Bures-Wasserstein Gradient Estimator (Proof of \cref{thm:wasserstein_gradient_variance_bound})}\label{section:proof_wasserstein_gradient_variance_bound}

\thmwassersteingradientvariancebound*

Recall the definition of the gradient estimator
\[
    &
    \widehat{\nabla^{\text{bonnet--price}}_{\mathrm{BW}} \mathcal{E}}\lt(q; \epsilon\rt)\lt(x\rt)
    \nonumber
    \\
    &\quad= \;
    x \mapsto \widehat{\nabla_m^{\text{bonnet}} \mathcal{E}}(q; \epsilon) + 2 \widehat{\nabla_{\Sigma}^{\text{price}} \mathcal{E}}(q; \epsilon) \lt(x - m\rt) 
\nonumber
    \\
    &\quad= \;
    x \mapsto \nabla U\lt(Z\rt) + \nabla^2 U\lt(Z\rt) \lt(x - m\rt) \; ,
\nonumber
\]
where $Z$ is sampled given $\epsilon$ as, for example, $Z = \phi_{\lambda}(\epsilon)$.
First, we can decompose the gradient variance into two terms, each corresponding to the contribution of the gradient with respect to the two parameters of $q_*$, $m_t$ and $\Sigma_t$.

\begin{proof}
Since the gradients corresponding to $m_t$ and $\Sigma_t$ are orthogonal in $\mathrm{L}^2\lt(\psi^*\rt)$, taking expectation over $\epsilon \sim \varphi$ and $(X, X_*) \sim \psi^*$, we have 
\[
    &
    \mathbb{E}\lt[ \norm{\widehat{\nabla^{\text{bonnet--price}}_{\mathrm{BW}} \mathcal{E}}\lt(q; \epsilon\rt)\lt(X_t\rt) - \nabla_{\mathrm{BW}} \mathcal{E}\lt(q_*\rt)\lt(X_*\rt)}_2^2
    \rt]
\nonumber
    \\
    &=
    \mathbb{E}
    \Big\lVert
    \big( \widehat{\nabla_m^{\text{bonnet}} \mathcal{E}}(q; \epsilon) + 2 \widehat{\nabla_{\Sigma}^{\text{price}} \mathcal{E}}(q; \epsilon) \lt( X - m \rt) \big) 
\nonumber
    \\
    &\quad\qquad
    - 
    \lt( \nabla_m \mathcal{E}(q_*) + 2 \nabla_{\Sigma} \mathcal{E}(q_*) \lt( X_* - m_* \rt) \rt)
    {\Big\rVert}_2^2
\nonumber
    \\
    &=
    \mathbb{E}
    \norm*{ 
    \widehat{\nabla_m^{\text{bonnet}} \mathcal{E}}(q; \epsilon) - \nabla_m \mathcal{E}(q_*)
    }_2^2    
\nonumber
    \\
    &\;\;
    +
    \mathbb{E}
    \norm*{
    2 \widehat{\nabla_{\Sigma}^{\text{price}} \mathcal{E}}(q; \epsilon) \lt( X - m \rt) 
    -
    2 \nabla_{\Sigma} \mathcal{E}(q_*) \lt( X_* - m_* \rt) 
    }_2^2
\nonumber
    \\
    &=
    \underbrace{
    \mathbb{E}_{\epsilon \sim \varphi}
    \norm*{ 
    \widehat{\nabla_m^{\text{bonnet}} \mathcal{E}}(q; \epsilon) - \nabla_m \mathcal{E}(q_*)
    }_2^2    
    }_{\text{variance of gradient w.r.t. $m_t$}}
\nonumber
    \\
    &
    \;\;+
    \underbrace{
    \mathbb{E}
    \norm*{
    2 \widehat{\nabla_{\Sigma}^{\text{price}} \mathcal{E}}(q; \epsilon) \lt( X - m \rt) 
    -
    2 \nabla_{\Sigma} \mathcal{E}(q_*) \lt( X_* - m_* \rt) 
    }_2^2 
    }_{\text{variance of gradient w.r.t. $\Sigma$}} .
    \label{eq:wasserstein_gradient_variance_decomposition}
\]
We will analyze each term separately.

First, the gradient variance contributed by $m$ is bounded as
\[
    &\mathbb{E}_{\epsilon \sim \varphi}
    \norm*{ 
    \widehat{\nabla_m^{\text{bonnet}} \mathcal{E}}(q; \epsilon) - \nabla_m \mathcal{E}(q_*)
    }_2^2    
\nonumber
    \\
    &\qquad=
    \mathbb{E}_{Z \sim q}
    \norm*{ 
    \nabla U\lt(Z\rt) - \mathbb{E}_{q_*}\nabla U 
    }_2^2    
\nonumber
    \\
    &\qquad\leq
    4 L \, \mathbb{E}_{(X, X_*) \sim \psi} \lt[ \mathrm{D}_U\lt(X, X_*\rt) \rt] + 2 d L \; ,
    \label{eq:wasserstein_gradient_variance_bound_location}
\]
where the last step follows from \cref{thm:potential_gradient_variance_bound}.

Second, the gradient variance contributed by $\Sigma_t$ can be decomposed as
\[
    &
    \mathbb{E}
    \norm*{
    2 \widehat{\nabla_{\Sigma}^{\text{price}} \mathcal{E}}(q; \epsilon) \lt( X - m \rt) 
    -
    2 \nabla_{\Sigma} \mathcal{E}(q_*) \lt( X_* - m_* \rt) 
    }_2^2
\nonumber
    \\
    &=
    \mathbb{E}\norm*{
        \nabla^2 U\lt(Z\rt) \lt( X - m \rt) 
        -
        \mathbb{E}_{q_*} \lt[ \nabla^2 U \rt] \lt( X_* - m_* \rt) 
    }_2^2
\nonumber
    \\
    &\leq
    2 
    \underbrace{
    \mathbb{E}_{Z \sim q, X \sim q}\norm*{
        \nabla^2 U\lt(Z\rt) \lt( X - m \rt) 
    }_2^2
    }_{\triangleq V_{\mathrm{mul}}}
\nonumber
    \\
    &\quad
    +
    2
    \underbrace{
    \mathbb{E}_{X_* \sim q_*}\norm*{
        \mathbb{E}_{q_*}\lt[\nabla^2 U\rt] \lt( X_* - m_* \rt) 
    }_2^2
    }_{\triangleq V_{\mathrm{add}}} 
    \label{eq:wasserstein_gradient_variance_covariance_eq1}
\]
by Young's inequality.

The gradient variance $\widehat{\nabla_{\Sigma}^{\text{price}} \mathcal{E}}$, and in turn $\widehat{\nabla_{\mathrm{BW}}^{\text{bonnet--price}} \mathcal{E}}$, is dominated by the multiplicative noise variance  $V_{\mathrm{mul}}$.
Therefore, bounding $V_{\mathrm{mul}}$ is the most crucial step, which follows from~\cref{thm:covariance_weighted_hessian_bound}.
That is,
\[
    V_{\mathrm{mul}}
    &=
    \mathbb{E}_{X \sim q, Z \sim q}\norm*{
        \nabla^2 U\lt(Z\rt) \lt( X - m \rt) 
    }_2^2
\nonumber
    \\
    &=
    \mathbb{E}_{Z \sim q} \operatorname{tr}\lt(
        {(\nabla^2 U\lt(Z\rt))}^2 \mathbb{E}_{X \sim q} \lt( X - m \rt) {\lt( X - m \rt)}^{\top}
    \rt)
\nonumber
    \\
    &=
    \mathbb{E}_{Z \sim q}
    \operatorname{tr}\lt(
        \nabla^2 U\lt(Z\rt) \Sigma \nabla^2 U\lt(Z\rt)
    \rt)
\nonumber
    \\
    &\leq
    L \lt( 2 \sqrt{\kappa} + \kappa \rt) \mathbb{E}_{ (X,X_*) \sim \psi} \lt[ \mathrm{D}_U\lt(X,X_*\rt) \rt]
    + 
    3 d L  \, ,
\nonumber
\]
where the last step follows from \cref{thm:covariance_weighted_hessian_bound}.
On the other hand, $V_{\text{add}}$ immediately follows from the stationarity condition for minimizing the free energy $\mathcal{F}$.
\[
    V_{\text{add}}
    &=
    \mathbb{E}_{X_* \sim q_*}\norm*{
        \mathbb{E}_{q_*}\nabla^2 U \lt( X_* - m_* \rt) 
    }_2^2
\nonumber
    \\
    &=
    \operatorname{tr}\lt(
        {\lt( \mathbb{E}_{q_*}\nabla^2 U \rt)}^2
        \mathbb{E}_{X_* \sim q_*}
        {\lt( X_* - m_* \rt)}
        {\lt( X_* - m_* \rt)}^{\top}
    \rt)
\nonumber
    \\
    &=
    \operatorname{tr}\lt(
        {\lt( \mathbb{E}_{q_*}\nabla^2 U \rt)}^2
        \Sigma_*
    \rt) \; .
\nonumber
\shortintertext{The stationary condition (\cref{thm:stationary_condition}) yields} 
    &=
    \operatorname{tr}\lt(
        \Sigma_*^{-2}
        \Sigma_*
    \rt)
\nonumber
    \\
    &=
    \operatorname{tr}\lt(
        \Sigma_*^{-1}
    \rt) \; .
\nonumber
\]
Therefore, resuming from \cref{eq:wasserstein_gradient_variance_covariance_eq1},
\[
    &
    \mathbb{E}
    \norm*{
    2 \widehat{\nabla_{\Sigma}^{\text{price}} \mathcal{E}}(q; \epsilon) \lt( X - m \rt) 
    -
    2 \nabla_{\Sigma} \mathcal{E}(q_*) \lt( X_* - m_* \rt) 
    }_2^2
\nonumber
    \\
    &\;\leq
    2 V_{\text{add}} + 2 V_{\text{mul}}
\nonumber
    \\
    &\leq
    L \lt( 4 \sqrt{\kappa} + 2 \kappa \rt) \mathbb{E}_{ (X,X_*) \sim \psi} \lt[ \mathrm{D}_U\lt(X,X_*\rt) \rt]
\nonumber
    \\
    &\qquad
    + 
    6 d L 
    +
    2
    \operatorname{tr}\lt(
        \Sigma_*^{-1}
    \rt)
\nonumber
    \\
    &\leq
    L \lt( 4 \sqrt{\kappa} + 2 \kappa \rt) \mathbb{E}_{ (X,X_*) \sim \psi} \lt[ \mathrm{D}_U\lt(X,X_*\rt) \rt]
    + 
    8 d L  \; ,
    \label{eq:wasserstein_gradient_variance_covariance_eq2}
\]
where we used the fact that $\Sigma_*^{-1} = \mathbb{E}_{q_*}\nabla^2 U \preceq L \mathrm{I}_d$.

Combining \cref{eq:wasserstein_gradient_variance_bound_location,eq:wasserstein_gradient_variance_covariance_eq2} into \cref{eq:wasserstein_gradient_variance_decomposition}, 
\[
    &
    \mathbb{E}\lt[ \norm{\widehat{\nabla^{\text{bonnet--price}} \mathcal{E}}\lt(q; \epsilon\rt)\lt(X_t\rt) - \nabla \mathcal{E}\lt(q_*\rt)\lt(X_*\rt)}_2^2
    \rt]
\nonumber
    \\
    &\leq
    \Big\{
    4 L \, \mathbb{E}_{(X, X_*) \sim \psi} \lt[ \mathrm{D}_U\lt(X, X_*\rt) \rt]
    + 2 d L
    \Big\}
\nonumber
    \\
    &\qquad
    +
    \Big\{
    L \lt( 4 \sqrt{\kappa} + 2 \kappa \rt) \mathbb{E}_{ (X,X_*) \sim \psi} \lt[ \mathrm{D}_U\lt(X,X_*\rt) \rt]
    + 
    8 d L
    \Big\}
\nonumber
    \\
    &\leq
    10 L \kappa \, \mathbb{E}_{(X, X_*) \sim \psi} \lt[ \mathrm{D}_U\lt(X, X_*\rt) \rt]
    + 
    10 d L \; ,
\nonumber
\]
where we used the fact that $\kappa \geq 1$.
\end{proof}

\newpage

\subsubsection{Wasserstein Bregman Divergence Identity (Proof of \cref{thm:wasserstein_bregman_equivalence})}\label{section:proof_wasserstein_bregman_equivalence}

\thmwassersteinbregmanequivalence*

\begin{proof}
By definition of the Bregman divergence, 
\[
    &
    \mathbb{E}_{(X,Y) \sim \psi^*}\lt[\mathrm{D}_U\lt(X, Y\rt)\rt]
\nonumber
    \\
    &=
    \mathbb{E}_{(X,Y) \sim \psi^*}\lt[
        U\lt(X\rt) - U\lt(Y\rt) 
        - \inner{\nabla U\lt(Y\rt), X - Y}
    \rt] 
\nonumber
    \\
    &=
    \mathcal{E}\lt(p\rt) - \mathcal{E}\lt(q\rt)
    -
    \mathbb{E}_{(X,Y) \sim \psi^*}\lt[
        \inner{\nabla U\lt(Y\rt), X - Y}
    \rt] 
\nonumber
    \\
    &=
    \mathrm{D}_{\mathcal{E}}\lt(p, q\rt)
    +
    \mathbb{E}_{(X,Y) \sim \psi^*}\lt[
        \inner{\nabla \mathcal{E}\lt(q\rt)\lt(Y\rt), X - Y}
    \rt] 
\nonumber
    \\
    &\qquad\qquad\qquad
    -
    \mathbb{E}_{(X,Y) \sim \psi^*}\lt[
        \inner{\nabla U\lt(Y\rt), X - Y}
    \rt] 
    \; .
    &&\label{eq:wasserstein_bregman_equivalence_decomposition}
\]
It remains to show that the two inner product terms cancel.

For any $r = \mathrm{Normal}(m, \Sigma) \in \mathrm{BW}(\mathbb{R}^d)$, the Wasserstein gradient $\nabla_{\mathrm{W}}$ of the energy $\mathcal{E}$ is $\nabla_{\mathrm{W}} \mathcal{E} = \nabla U$.
Furthermore, the Bures-Wasserstein gradient of $\mathcal{E}$ is the orthogonal $\mathrm{L}^2\lt(r\rt)$ projection of the Wasserstein gradient onto the tangent space of $\mathrm{BW}(\mathbb{R}^d)$~\citep[Appendix C.1]{lambert_variational_2022}
\[
    \mathcal{T}_{r} \mathrm{BW}(\mathbb{R}^d)
    =
    \lt\{
        x \mapsto a + S \lt(x - m\rt) \mid a \in \mathbb{R}^d, S \in \mathbb{S}^d
    \rt\}
    \; .
\nonumber
\]
That is,
\[
    \nabla_{\mathrm{BW}} \mathcal{E}\lt(r\rt)
    &=
    \operatorname{proj}_{\mathcal{T}_r \mathrm{BW}(\mathbb{R}^d) }\lt( \nabla_{\mathrm{W}} \mathcal{E} \rt)
\nonumber
\\
    &=
    \operatorname{proj}_{\mathcal{T}_r \mathrm{BW}(\mathbb{R}^d) }\lt( \nabla U \rt)
\nonumber
    \\
    &=
    \argmin_{w \in \mathcal{T}_r \mathrm{BW}(\mathbb{R}^d)} \mathbb{E}_r \norm{ \nabla U - w }_2^2 \; ,
\nonumber
\]
which is the unique element of $\mathcal{T}_r \mathrm{BW}(\mathbb{R}^d)$ satisfying, for all $v \in \mathcal{T}_r \mathrm{BW}(\mathbb{R}^d)$,
\[
    \mathbb{E}_r
    \inner{ \nabla_{\mathrm{BW}} \mathcal{E}(r), v }
    &=
    \mathbb{E}_r
    \inner{ \operatorname{proj}_{\mathcal{T}_r \mathrm{BW}(\mathbb{R}^d) }\lt( \nabla U \rt), v }
\nonumber
    \\
    &=
    \mathbb{E}_p\inner{\nabla U, v} \; .
    &&\label{eq:wasserstein_bregman_equivalence_inner_product}
\]
Now, $\mathcal{T}_{r} \mathrm{BW}(\mathbb{R}^d)$ is equivalent to the set of all affine maps with a symmetric matrix.
Therefore, for any such map, the identity
\cref{eq:wasserstein_bregman_equivalence_inner_product} holds.

Under the optimal coupling $\psi^*$, denote the optimal transport map from $p = \mathrm{Normal}(m_p, \Sigma_p)$ to $q = \mathrm{Normal}(m_q, \Sigma_q)$ as $M^*_{p \mapsto q}$, which is an affine function of the form of
\[
    M^*_{p \mapsto q}
    =
    S_{p \mapsto q}(x - m_p) + m_q 
    \; ,
\nonumber
\]
where
\[
    S_{p \mapsto q} = \Sigma_p^{-\frac12}\left(\Sigma_p^{\frac12}\Sigma_q\Sigma_p^{\frac12}\right)^{\frac12}\Sigma_p^{-\frac12}
\nonumber
\]
is a symmetric matrix.
From inspection, it is clear that  $M^*_{p \mapsto q} \in \mathcal{T}_{p} \mathrm{BW}(\mathbb{R}^d)$.
Consequently, we can also conclude that $M^*_{p \mapsto q} - \mathrm{Id} \in \mathcal{T}_{p} \mathrm{BW}(\mathbb{R}^d)$.
Therefore, 
\[
    &
    \mathbb{E}_{(X,Y) \sim \psi^*} \inner{\nabla_{\mathrm{BW}} \mathcal{E}\lt(q\rt)(Y), X - Y}
\nonumber
    \\
    &\quad=
    \mathbb{E}_{q} \inner{\nabla_{\mathrm{BW}} \mathcal{E}\lt(q\rt), M^*_{p \mapsto q} - \mathrm{Id}}
\nonumber
    \\
    &\quad=
    \mathbb{E}_{q} \inner{\nabla U, M^*_{p \mapsto q} - \mathrm{Id}}
    &&\text{(\cref{eq:wasserstein_bregman_equivalence_inner_product})}
\nonumber
    \\
    &\quad=
    \mathbb{E}_{(X,Y) \sim \psi^*} \inner{\nabla U(Y), X - Y} \; ,
\nonumber
\]
which means the outstanding inner products in \cref{eq:wasserstein_bregman_equivalence_decomposition} cancel out completely.
\end{proof}

\newpage

\subsubsection{Coercivity of Bures-Wasserstein Gradient (Proof of \cref{thm:wasserstein_coercivity})}\label{section:proof_wasserstein_coercivity}

\thmwassersteincoercivity*

To prove \cref{thm:wasserstein_coercivity}, we need the energy $\mathcal{E}$ to be $\mu$-strongly convex under a certain notion of convexity in Bures-Wasserstein space.
This is given by the following supporting result:

\begin{lemma}[Lemma B.1; \citealp{diao_forwardbackward_2023}]\label{thm:energy_strongly__geodesicallyconvex}
    Suppose \cref{assumption:potential_convex_smooth} holds.
    Then, for any $q \in \mathrm{BW}\lt(\mathbb{R}^d\rt)$ and any affine map $M : \mathbb{R}^d \to \mathbb{R}^d$, the expected energy $\mathcal{E}$ is $\mu$-strongly geodesically convex such that
    \[
        &\mathcal{E}({(\mathrm{Id} + M)}_{\# q}) - \mathcal{E}\lt(q\rt) 
        \nonumber
        \\
        &\qquad\geq 
        \mathbb{E}_{q}\inner{ \nabla_{\mathrm{BW}} \mathcal{E}\lt(q\rt), M } + \frac{\mu}{2} \mathbb{E}_q \norm{M}_2^2 \; .
\nonumber
    \]
\end{lemma}

Since the optimal transport map between $q_t$ and $q_*$ is an affine map, the difference $X - Y$ is also an affine map.
Therefore, \cref{thm:energy_strongly__geodesicallyconvex} can be used to establish coercivity.

\begin{proof}[Proof of \cref{thm:wasserstein_coercivity}]
Under the optimal coupling $\psi_* \in \Psi(p, q)$, the difference $X - Y$ is an affine map over the Law of $X$ or $Y$.
Therefore,
\[
    &
    \mathbb{E}_{(X, Y) \sim \psi_*} \inner{ \nabla_{\mathrm{BW}} \mathcal{E}\lt(p\rt)(X) - \nabla \mathcal{E}\lt(q\rt)(Y) , X - Y }
\nonumber
    \\
    &=
    - \mathbb{E}_{(X, Y) \sim \psi_*} \inner{ \nabla_{\mathrm{BW}} \mathcal{E}\lt(p\rt)(X) , Y - X }
\nonumber
    \\
    &\qquad
    -
    \mathbb{E}_{(X, Y) \sim \psi_*} \inner{ \nabla_{\mathrm{BW}} \mathcal{E}\lt(q\rt)(Y) , X - Y }
\nonumber
    \\
    &\geq
    - \lt( \mathcal{E}\lt(q\rt) - \mathcal{E}\lt(p\rt) - \frac{\mu}{2} {\mathrm{W}_2\lt(p, q\rt)}^2 \rt) \; .
\nonumber
\shortintertext{The $\mu$-strong geodesic convexity (\cref{thm:energy_strongly__geodesicallyconvex}) of $\mathcal{E}$ implies}
    &\qquad
    -
    \mathbb{E}_{(X, Y) \sim \psi_*} \inner{ \nabla_{\mathrm{BW}} \mathcal{E}\lt(q\rt)(Y) , X - Y }
\nonumber
    \\
    &=
    \frac{\mu}{2} {\mathrm{W}_2\lt(p, q\rt)}^2
    +
    \big(
    \mathcal{E}\lt(p\rt) - \mathcal{E}\lt(q\rt)
\nonumber
    \\
    &\qquad
    -
    \mathbb{E}_{(X, Y) \sim \psi_*} \inner{ \nabla_{\mathrm{BW}} \mathcal{E}\lt(q\rt)(Y) , X - Y }
    \big) 
\nonumber
    \\
    &=
    \frac{\mu}{2} {\mathrm{W}_2\lt(p, q\rt)}^2
    +
    \mathrm{D}_{\mathcal{E}}\lt(p, q\rt)
    \; ,
\nonumber
\]
where we have applied \cref{eq:energy_wasserstein_bregman}.
\end{proof}

\clearpage
\subsection{Parameter Space Proximal Gradient Descent}

\subsubsection{Iteration Complexity (Proof of \cref{thm:parameterspace_proximal_gradient_descent_pricestein})}\label{section:proof_parameterspace_proximal_gradient_descent_pricestein}

\thmparameterspaceproximalgradientdescentpricestein*

This is a corollary of \cref{thm:parameterspace_proximal_gradient_descent_convergence}, where the sufficient conditions are established in \cref{thm:parameter_gradient_variance_bound,thm:parameter_bregman_equivalence}.

\begin{proof}
Under the stated conditions, we can invoke \cref{thm:parameter_gradient_variance_bound}, where, for any $\lambda \in \Lambda$ the generic coupling $\psi$ in the statement of \cref{thm:parameter_gradient_variance_bound} is set as $\psi = \psi^{\mathrm{rep}} \in \Psi(q_{\lambda}, q_{\lambda_*})$, the coupling associated with the transport map $M_{q_{\lambda} \mapsto q_{\lambda_*}}^{\mathrm{rep}}$ induced by the parametrization in \cref{assumption:parametrization}.
Then, for any $\lambda \in \Lambda$, 
\[
    &
    \mathbb{E}_{\epsilon \sim \varphi}\lt[ \norm{\widehat{\nabla_{\lambda}^{\text{bonnet--price}} \mathcal{E}}\lt(q_{\lambda}; \epsilon\rt) - \nabla \mathcal{E}\lt(q_{\lambda_*}\rt)}_2^2
    \rt]
\nonumber
    \\
    &\leq
    4 L \kappa \, \mathbb{E}_{(X, X_*) \sim \psi^{\mathrm{rep}}} \lt[ \mathrm{D}_U\lt(X, X_*\rt) \rt] 
    + 10 d L 
    &&\text{(\cref{thm:parameter_gradient_variance_bound})}
\nonumber
    \\
    &=
    4 L \kappa \, \mathrm{D}_{\lambda \mapsto \mathcal{E}(q_{\lambda})}\lt(\lambda, \lambda_*\rt)
    + 10 d L \; .
    &&\text{(\cref{thm:parameter_bregman_equivalence})}
\nonumber
\]
This establishes
\[
    \text{\cref{thm:parameter_gradient_variance_bound} \& \cref{thm:parameter_bregman_equivalence}} \quad\Rightarrow\quad
    \text{\cref{assumption:parameterspace_gradient_variance}}
\nonumber
\]
with the constants $L_{\epsilon} = 5/2 L \kappa$ and $\sigma^2 = 5 d L$.
Then 
\[
    \text{\cref{assumption:potential_convex_smooth} \& \cref{assumption:parameterspace_gradient_variance}} 
    \;\Rightarrow\;
    \text{\cref{thm:parameterspace_proximal_gradient_descent_convergence}}
    \; .
\nonumber
\]
Substituting for the constants $L_{\epsilon}$ and $\sigma^2$, the parameters of the step size schedule in \cref{eq:stepsize_schedule} become
\[
    \gamma_0 &= \frac{1}{10 L \kappa}
\nonumber    
    \\
    \tau &= 8 \kappa 
\nonumber
    \\
    t_* 
    &= 
    \min\lt\{
    \lt\lceil
    \frac{1}{\log(1/(1 - \nicefrac{1}{10 \kappa^2}))}
    \log\lt(
    \frac{\kappa}{d} \Delta^2
    \rt) 
    \rt\rceil
    , T
    \rt\} \; ,
\nonumber
\]
which guarantee 
\[
    T \geq \max\{B_{\mathrm{var}}, B_{\mathrm{bias}}\}
    \quad\Rightarrow\quad
    \mu \mathbb{E}[{\mathrm{W}_2\lt(q_T, q_*\rt)}^2] \leq \epsilon
\nonumber
\]
with the constants
\[
        B_{\mathrm{var}}
        &=
        40 d \kappa \frac{1}{\epsilon} 
        + 
        10 \sqrt{d} \kappa^{3/2}
        \nonumber
        \\
        &\qquad\times
        \lt\{
            \log \lt( \frac{\kappa }{d} \mu V_0 \rt)
            +
            \frac{1}{10 \kappa^2}
            +
            \sqrt{2}
        \rt\} \frac{1}{\sqrt{\epsilon}}
        \nonumber
        \\
        B_{\mathrm{bias}}
        &=
        10 \kappa^2 \log\lt( 2 \Delta^2 \frac{1}{\epsilon} \rt) 
        \; .
\nonumber
\]
\end{proof}

\newpage

\subsubsection{Variance Bound on the Parameter Gradient Estimator (Proof of \cref{thm:parameter_gradient_variance_bound})}\label{section:proof_parameter_gradient_variance_bound}

\thmparametergradientvariancebound*

Recall the definition of the second-order parameter gradient
\[
    \widehat{\nabla_{\lambda}^{\text{bonnet--price}} \mathcal{E}}\lt(q_{\lambda}; \epsilon\rt)
    =
    \begin{bmatrix}
         \widehat{\nabla_{m}^{\text{bonnet}} \mathcal{E}}\lt(q_{\lambda}; \epsilon\rt) \\
         \widehat{\nabla_{C}^{\text{price}} \mathcal{E}}\lt(q_{\lambda}; \epsilon\rt) 
    \end{bmatrix} 
    =
    \begin{bmatrix}
         \nabla U\lt(Z\rt) \\
         C^{\top} \nabla^2 U\lt(Z\rt)
    \end{bmatrix} 
\nonumber
    \; ,
\]
where $Z = \phi_{\lambda}\lt(\epsilon\rt) = C \epsilon + m$.
We will decompose the gradient variance into the variance of the location component $\widehat{\nabla_{m}^{\text{bonnet}} \mathcal{E}}\lt(q_{\lambda}; \epsilon\rt)$ and the scale component $\widehat{\nabla_{C}^{\text{price}} \mathcal{E}}\lt(q_{\lambda}; \epsilon\rt)$, and then bound each separately.

\begin{proof}
Since the location and scale compoments of the gradient estimator are orthogonal, 
\[
    &
    \mathbb{E}_{\epsilon \sim \varphi}\lt[ \norm{\widehat{\nabla_{\lambda} \mathcal{E}}\lt(q_{\lambda}; \epsilon\rt) - \nabla \mathcal{E}\lt(q_{\lambda_*}\rt)}_2^2
    \rt]
\nonumber
    \\
    &\;=
    \underbrace{
    \mathbb{E}_{\epsilon \sim \varphi}
    \norm*{ 
    \widehat{\nabla_{m}^{\text{bonnet}} \mathcal{E}}\lt(q_{\lambda}; \epsilon\rt)
    -
    \nabla_{m} \mathcal{E}\lt(q_{\lambda}\rt)
    }_2^2    
    }_{\text{variance of gradient w.r.t. $m_t$}}
\nonumber
    \\
    &\qquad
    +
    \underbrace{
    \mathbb{E}_{\epsilon \sim \varphi}
    \norm[\Big]{
    \widehat{\nabla_{C}^{\text{price}} \mathcal{E}}\lt(q_{\lambda}; \epsilon\rt)
    -
    \nabla_{C} \mathcal{E}\lt(q_{\lambda}\rt)
    }_{\mathrm{F}}^2 
    }_{\text{variance of gradient w.r.t. $\Sigma$}}
    \; .
    \label{eq:parameter_gradient_variance_decomposition}
\]
The variance of the location component can immediately be bounded by \cref{thm:potential_gradient_variance_bound}.
\[
    &
    \mathbb{E}_{Z \sim q_{\lambda}}
    \norm*{ 
    \widehat{\nabla_{m}^{\text{bonnet}} \mathcal{E}}\lt(q_{\lambda}; \epsilon\rt)
    -
    \nabla_{m} \mathcal{E}\lt(q_{\lambda}\rt)
    }_2^2
\nonumber
    \\
    &\leq
    \mathbb{E}_{Z \sim q_{\lambda}}
    \norm*{ 
    \nabla U\lt(Z\rt)
    -
    \mathbb{E}_q \nabla U
    }_2^2
\nonumber
    \\
    &\leq
    4 L \, \mathbb{E}_{(X, X_*) \sim \psi} \lt[ \mathrm{D}_U\lt(X, X_*\rt) \rt] + 2 d L \; .
    \label{eq:parameter_gradient_variance_location}
\]
On the other hand, for the scale component, we first apply Young's inequality such that
\[
    &
    \mathbb{E}_{\epsilon \sim \varphi}\norm{
    \widehat{\nabla_{C}^{\text{price}} \mathcal{E}}\lt(q_{\lambda}; \epsilon\rt)
    -
    \nabla_{C} \mathcal{E}\lt(q_{\lambda}\rt)
    }_{\mathrm{F}}^2
\nonumber
    \\
    &\;=
    \mathbb{E}_{Z \sim q_{\lambda}}
    \norm[\Big]{
    C^{\top} \nabla^2 U\lt(Z\rt) 
    -
    C_*^{\top} \mathbb{E}_{q_*}\lt[ \nabla^2 U \rt] 
    }_{\mathrm{F}}^2
\nonumber
    \\
    &\;\leq
    2 \, 
    \underbrace{
    \mathbb{E}_{Z \sim q_{\lambda}}\norm*{ C^{\top} \nabla^2 U\lt(Z\rt) }_{\mathrm{F}}^2
    }_{V_{\text{mul}}}
    +
    2
    \underbrace{
    \norm*{ C_*^{\top} \mathbb{E}_{q_*}\lt[\nabla^2 U\rt] }_{\mathrm{F}}^2
    }_{V_{\text{add}}}
    \label{eq:parameter_gradient_variance_scale_decomposition}
    \; .
\]
Again, similarly to the Bures-Wasserstein case, the variance is dominated by the multiplicative noise of the scale $V_{\text{mul}}$, which can be bounded by \cref{thm:covariance_weighted_hessian_bound}.
\[
    V_{\text{mul}}
    &=
    \mathbb{E}_{Z \sim q_{\lambda}}\norm*{ C^{\top} \nabla^2 U\lt(Z\rt) }_2^2
\nonumber
    \\
    &=
    \mathbb{E}_{Z \sim q_{\lambda}} 
    \operatorname{tr}\lt(
    \nabla^2 U\lt(Z\rt)
    C 
    C^{\top} \nabla^2 U\lt(Z\rt)
    \rt)
\nonumber
    \\
    &=
    \mathbb{E}_{Z \sim q_{\lambda}} 
    \operatorname{tr}\lt(
    \nabla^2 U\lt(Z\rt) \Sigma \nabla^2 U\lt(Z\rt)
    \rt)
\nonumber
    \\
    &\leq
    L \lt( 2 \sqrt{\kappa} + \kappa \rt) \mathbb{E}_{ (X,X_*) \sim \psi} \lt[ \mathrm{D}_U\lt(X,X_*\rt) \rt]
    + 
    3 d L  \; . 
    \label{eq:parameter_gradient_variance_scale_multiplicative}
\]
On the other hand, $V_{\text{add}}$ immediately follows from the properties of the stationary point (\cref{thm:stationary_condition}) as
\[
    V_{\text{add}}
    &=
    \norm*{ C_*^{\top} \Sigma_*^{-1} }_{\mathrm{F}}^2
\nonumber
    \\
    &=
    \operatorname{tr}\lt(
        \Sigma_*^{-1} C_* C_*^{\top} \Sigma_*^{-1}
    \rt)
\nonumber
    \\
    &=
    \operatorname{tr}\lt(
        \Sigma_*^{-1} \Sigma_* \Sigma_*^{-1}
    \rt)
\nonumber
    \\
    &=
    \operatorname{tr}\lt(
        \Sigma_*^{-1}
    \rt) \; .
    \label{eq:parameter_gradient_variance_scale_additive}
\]
Therefore, 
\[
    &
    \mathbb{E}_{\epsilon \sim \varphi}\norm{
    \widehat{\nabla_{C}^{\text{bonnet--price}} \mathcal{E}}\lt(q_{\lambda}; \epsilon\rt)
    -
    \nabla_{C} \mathcal{E}\lt(q_{\lambda}\rt)
    }_{\mathrm{F}}^2
\nonumber
    \\
    &\leq
    2 V_{\text{mul}} + 2 V_{\text{add}}
\nonumber
    \\
    &\leq
    \Big\{
    L \lt( 4 \sqrt{\kappa} + 2 \kappa \rt) \mathbb{E}_{ (X,X_*) \sim \psi} \lt[ \mathrm{D}_U\lt(X,X_*\rt) \rt]
    +
    6 d L
    \Big\}
\nonumber
    \\
    &\qquad
    +
    \operatorname{tr}\lt(\Sigma_*^{-1}\rt)
\nonumber
    \\
    &\leq
    L \lt( 4 \sqrt{\kappa} + 2 \kappa \rt) \mathbb{E}_{ (X,X_*) \sim \psi} \lt[ \mathrm{D}_U\lt(X,X_*\rt) \rt]
    +
    8 d L \; .
    \label{eq:parameter_gradient_variance_scale_final_bound}
\]
where we used the fact that, from \cref{thm:stationary_condition},
\[
    \Sigma_*^{-1} \;=\; \mathbb{E}_{q_*}\nabla^2 U \;\preceq\; L \mathrm{I}_d \; .
    \nonumber
\]

Combining \cref{eq:parameter_gradient_variance_scale_final_bound,eq:parameter_gradient_variance_location,eq:parameter_gradient_variance_scale_additive} into \cref{eq:parameter_gradient_variance_decomposition}, 
\[
    &
    \mathbb{E}_{\epsilon \sim \varphi} \norm[\big]{\widehat{\nabla_{\lambda}^{\text{bonnet--price}} \mathcal{E}}\lt(q_{\lambda}; \epsilon\rt) - \nabla \mathcal{E}\lt(q_{\lambda_*}\rt)}_2^2
\nonumber
    \\
    &\leq
    \Big\{
    4 L \, \mathbb{E}_{(X, X_*) \sim \psi} \lt[ \mathrm{D}_U\lt(X, X_*\rt) \rt] + 2 d L 
    \Big\}
\nonumber
    \\
    &\qquad
    +
    \Big\{
    L \lt( 4 \sqrt{\kappa} + 2 \kappa \rt) \mathbb{E}_{ (X,X_*) \sim \psi} \lt[ \mathrm{D}_U\lt(X,X_*\rt) \rt]
    +
    8 d L
    \Big\}
\nonumber
    \\
    &\leq
    10 L \kappa \, \mathbb{E}_{(X, X_*) \sim \psi} \lt[ \mathrm{D}_U\lt(X, X_*\rt) \rt] 
    + 10 d L \; ,
\nonumber
\]
where we used the fact that $\kappa \geq 1$.
\end{proof}

\newpage

\subsubsection{Bregman Divergence Identity (Proof of \cref{thm:parameter_bregman_equivalence})}\label{section:proof_parameter_bregman_equivalence}

\thmparameterbregmanequivalence*

\begin{proof}
By definition,
\[
    &
    \mathbb{E}_{ (X,X') \sim \psi^{\mathrm{rep}}} \mathrm{D}_U\lt(\lambda, \lambda'\rt)
\nonumber
    \\
    &=
    \mathbb{E}_{ (X,X') \sim \psi^{\mathrm{rep}}} \lt[
        U\lt(X\rt) - U\lt(X'\rt) - \inner{\nabla U\lt(X'\rt), X - X'}
    \rt]
\nonumber
    \\
    &=
    \mathcal{E}(q_{\lambda}) - \mathcal{E}(q_{\lambda'})
    -
    \mathbb{E}_{ (X,X') \sim \psi^{\mathrm{rep}}} \lt[
        \inner{\nabla U\lt(X'\rt), X - X'}
    \rt] \; .
\nonumber
\]
Therefore, we only need to show that 
\[
    \mathbb{E}_{ (X,X') \sim \psi^{\mathrm{rep}}} \lt[
        \inner{\nabla U\lt(X'\rt), X - X'}
    \rt]
    =
    \inner{ \nabla_{\lambda} \mathcal{E}(q_{\lambda}), \lambda - \lambda' } \; ,
\nonumber
\]
which is essentially \cref{eq:parameter_inner_product_coupling_inner_product_equivalence}.
Denoting $\lambda = (m, \operatorname{vec} C)$ and $\lambda' = (m', \operatorname{vec} C')$, this follows from
\[
    &
    \inner{ \nabla_{\lambda} \mathcal{E}(q_{\lambda}), \lambda - \lambda' }
\nonumber
    \\
    &\quad=
    \inner*{ 
    \begin{bmatrix}
        \mathbb{E}_{\epsilon \sim \varphi} \nabla U(\phi_{\lambda} \lt(\epsilon\rt)) \\
        \mathbb{E}_{\epsilon \sim \varphi} \nabla U(\phi_{\lambda} \lt(\epsilon\rt)) \epsilon^{\top}  \\
    \end{bmatrix} ,
    \begin{bmatrix}
        m - m' \\ C - C'
    \end{bmatrix}
    }
\nonumber
    \\
    &\quad=
    \inner*{ 
        \mathbb{E}_{\epsilon \sim \varphi} \nabla U(\phi_{\lambda} \lt(\epsilon\rt)), m - m'
    }
\nonumber
    \\
    &
    \qquad\qquad\qquad
    +
    \inner{ 
        \mathbb{E}_{\epsilon \sim \varphi}  \nabla U(\phi_{\lambda} \lt(\epsilon\rt)) {\epsilon}^{\top}, C - C'
    }_{\mathrm{F}}
\nonumber
    \\
    &\quad=
    \mathbb{E}_{\epsilon \sim \varphi} 
    \Big[
    \inner*{ 
        \nabla U(\phi_{\lambda} \lt(\epsilon\rt)), m - m'
    }
\nonumber
    \\
    &
    \qquad\qquad\qquad
    +
    \inner{ 
        \nabla U(\phi_{\lambda} \lt(\epsilon\rt)) , \lt( C - C' \rt) \epsilon
    }_{\mathrm{F}}
    \Big]
\nonumber
    \\
    &\quad=
    \mathbb{E}_{\epsilon \sim \varphi} 
    \inner*{ 
        \nabla U(\phi_{\lambda} \lt(\epsilon\rt)), \lt(C \epsilon + m\rt) - \lt( C' \epsilon + m' \rt)
    }
\nonumber
    \\
    &\quad=
    \mathbb{E}_{\varphi} 
    \inner*{ 
        \nabla U(\phi_{\lambda}), \phi_{\lambda} - \phi_{\lambda'}
    }
\nonumber
    \\
    &\quad=
    \mathbb{E}_{(Z,Z') \sim \psi^{\mathrm{rep}}} 
    \inner*{ 
        \nabla U(Z'), Z - Z'
    } \; .
\nonumber
\]
\end{proof}

\newpage

\subsubsection{Coercivity of the Parameter Gradient (Proof of \cref{thm:parameter_coercivity})}\label{section:proof_parameter_coercivity}

\thmparametercoercivity*

Under \cref{assumption:parametrization,assumption:potential_convex_smooth}, the energy is $\mu$-strongly convex.

\begin{lemma}[name={Theorem 9; \citealp{domke_provable_2020}}]\label{eq:domke_strong_convexity_transfer}
    Suppose \cref{assumption:potential_convex_smooth,assumption:parametrization} hold. 
    Then $\lambda \mapsto \mathcal{E}(q_{\lambda})$ is $\mu$-strongly convex.
\end{lemma}
An alternative proof is also presented by \citet[Theorem 2]{kim_convergence_2023}.
Regardless of the way we prove \cref{eq:domke_strong_convexity_transfer}, coercivity immediately follows from the basic properties of strongly convex functions.

\begin{proof}[Proof of \cref{thm:parameter_coercivity}]
\[
    &
    \inner*{ \nabla_{\lambda_t} \mathcal{E}\lt(q_{\lambda_t}\rt)  - \mathcal{E}\lt(q_{\lambda_*}\rt), \lambda_t - \lambda_* }
\nonumber
    \\
    &\quad=
    \inner*{ \nabla_{\lambda_t} \mathcal{E}\lt(q_{\lambda_t}\rt) , \lambda_t - \lambda_* }
    -
    \inner*{ \nabla_{\lambda_*} \mathcal{E}\lt(q_{\lambda_*}\rt) , \lambda_t - \lambda_* } \; .
\nonumber
\shortintertext{Applying \cref{eq:domke_strong_convexity_transfer},}
    &\quad\geq
    \frac{\mu}{2} \norm{\lambda_t - \lambda_*}_2^2
\nonumber
    \\
    &\qquad
    +
    \lt(
    \mathcal{E}\lt(q_{\lambda}\rt)
    -
    \mathcal{E}\lt(q_{\lambda_*}\rt)
    -
    \inner*{ \nabla_{\lambda_*} \mathcal{E}\lt(q_{\lambda_*}\rt) , \lambda_t - \lambda_* }
    \rt)
\nonumber
    \\
    &\quad=
    \frac{\mu}{2} \norm{\lambda_t - \lambda_*}_2^2
    +
    \mathrm{D}_{\lambda \mapsto \mathcal{E}(q_{\lambda})}\lt(\lambda_t, \lambda_*\rt)
    \; .
\nonumber
\]
\end{proof}

\end{document}